%% file: inverse_bandit.tex
\documentclass[11pt,twoside]{article}

\usepackage{fullpage}

\input{ap_style}
\input{wg_style}
\input{macros_arxiv}

\usepackage{wrapfig}
\usepackage{subfig}

\ifnum\ShowAuthNotes=1 
\newcommand{\authnote}[2]{{ \footnotesize \bf{\color{DarkRed}[#1's note:
{\color{DarkBlue}#2}]}}}
\else
\newcommand{\authnote}[2]{}
\fi

\newcommand{\wg}[1]{{\authnote{Wenshuo} {#1}}}

\newcommand{\vm}[1]{{\authnote{Vidya} {#1}}}
\newcommand{\ap}[1]{{\authnote{Ashwin} {#1}}}

\begin{document}

\begin{center}

  {\bf{\LARGE{Learning from an Exploring Demonstrator: \\Optimal Reward Estimation for Bandits}}}

\vspace*{.4in}

{\large{
\begin{tabular}{ccc}
Wenshuo Guo$^{\diamond}$, Kumar Krishna Agrawal$^{\diamond}$, Aditya Grover$^{\dagger}$, Vidya Muthukumar$^{\ddag}$,\\ Ashwin Pananjady$^{\ddag}$
\end{tabular}
}}
\vspace*{.3in}

\begin{tabular}{c}
$^\diamond$Department of Electrical Engineering and Computer Sciences, \\University of California, Berkeley\\
$^{\dagger}$Department of Computer Science, UCLA\\
$^\ddag$School of Electrical \& Computer Engineering and School of Industrial \& Systems Engineering, \\Georgia Institute of Technology
\end{tabular}

\vspace*{.2in}

\today

\vspace*{.2in}

\begin{abstract}
\input{sec_arxiv_v2/abstract}
\end{abstract}
\end{center}


\input{sec_arxiv_v2/intro}
\input{sec_arxiv_v2/related_work}

\input{sec_arxiv_v2/prelim}

\input{sec_arxiv_v2/main_theory}

\input{sec_arxiv_v2/upper_bound}
\input{sec_arxiv_v2/experiment}
\input{sec_arxiv_v2/conclu}

\subsection*{Acknowledgments}

We thank Krishna Acharya and Jim James for their careful reading of a draft of this paper, and for making several important suggestions.
We thank Kwang-Sung Jun for sharing the gene expression data that were used for a subset of the experimental results presented in this paper.
AG, VM, and AP were supported by research fellowships from the Simons Institute for the Theory of Computing when part of this work was performed. WG acknowledges support from a Google PhD Fellowship; AP acknowledges support from the National Science Foundation grant CCF-2107455.

\bibliographystyle{abbrvnat}
\bibliography{ref_arxiv}

\clearpage
\appendix
\begin{center}

  {\bf{\LARGE{Appendix}}}
\end{center}

In the following appendices, we collect proofs of all main results, and also present some additional numerical experiments. Throughout our proofs, we suppose that $T$ is greater than some absolute constant. We will use $c,C, c_1, C_1, \ldots$ to denote universal positive constants that may change from line to line.
We also define the shorthand notation
\begin{align}\label{eq:n0}
    \nn \defeq \nicefrac{4(T^{\alpha} - 1)}{\alpha \Delta_\aidx^2},
\end{align}
which will appear in multiple proofs and simplifies our exposition.

The appendices are organized as follows. 
Appendix~\ref{app:sec:proof-LB} provides the proof of Theorem~\ref{thm:LB}, our information-theoretic lower bound on reward estimation from a single demonstration of any algorithm.
Appendix~\ref{app:prelimlemmas} collects preliminary lemmas for general bandit algorithms that are used as building blocks in all subsequent proofs.
Appendix~\ref{app:prop-regret} provides, for completeness, proofs of high-probability regret bounds of SAE and UCB implemented with our inflated confidence widths.
Appendix~\ref{app:proof:thm:SAE-reward-estimator} provides the proof of Theorem~\ref{thm:SAE-UCB-reward-estimator}, our upper bound on reward estimation error, from a demonstration of the SAE algorithm, and Appendix~\ref{app:proof:thm:UCB-reward-estimator} provides the corresponding proof for the UCB case.
Finally, Appendix~\ref{app:exp} presents additional experimental details and results.

\input{sec_arxiv_v2/app_proof_LB}

\input{sec_arxiv_v2/prelim_lemmas}
\input{sec_arxiv_v2/forward_proofs}
\input{sec_arxiv_v2/app_proof_SAE}
\input{sec_arxiv_v2/app_proof_UCB}
\input{sec_arxiv_v2/app_exp}

\end{document}

%% file: wg_style.tex
\usepackage{diagbox}
\usepackage{makecell}
\usepackage{booktabs}
\usepackage{tikz,tikz-3dplot}
\usepackage{amsmath}
\usepackage{amsfonts}
\usepackage{amssymb}
\usepackage{amsthm}
\usepackage{url,ifthen}
\usepackage{srcltx}
\usepackage{multirow}
\usepackage{boxedminipage}
\usepackage{nicefrac}
\usepackage{xspace}
\usepackage{graphicx}
\usepackage{color}
\usepackage{colortbl}
\usepackage{setspace}
\usepackage{pgfplots}
\pgfplotsset{compat=1.12}
\usepackage{algorithm,algorithmic}

\usepackage{xcolor}
\definecolor{DarkGreen}{rgb}{0.1,0.5,0.1}
\definecolor{DarkRed}{rgb}{0.5,0.1,0.1}
\definecolor{DarkBlue}{rgb}{0.1,0.1,0.5}
\definecolor{Gray}{rgb}{0.2,0.2,0.2}

\definecolor{c1}{RGB}{38, 70, 83}
\definecolor{c2}{RGB}{42, 157, 143}
\definecolor{c3}{RGB}{233, 196, 106}
\definecolor{c5}{RGB}{231, 111, 81}
\definecolor{c4}{RGB}{244, 162, 97}

\definecolor{c1}{RGB}{38, 70, 83}
\definecolor{c2}{RGB}{42, 157, 143}
\definecolor{c3}{RGB}{233, 196, 106}
\definecolor{c5}{RGB}{231, 111, 81}
\definecolor{c4}{RGB}{244, 162, 97}

\usepackage{caption}
\usepackage{hyperref}
\hypersetup{
    unicode=false,          
    pdftoolbar=true,        
    pdfmenubar=true,        
    pdffitwindow=false,      
    pdfnewwindow=true,      
    colorlinks=true,       
    linkcolor=DarkBlue,          
    citecolor=DarkGreen,        
    filecolor=DarkRed,      
    urlcolor=DarkBlue,          
    pdftitle={},
    pdfauthor={},
}

\def\draft{1}

\def\submit{0}

\ifnum\draft=1 
    \def\ShowAuthNotes{1}
\else
    \def\ShowAuthNotes{0}
\fi

\ifnum\submit=1
\newcommand{\forsubmit}[1]{#1}
\newcommand{\forreals}[1]{}
\else
\newcommand{\forreals}[1]{#1}
\newcommand{\forsubmit}[1]{}
\fi
\newcommand{\chapterref}[1]{\hyperref[ch:#1]{Chapter~\ref{ch:#1}}}

\newcommand{\claimref}[1]{\hyperref[claim:#1]{Claim~\ref{claim:#1}}}

\newcommand{\corollaryref}[1]{\hyperref[cor:#1]{Corollary~\ref{cor:#1}}}

\newcommand{\definitionref}[1]{\hyperref[def:#1]{Definition~\ref{def:#1}}}

\newcommand{\equationref}[1]{\hyperref[eq:#1]{Equation~\ref{eq:#1}}}

\newcommand{\factref}[1]{\hyperref[fact:#1]{Fact~\ref{fact:#1}}}

\newcommand{\figureref}[1]{\hyperref[fig:#1]{Figure~\ref{fig:#1}}}

\newcommand{\tableref}[1]{\hyperref[tab:#1]{Table~\ref{tab:#1}}}

\newcommand{\itemref}[1]{\hyperref[item:#1]{Item~(\ref{item:#1})}}

\newcommand{\lemmaref}[1]{\hyperref[lem:#1]{Lemma~\ref{lem:#1}}}

\newcommand{\propref}[1]{\hyperref[prop:#1]{Proposition~\ref{prop:#1}}}

\newcommand{\propositionref}[1]{\hyperref[prop:#1]{Proposition~\ref{prop:#1}}}

\newcommand{\remarkref}[1]{\hyperref[rem:#1]{Remark~\ref{rem:#1}}}

\newcommand{\sectionref}[1]{\hyperref[sec:#1]{Section~\ref{sec:#1}}}

\newcommand{\theoremref}[1]{\hyperref[thm:#1]{Theorem~\ref{thm:#1}}}

\newcommand{\E}{\mathbb{E}}

\newcommand{\tv}{\text{TV}}

\renewcommand{\hat}{\widehat}

\newcommand{\cA}{{\cal A}}

\newcommand{\cE}{{\cal E}}
\newcommand{\cF}{{\cal F}}

\newcommand{\cM}{{\cal M}}

\newcommand{\cO}{{\cal O}}

\newcommand{\cS}{{\cal S}}

\newcommand{\defeq}{\stackrel{\small \mathrm{def}}{=}}
\renewcommand{\leq}{\leqslant}

\renewcommand{\geq}{\geqslant}

\usepackage{bm}

\newcommand{\Ind}{\mathbb I}

\newcommand{\ignore}[1]{}

\renewcommand{\epsilon}{\varepsilon}

\newcommand{\remove}[1]{}

\def\mydatewithyear{\leavevmode\hbox{\the\year/\twodigits\month/\twodigits\day}}
\def\twodigits#1{\ifnum#1<10 0\fi\the#1}

\def\twodigits#1{\ifnum#1<10 0\fi\the#1}

\makeatletter
\newtheorem*{rep@theorem}{\rep@title}
\newcommand{\newreptheorem}[2]{%
\newenvironment{rep#1}[1]{%
 \def\rep@title{#2 \ref{##1}}%
 \begin{rep@theorem}}%
 {\end{rep@theorem}}}
\makeatother

\newreptheorem{theorem}{Theorem}
\newreptheorem{corollary}{Corollary}
\newreptheorem{lemma}{Lemma}
\newreptheorem{proposition}{Proposition}

\newcounter{protocol}
\makeatletter
\newenvironment{protocol}[1][htb]{%
  \let\c@algorithm\c@protocol
  \renewcommand{\ALG@name}{Procedure}
  \begin{algorithm}[#1]%
  }{\end{algorithm}
}
\makeatother

%% file: macros_arxiv.tex
\newcommand{\numarms}{K}
\newcommand{\numpulls}{n}

\newcommand{\nalg}{n^{\alg}}
\newcommand{\aidx}{i}
\newcommand{\aset}{[K]}
\newcommand{\action}{I}
\newcommand{\reward}{r}
\newcommand{\samplemean}{\bar \mu}
\newcommand{\estmean}{\hat \mu}

\newcommand{\alg}{\mathcal{A}}
\newcommand{\nn}{\kappa_i}

\newcommand{\armdist}{\nu}
\newcommand{\bestmean}{\mu^\ast}
\newcommand{\armmean}{\mu}
\newcommand{\gap}{\Delta}
\newcommand{\Regret}{R}
\newcommand{\ci}{C}

\newcommand{\infoset}{\xi}
\newcommand{\instance}{\cM}

\newcommand{\muhat}{\widehat{\mu}}

\newcommand{\Ecal}{\mathcal{E}}

\newcommand{\Deltahat}{\widehat{\Delta}}
\newcommand{\ACal}{\mathcal{A}}
\newcommand{\kl}{\text{KL}}
\newcommand{\Bern}{\text{Bern}}

\newcommand{\Ecali}{\mathcal{E}^{(i)}}
\newcommand{\Ecalistar}{\mathcal{E}^{(i^*)}}
\newcommand{\Ecaliistar}{\mathcal{E}^{(i^*, i)}}
\newcommand{\Ecalij}{\mathcal{E}^{(j, i)}}
\newcommand{\EcalUCB}{\mathcal{E}_{\mathsf{UCB}}}
\newcommand{\EcalSAE}{\mathcal{E}_{\mathsf{SAE}}}
\newcommand{\tauhigh}{\overline{\tau}}

\newcommand{\tend}{\overline{t}}
\newcommand{\tepoch}{s}
\newcommand{\NUMRUNS}{100}

%% file: sec_arxiv_v2/abstract.tex
We introduce the ``inverse bandit'' problem of estimating the rewards of a multi-armed bandit instance from observing the learning process of a low-regret demonstrator. Existing approaches to the related problem of inverse reinforcement learning assume the execution of an optimal policy, and thereby suffer from an identifiability issue. In contrast, we propose to leverage the demonstrator's behavior en route to optimality, and in particular, the exploration phase, for reward estimation. We begin by establishing a general information-theoretic lower bound under this paradigm that applies to any demonstrator algorithm, which characterizes a fundamental tradeoff between reward estimation and the amount of exploration of the demonstrator. Then, we  develop simple and efficient reward estimators for upper-confidence-based demonstrator algorithms that attain the optimal tradeoff, showing in particular that consistent reward estimation---free of identifiability issues---is possible under our paradigm. Extensive simulations on both synthetic and semi-synthetic data corroborate our theoretical results. 

%% file: sec_arxiv_v2/intro.tex
\section{Introduction}\label{sec:intro}

Reward specification plays a crucial role in building safe and reliable machine learning systems that are aligned with human values~\citep{amodei2016concrete}. However, as pointed out in the extensive behavioral science literature, it is challenging to achieve this alignment, and hand-designed rewards are often misspecified~\citep{anderson2001behavioral, gershman2015novelty, macglashan2015between,bouneffouf2017bandit, gershman2018deconstructing}. 
The paradigm of \textit{inverse} reinforcement learning (IRL) presents a compelling workaround to explicit reward specification, and leverages the implicit optimality in expert demonstrations to infer a reward function.
In particular,
this paradigm places emphasis on behavioral demonstrations---that is, the demonstrator's actions themselves---as reflecting human values. Popular types of IRL include imitating the optimal policy~\citep{ho2016generative,li2017infogail}, apprenticeship learning~\citep{abbeel2004apprenticeship}, meta-learning~\citep{finn2017one} and learning the reward function (the original formulation of IRL)~\citep{ng2000algorithms,ramachandran2007bayesian,ziebart2008maximum, suay2016learning}.

Arguably the most outstanding challenge in reward-based IRL is that the reward function may not be uniquely identifiable from the agent's behavior, and infinitely many reward functions can explain the demonstrator's actions. 
This issue is particularly pronounced when we assume demonstrations from the \emph{optimal} policy~\citep{ng2000algorithms}, and subsequent work in IRL has developed heuristics to regularize the space of reward functions depending on how well they explain behavior
\citep{ziebart2008maximum, ramachandran2007bayesian}. Even so, some of these approaches, including maximum-entropy IRL~\citep{ziebart2008maximum}, still suffer from their own identifiability issues. 

The central message of this paper is that the reward identifiability issue can be alleviated even in the case where we have a \emph{single} demonstration, provided the demonstrator improves over time by exploration and then exploitation. In other words, such a demonstrator begins her trajectory facing an unknown environment, explores the environment through a sequence of actions, and eventually settles on an (approximately) optimal policy. 
Coincidentally, the original introduction of the IRL problem to the AI community did involve learning from this type of evolving demonstrator~\citep{russell1998learning}.
Concretely,~\cite{russell1998learning} frames the goal of IRL as being \emph{``to output the reward function that the agent is optimizing...given measurements of an agent’s behavior over time”}, and asks whether we can determine the reward function ``by observation \emph{during}, rather than \emph{after}, learning". Indeed, the process of policy improvement leaks information: \textit{when} the demonstrator ceases to use a suboptimal policy might contain useful signal about \textit{how} suboptimal that policy is. This, in turn, provides more information about the reward function than observations solely of the optimal policy.

We make this intuition formal and provide simple, tractable and optimal reward estimators from demonstrations in the \textit{multi-armed bandit (MAB)} setting that alleviate the aforementioned identifiability issue.
Note that the identifiability issue from optimal demonstrations is particularly acute in MAB: this is because no information about the suboptimal arms' rewards is revealed from the optimal demonstration, only the fact that they are suboptimal.
In addition to this conceptual motivation, studying the problem in the MAB setting also has several independent motivations. 
First and most directly, MAB forms the cornerstone of experiment design in several applications: two notable examples are hyperparameter selection in large-scale machine learning~\citep{li2017hyperband} and protocol selection for battery charging~\citep{attia2020closed}, where sequential experiment design is performed using popular, off-the-shelf bandit algorithms.
Being able to infer the utility of various alternative options from prior experimentation holds substantial value, as we can use this inference to assess the performance of all configurations that were involved in the experiment without actually rerunning the experiment itself (which may be expensive).
Second, it is also worth noting that humans frequently face MAB problems in the real world~\citep{anderson2001behavioral, bouneffouf2017bandit}.
It is often desirable to make inferences about their intrinsic preferences (e.g. a latent measure of customer utility) from observing their behavior, which can, in turn, be modeled from observing past interactions with a known environment. 

Our paradigm is applicable in both such cases. In contrast to learning purely from the exploitation phase in which the demonstrator pulls the optimal arm, we will use a model for the MAB algorithm---in particular, the temporal information revealed by the choices of which arms to pull over the course of the algorithm---to make inferences about the suboptimal arms.

\paragraph{Contributions.}
\begin{figure}
\centering
    \includegraphics[width=0.55\textwidth]{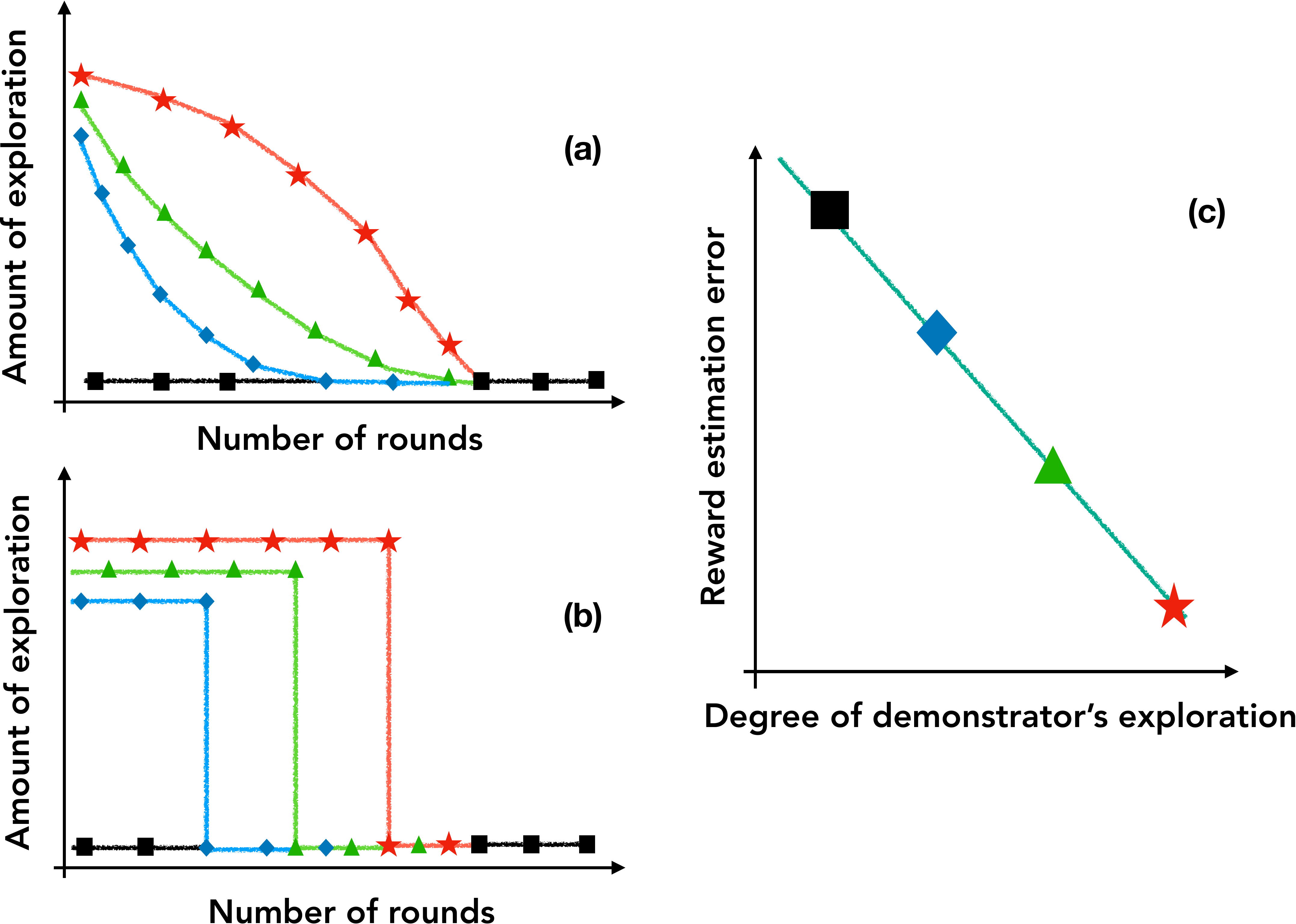}
    \caption{An illustration of exploration decreasing along the learning path for algorithms in the (a) upper-confidence-bound (UCB) family~\citep{auer2002finite} (b) successive-arm-elimination (SAE) family~\citep{even2006action}. (c) The tradeoff between exploration and reward estimation, signifying that algorithms that explore more are easier to estimate rewards from. }
    \label{fig:teaser}
\end{figure}

We formally introduce the inverse bandit problem and take a fundamental approach to it, providing both information-theoretic lower bounds and provably optimal algorithms. What is notable, and perhaps surprising, from our work is that the demonstrator's sequence of choices can reveal not only the relative suboptimality of arms but also the \emph{extent} of suboptimality, enabling consistent estimation of the reward of each arm from the behavior of a single demonstrator. 
To the best of our knowledge, this constitutes the first analysis of this type for inverse reward estimation, whether in bandits or RL. In more detail:

\textbullet \;\;We first derive information-theoretic lower bounds that apply to any demonstrator algorithm (Theorem~\ref{thm:LB}), which provide a quantitative tradeoff between exploration and reward estimation. This is illustrated schematically in Figure~\ref{fig:teaser}(c). In the special case of two arms, these bounds show that reward estimation error is inversely proportional to the square root of the \emph{regret} of the demonstrator's algorithm (Corollary~\ref{cor:two-armed}), thereby formalizing our claim from the abstract.

\textbullet \;\;We develop simple and efficient reward estimation procedures (Procedures~\ref{proc:SAE-estimator} and \ref{proc:UCB-estimator}) for demonstrations based on the popular \emph{successive-arm-elimination} (SAE)~\citep{even2006action} and \emph{upper-confidence-bound} (UCB)~\citep{lai1985asymptotically} algorithms, and prove upper bounds on the estimation error which match our lower bounds (Theorem~\ref{thm:SAE-UCB-reward-estimator}). 
Both algorithms can be naturally parameterized by their amount of in-built exploration, and are schematically represented in Figures~\ref{fig:teaser}(a) and~\ref{fig:teaser}(b), respectively.
In particular, these show that for either type of demonstrator, exploration can be optimally leveraged in reward estimation, even though the exploration schedule takes different forms\footnote{While Section~\ref{sec:upper_bound} provides exact details of both the SAE and UCB algorithms, we provide a short high-level description here. While SAE and UCB algorithms both trade off exploration and exploitation in a similar manner, their day-to-day behavior diverges sharply. In particular, SAE-based algorithms have a marked transition between their exploration and exploitation phases. On the other hand, in UCB-based algorithms the amount of exploration reduces smoothly with time (as reflected in Figure~\ref{fig:teaser}(b).}.


\textbullet \;\;Our theory is corroborated by extensive simulations and semi-synthetic experiments (e.g. on battery charging and gene expression datasets).

After discussing related work in the next subsection, we provide background on bandit algorithms and regret and formally state the inverse bandit problem in Section~\ref{sec:prelim}. Section~\ref{sec:main-results} contains statements and discussions of our main theoretical results, and we present and discuss our experiments in Section~\ref{sec:expts-main}. We conclude with a discussion of future work in Section~\ref{sec:conc}.

%% file: sec_arxiv_v2/related_work.tex
\paragraph{Related work.} 

\emph{Alternative IRL paradigms.} 
Recent work has explored easier settings that have the scope to avoid reward identifiability issues in IRL by assuming either additional structure on the reward or access to side-information~\citep{amin2017repeated, gershman2016empirical, geng2020identifying, ballard2019joint, jeon2020reward, fu2017learning}. In contrast to this line of work, our reward estimation procedure recovers the exact reward function in the limit, without additional assumptions, for a natural class of low-regret demonstrations.
A parallel line of work on learning from demonstrations, including imitation learning, studies alternative approaches that directly copy the demonstrators' actions without specifying a reward function~\citep{ho2016generative,li2017infogail}. While these approaches have had many successful applications, the lack of reward identification limits their use in others.
For instance, planning across multiple environments---with different transition dynamics---cannot be accomplished purely by imitation learning since the optimal policy can vary significantly. 
On the other hand, a learned reward function can be used to transfer knowledge across environments.



\emph{Learning from ``improving demonstrators".} 
Our paradigm of learning from an exploring demonstrator is similar in spirit to a line of recent work on learning from improving demonstrators~\citep{gao2018reinforcement, jacq2019learning, wu2019imitation, balakrishna2020policy, ramponi2020inverse}. We highlight two key differences in both the setting and theoretical scope. First, we show that reward estimation is possible from observing a \textit{single} demonstration, and consistent estimation is obtainable as the horizon of the demonstration grows. On the other hand, related work on learning from improving demonstrators is based on estimating population-based quantities arising from RL algorithms like soft policy iteration on gradient descent, and requires a large number of demonstrations for estimation.
At a lower level, our analysis proves 
not only consistency, but also
non-asymptotic guarantees on reward estimation that are matched by information-theoretic lower bounds.
On the other hand, there are no finite-sample guarantees available in the literature on learning from improving demonstrators, optimal or otherwise.
\smallskip

Finally, we mention that IRL in bandits has been considered by two recent papers, but the settings are motivated by social choice~\citep{Noothigattu_Yan_Procaccia_2021} and imitation/assisted learning~\citep{chanassistive2019}, as opposed to reward learning from a single demonstration.

%% file: sec_arxiv_v2/prelim.tex
\section{Preliminaries}\label{sec:prelim}

We formally define the problem of reward estimation in a multi-armed bandit (MAB) instance from a single demonstrator who uses a low-regret algorithm.
We begin by setting up standard notation for the MAB problem, and by formally defining (pseudo-)regret~\citep{lattimore2020bandit,slivkins2019introduction}.

\subsection{Multi-armed bandits (MAB) and regret minimization}

Consider a $K$-armed bandit instance with action (arm) set $\aset := \{1,2,\dots, K\}$. Every time an arm $\aidx \in \aset$ is pulled, a random reward is generated, independently of past actions, from an unknown probability distribution $\armdist_i$. 
We assume that the distribution $\armdist_i$ is supported on the interval $[0,1]$ for each $i \in \aset$, and denote by $\armmean_{i} = \mathbb{E}_{X\sim \armdist_i}[X]$ the expected reward of arm $\aidx$.
We assume throughout this paper that there is a unique best arm with highest expected reward. We let $i^\ast := {\arg \max}_{i \in \aset} \armmean_i$ denote the index of the best arm, and use $\bestmean := \armmean_{i^*}$ to denote its reward.
We refer to the remaining arms as suboptimal arms, and define the \textit{suboptimality gap} of arm $i \neq i^*$ to be $\gap_{i}:=\bestmean-\armmean_{i}$.
Owing to the uniqueness of the best arm, note that $\gap_i > 0$ for all $i \neq i^*$.


The demonstrator takes actions on the bandit instance with the goal of maximizing her accumulated reward over a finite horizon consisting of $T$ rounds. At each round $t\in \{1,2,\dots, T\}$ and based on her observations thus far, the demonstrator pulls an arm $\action_{t} \in \aset$ and receives a reward $\reward_{t} \sim \armdist_{\action_{t}}$. Define $\numpulls_{i,t}$ to be the number of times arm $\aidx$ has been pulled up to time~$t$. It is also useful to define the empirical reward estimate of arm $\aidx$ at time $t$ as $\samplemean_{\aidx,t} = \nicefrac{\big( \sum_{s=1}^t \reward_{s} \cdot \mathbb{I}\{\action_s = \aidx\}  \big)}{\numpulls_{\aidx,t}}$.
The performance of the demonstrator is measured by her \emph{regret}, which quantifies the difference between the best possible reward she could accrue if she knew which arm was the best one, and the actual accumulated reward. 

\begin{definition}[Pseudo-regret~\citep{lattimore2020bandit}] \label{def:regret}
The (expected) pseudo-regret is
\begin{align*}
\E[\Regret_T] &= T \bestmean - \mathbb{E}\Big[\sum_{t=1}^T \reward_{t}\Big] = \sum_{\aidx \in \aset}\gap_{i}\mathbb{E}[\numpulls_{\aidx,t}].
\end{align*}
\end{definition}
%
A \emph{low-regret} demonstrator is one whose regret scales \emph{sublinearly} in $T$ with high probability, that is, $\Regret_T = o(T)$ with probability that goes to $1$ as $T \to \infty$.

\subsection{The ``inverse bandit" problem}

The ``inverse bandit" problem is to estimate the expected rewards $\{ \mu_i \}_{i \in [K]}$ of a multi-armed bandit instance from observing only the actions of a demonstrator algorithm. Importantly, 
we \textit{do not} observe the rewards accrued at each round.
%
Consider a demonstration consisting of the sequence of actions $\{\action_t\}_{t = 1}^T$. A reward estimation \emph{procedure} is a mapping from  $\{\action_t\}_{t = 1}^T \mapsto \{\muhat_\aidx\}_{\aidx \neq i^*}$, where $\muhat_i$ denotes the mean estimate arm $i$. The goal of the reward estimation procedure is to minimize the expected estimation error for each arm\footnote{Our guarantees are most natural to state on the stringent arm-by-arm error metric, and yield $\ell_p$ guarantees.}  $i$, given by $\E[|\estmean_{\aidx} - \armmean_{\aidx}|]$. Here, the expectation is taken over the randomness of the received rewards and the sequence of actions. Furthermore, since the behavior of any natural demonstrator 
is invariant to constant shifts of all expected rewards, 
we assume that the procedure has access to the value of $\mu^*$ (but not the index $i^*$) to avoid trivial identifiability issues. Note that one can remove this recentering assumption and instead consider estimating the suboptimality gaps.

Note that this goal of \textit{estimation} is significantly more challenging than simply \textit{ranking} the arms: the latter problem is solvable by ordering the arms according to their pull counts, but does not produce cardinal reward values. Nevertheless, we will show shortly that reward estimation can indeed be performed from observing a single trajectory from a natural class of demonstrators.

%% file: sec_arxiv_v2/main_theory.tex
\section{Fundamental Limits on Reward Learning} \label{sec:main-results}

To provide a concrete baseline, we first prove information-theoretic lower bounds showing a fundamental tradeoff between reward estimation and exploration, regardless of the specific reward estimation procedure and the demonstrator's learning algorithm. 
At a high level, the identifiability issues that arise in IRL already suggest that exploration is necessary for nontrivial reward estimation; our lower bound makes this formal. 
We then present some intuitive but unsuccessful attempts to achieve this lower bound.

\subsection{Information-theoretic lower bound} \label{sec:LB}

%
The following theorem collects our lower bound.

\begin{theorem}(Proof in Appendix~\ref{app:sec:proof-LB})\label{thm:LB} For every $K$-armed Bernoulli bandit instance $\instance$ satisfying $\max_{i \in [K]} |\mu_i - 1/2| \leq 1/4$ and for each suboptimal arm $i \neq i^*$, the following is true. Suppose that the demonstrator employs algorithm $\alg$, and let $\E[\nalg_{i, T}]$ denote the expected number of times arm $i$ is pulled by $\mathcal{A}$ when presented the instance $\instance$.
Then there exists an instance $\instance'$ such that for any reward estimation procedure having knowledge of $\armmean^\ast$ and mapping $\{\action_1, \ldots, \action_T\} \mapsto \{\muhat_i\}_{i \in \aset, i \neq i^*}$,
\begin{align*}
\max_{\widetilde{\instance} \in \{\instance, \instance'\} }\E[|\estmean_{\aidx}- \armmean_\aidx(\widetilde{\instance})|] \geq \frac{1}{16}\cdot \left(\frac{1}{\sqrt{\E[\nalg_{\aidx, T}]}} \land 1 \right). 
\end{align*}
Here $\mu_i(\mathcal{M})$ denotes the $i$-th reward mean of the bandit instance $\mathcal{M}$.
\end{theorem}



Note that in addition to applying to any reward estimation procedure, Theorem~\ref{thm:LB} provides a fundamental limit for any choice of demonstrator algorithm in terms of the degree of exploration in that algorithm. 
Its proof utilizes information-theoretic lower bounds on the demonstrator's regret~\citep{kaufmann2016complexity}: even with the strong side information of noisy reward observations, we need sufficiently many pulls of arm $\aidx$ to be able to estimate its reward, since zero information is shared across arms in the MAB setting.
Thus, the efficacy of any inverse procedure for estimating $\mu_i$ is fundamentally limited by $\E[n_{i, T}]^{-1/2}$. 



\subsection{Some initial observations} 
Theorem~\ref{thm:LB} constitutes a fundamental limit on reward estimation from \emph{any} demonstrator algorithm, even if we know the algorithm beforehand. We now make some observations to help assess the types of demonstrator algorithms that allow us to match this lower bound.

\paragraph{The algorithm needs to satisfy instance-adaptivity.}
Ideally, we would aim to obtain reward estimation guarantees from \emph{any} plausible low-regret algorithm. Unfortunately, such a general statement cannot be true (even if we are satisfied with a worse estimation error bound) as witnessed by the following simple counterexample. 
Suppose that the demonstrator employs the explore-then-commit algorithm~\citep{lattimore2020bandit} which pulls arms randomly for $\cO(T^{2/3})$ rounds, and then pulls the arm with the highest estimated mean reward thereafter. 
This algorithm achieves regret $\cO(T^{2/3})$ for all bandit instances, and so constitutes a no-regret algorithm. 
However, it is easy to see (since the arm pulls provide no information about the rewards themselves) that nontrivial reward estimation is impossible from observing the actions alone. 
As this example shows, reward estimation is only possible when the algorithm exhibits some type of instance-dependent behavior (e.g., if the action sequence differs when the suboptimality gaps change). 


\paragraph{Does order-wise instance-optimal regret suffice?}
The next natural question that arises is whether it is possible to estimate the rewards from any algorithm that exhibits (order-wise) optimal instance-dependent behavior, \emph{even when we do not know the specific details of the algorithm}.
In particular, we might hope to use the number of pulls of a suboptimal arm by round $T$, which we denoted by $n_{i,T}$, as a sufficient statistic for our estimation procedure.
For example, classic instance-dependent bounds are of the form $n_{i, T} = \Theta \left(\frac{\log T}{\Delta_i^2}\right)$, where the constant inside the $\Theta(\cdot)$ varies across arms.
A possible estimator from this relation would be to construct $\Deltahat_i = C_0 \sqrt{\frac{\log T}{n_{i, T}}}$ for each suboptimal arm $i$, for some choice of common constant $C_0$.
While this estimator possesses the attractive property of being algorithm-agnostic, it turns out to not even be \emph{statistically consistent} (with respect to the number of rounds $T$), let alone match the fundamental limit given by Theorem~\ref{thm:LB}.
In fact, an elementary analysis verifies that $|\Deltahat_i - \Delta_i| = \Theta\left(\sqrt{\frac{\log T}{n_{i, T}}}\right) = \Theta(1)$, and so the estimation error does not decay with $T$. 
%
At a high level, such a ``naive" estimator 
does not effectively exploit the day-to-day structure present in a demonstrator algorithm, and consequently cannot match the lower bound in Theorem~\ref{thm:LB} (also see Appendix~\ref{app:exp}).
\smallskip

Our lower bound and preliminary observations motivate a class of procedures that utilizes\footnote{In addition to this conceptual motivation, assuming knowledge of the demonstrator's algorithm is reasonable, e.g., in experiment design settings where algorithms like UCB constitute the ``gold standard".} the characteristics of structured, instance-adaptive algorithms like \emph{successive-arm-elimination} (SAE)~\citep{even2006action} and \emph{upper-confidence-bounds} (UCB)~\citep{lai1985asymptotically} \nocite{auer2002finite} to perform reward estimation.

\begin{minipage}[ht]{0.46\textwidth}
\begin{algorithm}[H]
\begin{algorithmic}[1]
\caption{Successive arm elimination (SAE) with $O(T^\alpha)$ regret (for $0 < \alpha < 1$) or $O(\log T)$ regret (for $\alpha = 0$)}\label{alg:SAE-any-regret}
\STATE \textbf{Input:} $K$ arms, $\alpha \in [0,1)$, total rounds $T$.\\[1ex]
\STATE \textbf{Initialize:} Set SAE epoch $t_r = 1$, active set $\cS(1) \gets [\numarms]$ and round $t=0$. 
\WHILE{$|S(t_r)| > 1$}
\STATE{Sample arm $\aidx \in S(t_r)$ once and set $ t \gets t + 1$}
\STATE {Let $\samplemean_{\aidx, t}$ be the average reward of arm $\aidx$ by $t$}
\STATE{Set $\ci_{\aidx,t} \defeq \sqrt{\frac{2(T^\alpha-1)}{\alpha \cdot t_r}}$.}
	\FOR{\textbf{ each }$\aidx \in \cS(t_r)$ and \mbox{$\samplemean_{\aidx, t} \leq \samplemean_{\max}(t) - 2 \ci_{\aidx, t}$}}\label{alg:step:sae-check}
		\STATE $\cS(t_r) \gets S(t_r)\setminus \{i\}$.
    \ENDFOR
    \STATE{$t_r \gets t_r + 1$}
\ENDWHILE\\[1ex]
\STATE{Pull arm in $\cS$ and set $ t \gets t + 1$ until $t = T$.}

\end{algorithmic}
\end{algorithm}
\end{minipage}
\begin{minipage}[ht]{0.44\textwidth}
\begin{algorithm}[H]
\begin{algorithmic}[1]
\caption{Upper confidence bound (UCB) with $O(T^\alpha)$ regret (for $0 < \alpha < 1$) or $O(\log T)$ regret (for $\alpha = 0$)} \label{alg:UCB-any-regret}
\STATE \textbf{Input:} $K$ arms, $\alpha \in [0,1)$, total rounds $T$.\\[1ex]
\STATE \textbf{Initialize:} Set round $t=1$. Set for every arm a confidence width $\ci_{\aidx,0} = \infty$. \\[1ex]
\WHILE{$t<T$}

\STATE Pull arm $\action_t = \argmax_{i \in [K]} \samplemean_{\aidx,t-1} + \ci_{\aidx,t-1}$ (break ties arbitrarily). 
\STATE Let $\samplemean_{\aidx, t}$ be the average reward of arm $\aidx$ by time $t$, and let $\numpulls_{\aidx, t}$ be the number of times arm $\aidx$ is pulled by time $t$.
\STATE Set $\ci_{\aidx,t} \defeq \sqrt{\frac{2(T^\alpha-1)}{\alpha \cdot \numpulls_{\aidx,t}}}$
\STATE $t \gets t+1$
\ENDWHILE 
\end{algorithmic}
\end{algorithm}

\textit{Note: When $\alpha = 0$, we use that $\lim_{\alpha \to 0} \frac{T^\alpha - 1}{\alpha} = \log T$.}
\end{minipage}

%% file: sec_arxiv_v2/upper_bound.tex
\section{Optimal Reward Estimators} \label{sec:upper_bound}


Two popular families of algorithms in the MAB literature are \textit{successive-arm-elimination} (SAE) and \textit{upper-confidence-bounds} (UCB), presented formally in Algorithms~\ref{alg:SAE-any-regret} and~\ref{alg:UCB-any-regret}. While these algorithms differ in their round-by-round details, they are both based on the principle of optimism in the face of uncertainty, whereby exploration is encouraged by constructing an ``optimistic" upper-confidence-bound on the reward of an arm that is a decreasing function of the number of times that arm has been pulled thus far.

The SAE algorithm proceeds in multiple epochs; in each epoch, all active arms are pulled in a round robin fashion and their sample means are maintained. 
As soon as we observe that a certain arm is obviously suboptimal, we drop it from consideration and render it ``inactive", or eliminated.
The UCB algorithm instead intertwines exploration with exploitation.

\begin{remark}
The use of $\alpha$ in Algorithms~\ref{alg:SAE-any-regret} and~\ref{alg:UCB-any-regret} is only to obtain a more general class of algorithms with a smooth variation in their regret.
In particular, a higher value of $\alpha$ essentially inflates the confidence intervals, allowing for greater exploration.
Indeed, both the SAE and UCB algorithm incur a sublinear regret of $O(T^\alpha)$ in high probability for any $\alpha \in [0,1)$ (the statement and proof of this result is in Appendix~\ref{app:prop-regret} for completeness). 
The smaller the value of $\alpha$, the smaller the regret and---from the fundamental limits that we characterized in Theorem~\ref{thm:LB}---the harder it is to perform reward estimation.
The typical choice of $C_{i,t}$ is $\mathcal{O} \left( \sqrt{\frac{\log T}{\numpulls_{\aidx,t}}} \right)$, which yields the instance-optimal regret guarantee $\mathcal{O}\left(\sum_{i \neq i^*} \frac{\log T}{\Delta_i^2}\right)$, is recovered by taking the limit $\alpha \to 0$. 
It is important to note that we obtain consistency of estimation even in this case of minimal exploration, and our main ideas are already evident here.
\end{remark}





\subsection{Optimal reward estimation} \label{sec:rewardestimation}

The naive attempts from before suggest that one needs more delicate procedures in order to an optimal (or even consistent) estimator. We now present such estimators for the SAE and UCB algorithms, starting with SAE since the ideas are most intuitive when there is a clear separation between exploration and exploitation.

\paragraph{SAE reward estimator.}
Note from the description of SAE in Algorithm~\ref{alg:SAE-any-regret} that the transition from exploration to exploitation is particularly abrupt: for every arm $i$, there exists (with high probability) a round $\tau_i$ at which the condition for arm $i$ to be eliminated is met.
More formally, for a typical execution of SAE given by $\{I_1,\ldots,I_T\}$, we define this ``switching round" as
\begin{align}\label{eq:switchingroundSAE}
    \tau_i := \{t \geq 1: I_t = i\text{ and }  I_{t'} \neq i \;\forall t' > t\} .
\end{align}

\begin{figure}
\begin{minipage}{0.45\textwidth}
\begin{protocol}[H]
\begin{algorithmic}[1]
\caption{SAE reward estimator}\label{proc:SAE-estimator}
\STATE{\textbf{Input:} Sequence of actions $\{\action_1, \ldots, \action_T\}$; scalar $\mu^\ast$.\\[1ex]}
\STATE{Set $\hat{\imath} \in \argmax_{\aidx} \numpulls_{\aidx, T}$}
\FOR{\textbf{each }$\aidx \in [K], i \neq \hat{\imath}$:}
    \STATE{Compute $\tau_\aidx$ according to Eq.~\eqref{eq:switchingroundSAE}} 
\STATE{$\estmean_\aidx \defeq \bestmean - 2\cdot \ci_{\aidx,\tau_{\aidx}}.$}
\ENDFOR
\RETURN{$\estmean_\aidx$ for $\aidx \in [K]$.}
\end{algorithmic}
\end{protocol}
\end{minipage}
\hfill
\begin{minipage}{0.45\textwidth}
    \centering
    \includegraphics[width=\textwidth]{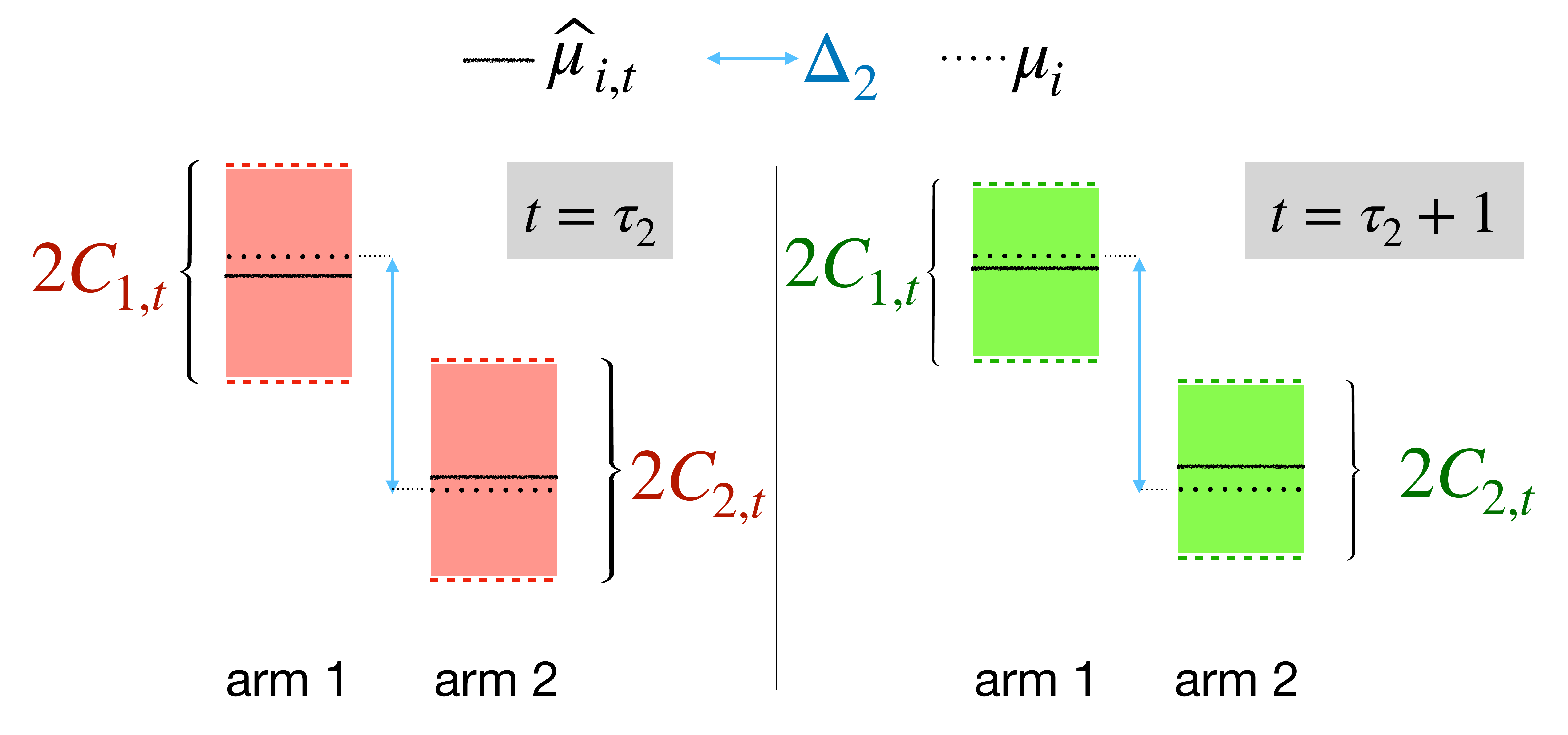}
    \caption{SAE on a $2$-armed bandit instance at the rounds $\tau_2$ and $\tau_2 + 1$. In Procedure~\ref{proc:SAE-estimator}, we exploit the fact that on both left and right $\Delta_2 \approx C_{1, t} + C_{2, t}$.}
    \label{fig:SAE-estimator} 
\end{minipage}
\end{figure}
Procedure~\ref{proc:SAE-estimator} estimates the suboptimality gap of arm $\aidx$ by exactly twice the width of the confidence interval at the switching round $\tau_i$, denoted by $C_{i,\tau_i}$.
Figure~\ref{fig:SAE-estimator} provides three-fold intuition for why this simple estimator is reasonable in the simplest case of $2$ arms (with $i^* = 1$).

First, at round $\tau_2$ arm $2$ is still in play; so the sum of the confidence widths $C_{1,\tau_2 + 1} + C_{2,\tau_2 + 1}$ must \emph{upper bound} the difference in sample means $\samplemean_{1,\tau_2} - \samplemean_{2,\tau_2}$. This is depicted on the left hand side of Figure~\ref{fig:SAE-estimator}.
Second, at round $\tau_2 + 1$ the condition for elimination of arm $2$ must be met; so the sum of the confidence intervals $C_{1,\tau_2 + 1} + C_{2,\tau_2 + 1}$ must \emph{lower bound} the difference in the algorithm's sample means $\samplemean_{1,\tau_2 + 1} - \samplemean_{2,\tau_2 + 1}$.
This is depicted on the right hand side of Figure~\ref{fig:SAE-estimator}.
%
%
Putting these together, we obtain an estimator that is close to the difference in sample means $\samplemean_{1,\tau_2 +1} - \samplemean_{2,\tau_2+1}$ (which is in turn very close to $\samplemean_{1,\tau_2} - \samplemean_{2,\tau_2}$).
Finally, since both arm $1$ and arm $2$ have been active until switching round $\tau_2$, their confidence widths are identical.
This leads to the particularly simple description of the SAE estimator in Procedure~\ref{proc:SAE-estimator}.

\paragraph{UCB reward estimator.}
While the details are significantly more complex for UCB, a similar idea works. 
In this case
suboptimal arms could be pulled throughout the decision-making process, but there will still exist (with high probability) a maximal round at which arm $i$ is pulled \textit{and} the optimal arm is pulled at least once there-after.
Let $\hat{\imath}$ denote the index of the arm that is pulled most often in the demonstration; this is our estimate of the optimal arm.
The switching round of interest is given by
\begin{align}\label{eq:switchingroundUCB}
\tau_i := \max\{ t: I_t = i \text{ and } I_{t'} = \hat{\imath} \text{ for some } t' > t \} .
\end{align}
Then, Procedure~\ref{proc:UCB-estimator} directly estimates the reward of arm $\aidx$ by exactly the difference in confidence widths of arms $i$ and $i^*$ at $\tau_i$.
As illustrated in Figure~\ref{fig:UCB-estimator} for the case of $2$ arms, the confidence widths can be significantly different for the optimal and suboptimal arm at the switching round for the case of UCB.
However, similar intuition as in the case of the SAE estimator continues to hold here; once again, arm $2$ is suboptimal and we work on the high-probability event that $i^* = \hat{\imath} = 1$.
First, at round $\tau_2$ the upper confidence bound of arm $2$ must exceed that of arm $1$; therefore, the difference in confidence widths must upper bound the difference in sample means.
Second, at round $t'$ the upper confidence bound of arm $1$ exceeds that of arm $2$; therefore, the difference in confidence widths must lower bound the difference in sample means.
Putting these together, we again obtain an estimator that is close to the difference in sample means $\samplemean_{1,\tau_2} - \samplemean_{2,\tau_2}$.
\begin{figure}
\begin{minipage}{0.45\textwidth}
\begin{protocol}[H]
\begin{algorithmic}[1]
\caption{UCB reward estimator}\label{proc:UCB-estimator}
\STATE{\textbf{Input:} Sequence of actions $\{\action_1, \ldots, \action_T\}$; scalar $\mu^\ast$.\\[1ex]}
\STATE{Set $\hat{\imath} \in \argmax_{\aidx} \numpulls_{\aidx, T}$}
\FOR{\textbf{each }$\aidx \in [K], i \neq \hat{\imath}$:}
\STATE{Compute $\tau_\aidx$ according to Eq.~\eqref{eq:switchingroundUCB}}
\STATE{$\estmean_\aidx \defeq  \bestmean - ( \ci_{\aidx,\tau_{\aidx}} - \ci_{\hat{\imath},\tau_{\aidx}}).$}
\ENDFOR
\RETURN{$\estmean_\aidx$ for $\aidx \in [K]$.}
\end{algorithmic}
\end{protocol}
\end{minipage}
\hfill
\begin{minipage}{0.45\textwidth}
    \centering
    \includegraphics[width=\textwidth]{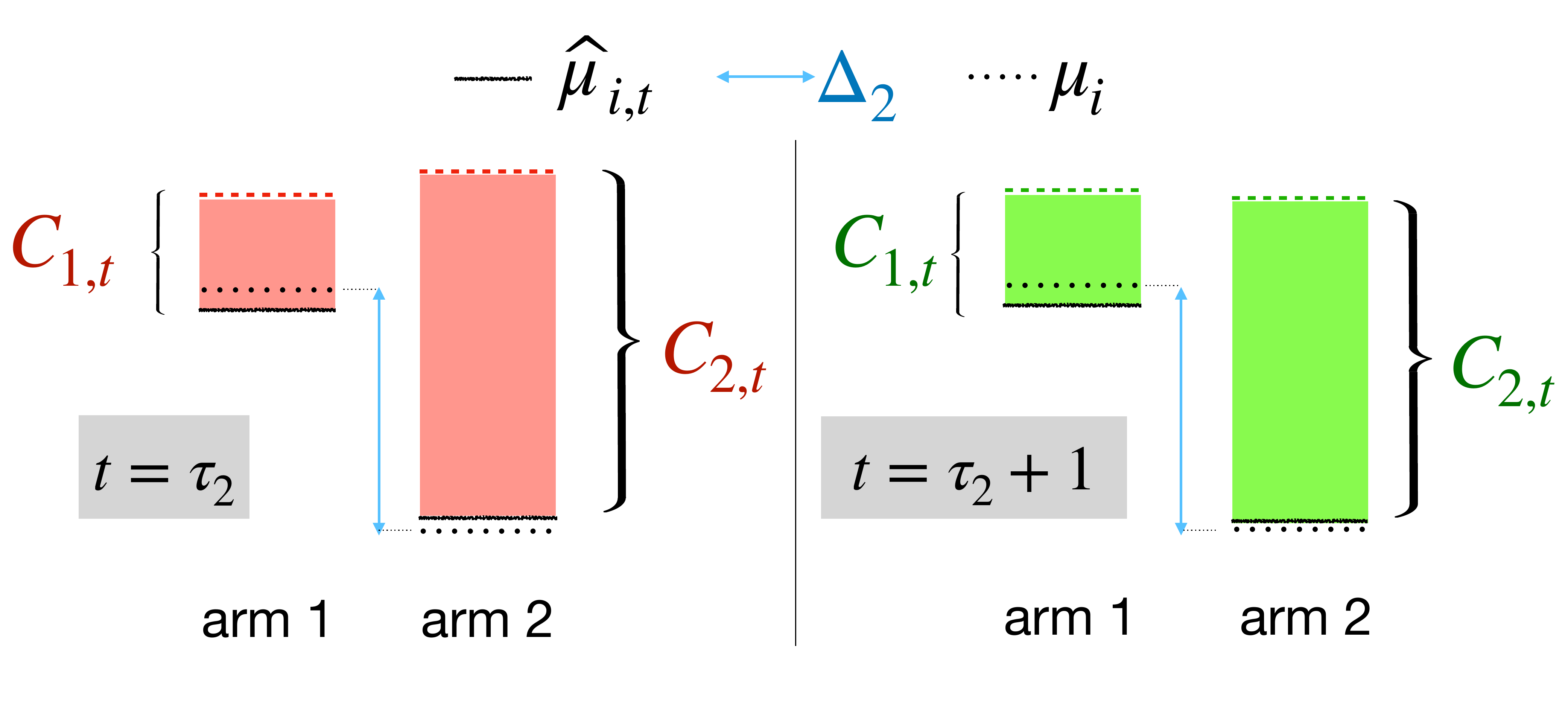}
    \caption{UCB on a $2$-armed bandit instance at the rounds $\tau_2$ and $\tau_2 + 1$. In Procedure~\ref{proc:UCB-estimator}, we exploit the fact that on both left and right $\Delta_2 \approx C_{2, t} - C_{1, t}$.}
    \label{fig:UCB-estimator}
\end{minipage}
\end{figure}

Our main theorem makes the above intuition precise and obtains a unified characterization of the estimation error $|\estmean_i - \armmean_i|$ for each $i \in \aset$ arising from demonstrations of either SAE or UCB.
\begin{theorem}(Proof in Appendix~\ref{app:proof:thm:SAE-reward-estimator} for SAE, Appendix~\ref{app:proof:thm:UCB-reward-estimator} for UCB)\label{thm:SAE-UCB-reward-estimator}
Suppose\footnote{This condition ensures, by Proposition~\ref{prop:SAE-UCB-regret}, that $\hat{\imath} = i^*$ with high probability.} $T \geq 64 \sum_{i \neq i^*} \frac{T^\alpha - 1}{\alpha \Delta_i^2}$, and let $n_{i, T}$ denote the number of times arm $i$ is pulled by either Algorithm~\ref{alg:SAE-any-regret} or~\ref{alg:UCB-any-regret}. Denote the total number of arms as $K$.
There is a universal positive constant $C$ such that for any suboptimal arm $\aidx$, Procedures~\ref{proc:SAE-estimator} and~\ref{proc:UCB-estimator} satisfy
\[\E|\estmean_\aidx - \armmean_\aidx| \leq C \sqrt{\frac{\log (\E[n_{i, T}]\sqrt{K})}{\E[n_{i, T}]} }.
\]
Furthermore, we have $\E[n_{i, T}] \geq c \cdot \frac{T^{\alpha} - 1}{\alpha \Delta_i^2}$ for a second universal constant $c > 0$.
\end{theorem}

Since the map $x \mapsto \log x/x$ is decreasing for large enough $x$, the two parts of the theorem also provide an upper bound on the estimation error purely in terms of the the tuple $(T, \alpha, \Delta_i)$. Nevertheless, we have chosen to state it in terms of the expected number of pulls of arm $i$ so as to bring into sharp focus the effect of exploration on reward estimation. Note that $\E[n_{i, T}]$ measures the degree to which the suboptimal arm $i$ is explored; Theorem~\ref{thm:SAE-UCB-reward-estimator} shows that a larger value of $n_{i, T}$ will lead to a smaller error. 
The precise quantitative relationship is also compelling: indeed, if we had oracle access to the reward samples accrued over the course of the demonstration, simply averaging them and outputting the sample mean would achieve a rate of the order $n_{i, T}^{-1/2}$. The theorem shows that a similar rate is achievable \textit{solely} using observations of the trajectory itself. 

\paragraph{The role of algorithmic hyperparameters.} Our procedures were based on knowing not just the particular type of demonstrator algorithm being employed but also its hyperparameters (since these were used to construct the confidence intervals). It is natural to ask if the latter assumption can be relaxed. We note that even without the knowledge of the constants in the confidence widths, the same reward estimation procedures will still able to estimate the suboptimality gaps up to a scaling constant that is \emph{common} to all arms. In particular, such a guarantee would suffice to argue statements of the form ``the second arm is twice as suboptimal as the third"; such relative comparisons of the arms' rewards are often sufficient in many applications.

\begin{figure*}[t]
  \begin{center}
    \begin{tabular}{cccc}
          \includegraphics[width=0.215\textwidth]{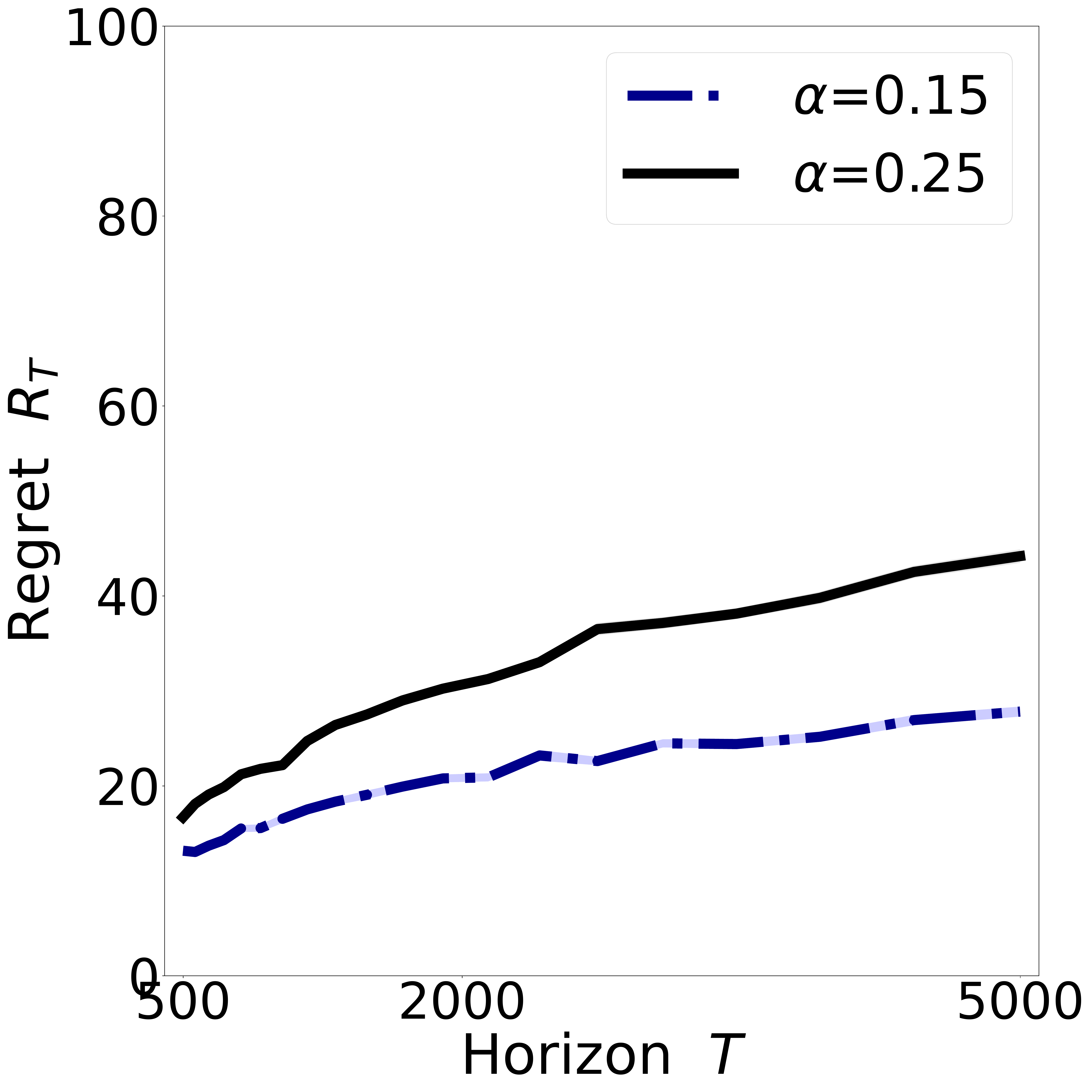} & \includegraphics[width=0.215\textwidth]{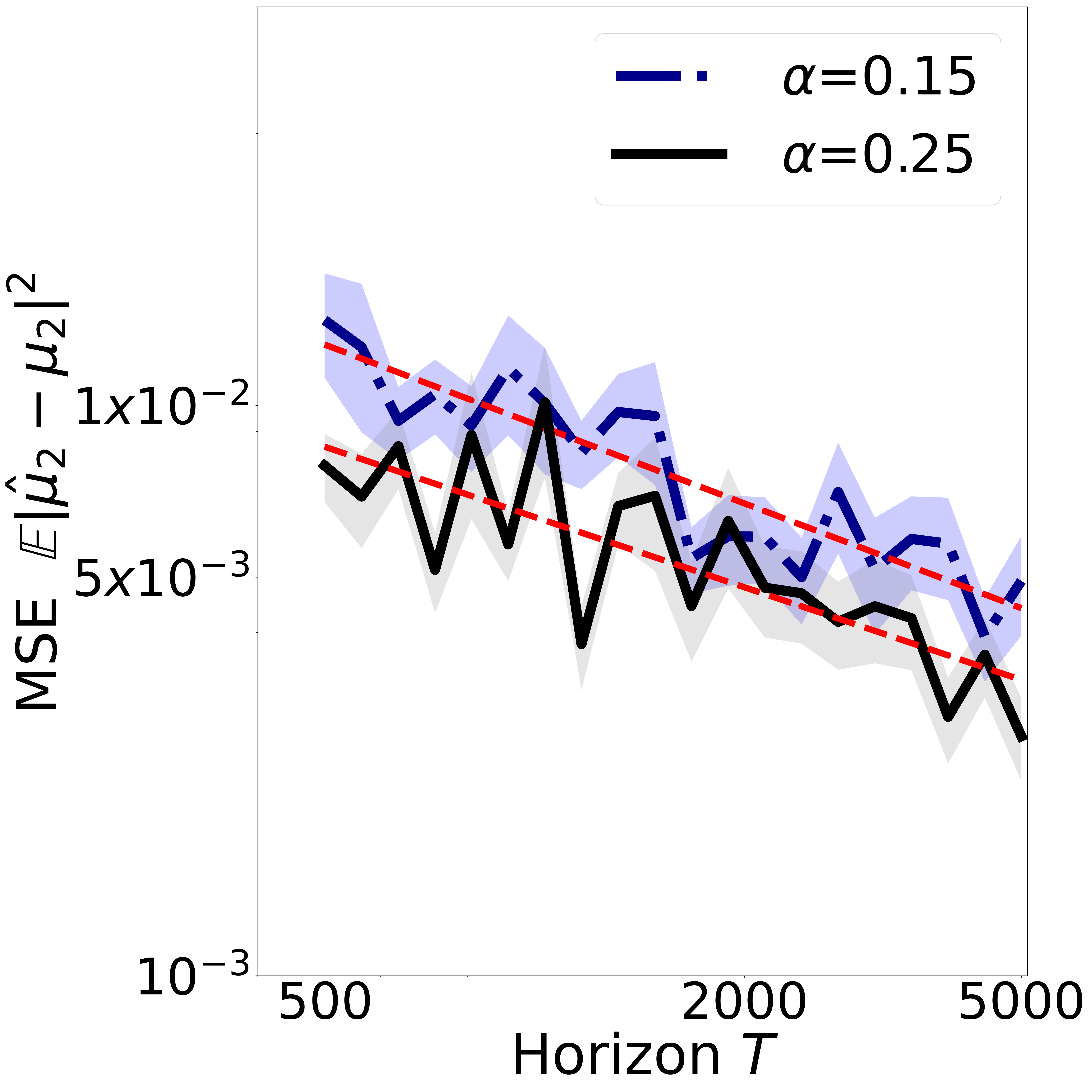} & \includegraphics[width=0.215\textwidth]{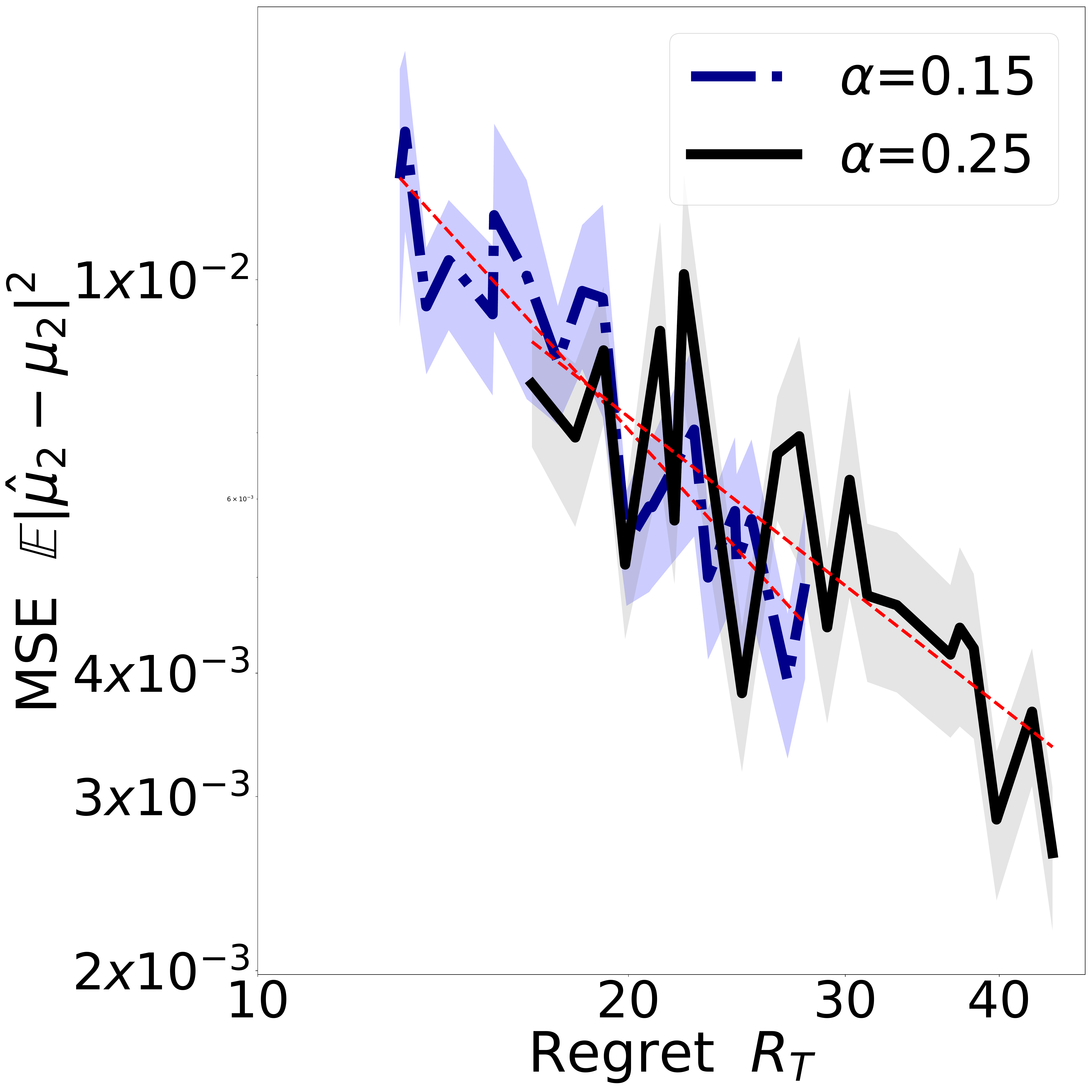} &  \includegraphics[width=0.215\textwidth]{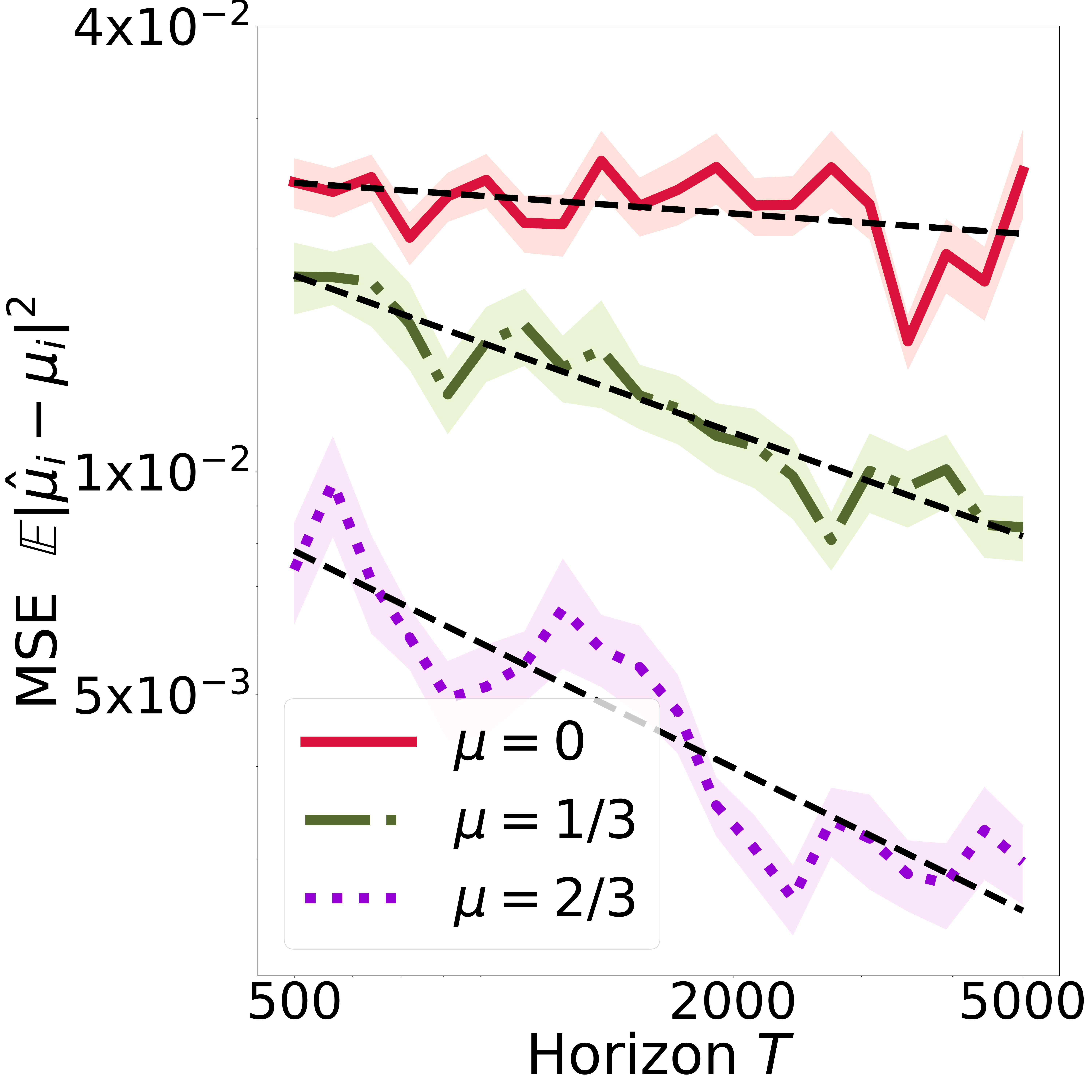} \\
          \; (a) & \;\; (b) & \;\;\; (c) & \;\; (d)
    \end{tabular}
    \caption{Results of $\NUMRUNS$ runs of simulation experiments for the UCB algorithm. Figures (a-c) are for a two-armed bandit instance with $\mu = (1, 1/2)$ and Gaussian rewards with unit variance. Here, individual curves represent two values of $\alpha \in \{0.15, 0.25\}$. Figure (d) is a $4$-armed instance with  $\mu = (1, 2/3, 1/3, 0)$ and Gaussian rewards with variance $1/4$. Here, individual curves represent the three suboptimal arms. Overall, these log-log plots corroborate our principal finding that better reward estimation is achievable from higher regret demonstrations; see the text for a detailed discussion.}
    \label{fig:all_sim_experiments}
  \end{center}
\end{figure*}

\paragraph{Technical novelty.} Let us make a few comments on the technical difficulties involved in proving Theorem~\ref{thm:SAE-UCB-reward-estimator}. Figures~\ref{fig:SAE-estimator} and~\ref{fig:UCB-estimator} suggest that the estimated gap closely tracks the difference in sample means $\samplemean_{\hat{\imath}, \tau_i} - \samplemean_{i, \tau_i}$. The first step is to make this precise: we show that 
in both cases, the overall estimation error is characterized, up to lower order terms, by 
the distance from the sample means to the true means at $\tau_i$.
The second step is to characterizing the sample-mean estimation error, and  is challenging for a number of reasons: (1) the sample means both in UCB and SAE are biased even for a fixed round $t$ due to adaptive sampling~\citep{nie2018adaptively,shin2019sample} (2) the switching round $\tau_i$ is itself random for both UCB and SAE, and (3) in the case of UCB, there is a discrepancy between the quantities $n_{i,\tau_i}$ and $n_{\hat{\imath},\tau_i}$. 
Substantial technical effort in our proofs goes into constructing high-probability lower-bounds on $n_{i,\tau_i}$ and $n_{\hat{\imath},\tau_i}$, both of the order of $\E[n_{i,T}]$.
The lower bound on $n_{\hat{\imath},\tau_i}$ appears to be the first of its kind, and does not follow even from other lower bounds on the total number of pulls of each arm derived in the literature~\citep{syed2010adapting}.
Instead, it requires a fine-grained understanding of the \emph{day-to-day} behavior of UCB.
We 
present a case-by-case analysis of UCB to provide these high-probability lower bounds, which may be of independent interest. 


\subsection{A consequence: reward estimation / regret tradeoffs for two-armed bandits}

A key message of our results is that more exploration in the demonstration is both necessary and sufficient for efficient reward estimation, in an arm-by-arm sense. 
In the special case of a two-armed bandit problem, this tradeoff can be expressed solely in terms of the regret:

\begin{corollary}(Informal) \label{cor:two-armed}
Let $i^* =1$. Procedures~\ref{proc:SAE-estimator} and~\ref{proc:UCB-estimator} achieve, from a demonstration of SAE or UCB with expected regret $\E[R_T]$, the bound
$\E[| \muhat_2 - \mu_2 |] \lesssim \sqrt{\frac{\Delta_2}{\E[R_T]}}.$ Conversely, any reward estimator $\muhat_2$ from a demonstration algorithm $\alg$ having expected regret $\E[R_T]$ must suffer error
$
    \E | \muhat_2 - \mu_2 | \gtrsim \sqrt{\frac{\Delta_2}{\E[R_T]}} \land 1.
$
\end{corollary}
The predictions of this corollary and our other results are now verified in numerical experiments. 

%% file: sec_arxiv_v2/experiment.tex
\section{Experiments} \label{sec:expts-main}

We now implement the reward estimators in Procedures~\ref{proc:SAE-estimator} and~\ref{proc:UCB-estimator} on a range of synthetic bandit instances and on a physics simulator derived from a real-world application in battery charging~\citep{attia2020closed,grover2018best}. Further experimental results and more detailed explanations of setups in this domain and including a new domain in gene expression data can be found in Appendix~\ref{app:exp}.

\paragraph{Simulated data.}
We simulate a $K=2$ armed bandit instance with Gaussian rewards distribution $X \sim N(\mu_i, \sigma^2) $ for each arm. The arm means $\mu_i$ are bounded in the range $[0, 1]$ with $\sigma^2=1.0$.
Our first set of experiments is based on simulations of Algorithms~\ref{alg:SAE-any-regret} and~\ref{alg:UCB-any-regret} (and the corresponding Procedures~\ref{proc:SAE-estimator} and~\ref{proc:UCB-estimator}). The results with two arms and UCB are illustrated in Figure~\ref{fig:all_sim_experiments}; SAE results are similar (see Appendix~\ref{app:exp}). 

Panel (a) of Figure~\ref{fig:all_sim_experiments} verifies that the regret is sublinear in $T$, with higher values of $\alpha$ incurring larger regret, as predicted by Proposition~\ref{prop:SAE-UCB-regret}. 
In panel (b), we plot the MSE of reward estimation ($\approx$ the square of the quantity in our theorems) from UCB against $T$, and observe that Procedures~\ref{proc:SAE-estimator} and~\ref{proc:UCB-estimator} attain smaller error when the algorithm has higher regret, i.e., for larger values of $\alpha$. We also see different slopes in these plots for different values of $\alpha$ (as predicted by Theorem~\ref{thm:SAE-UCB-reward-estimator}), and this motivates the question of whether a common quantity governs the scaling law across different choices of $\alpha$. Panel (c) confirms that this is indeed the case: the curves collapse onto each other when we plot the MSE against regret, and the slope of the best-fit lines---as predicted by Corollary~\ref{cor:two-armed}---are very close to~$-1$. 

In the $K$-armed case with $K = 4$, panel (d) demonstrate the variation of estimation rates across arms, where arms having large gaps (or lower values of $\mu$) are harder to estimate than those having small gaps. Once again, this corroborates the result of Theorem~\ref{thm:SAE-UCB-reward-estimator}, where we saw that the MSE must depend near-linearly on the gap of the arm since arms with larger gaps are pulled less often.

\paragraph{Application: Battery charging.}

In many scientific domains, we are interested in studying the performance landscape of a set of configurations.
For example, in battery charging, there are several electric current protocols for charging an electric battery~\citep{attia2020closed}.
Depending on the chosen protocol and a specified temperature regime, a battery undergoes a different range of chemical reactions that eventually determine its final lifetime.
Understanding relationships between charging protocols and induced battery lifetimes for different temperature regimes is crucial to designing the future generation of batteries that operate at a favorable point on this tradeoff. 

\begin{figure}[t]
\centering
    \begin{tabular}{cc} 
    \includegraphics[width=0.3\textwidth]{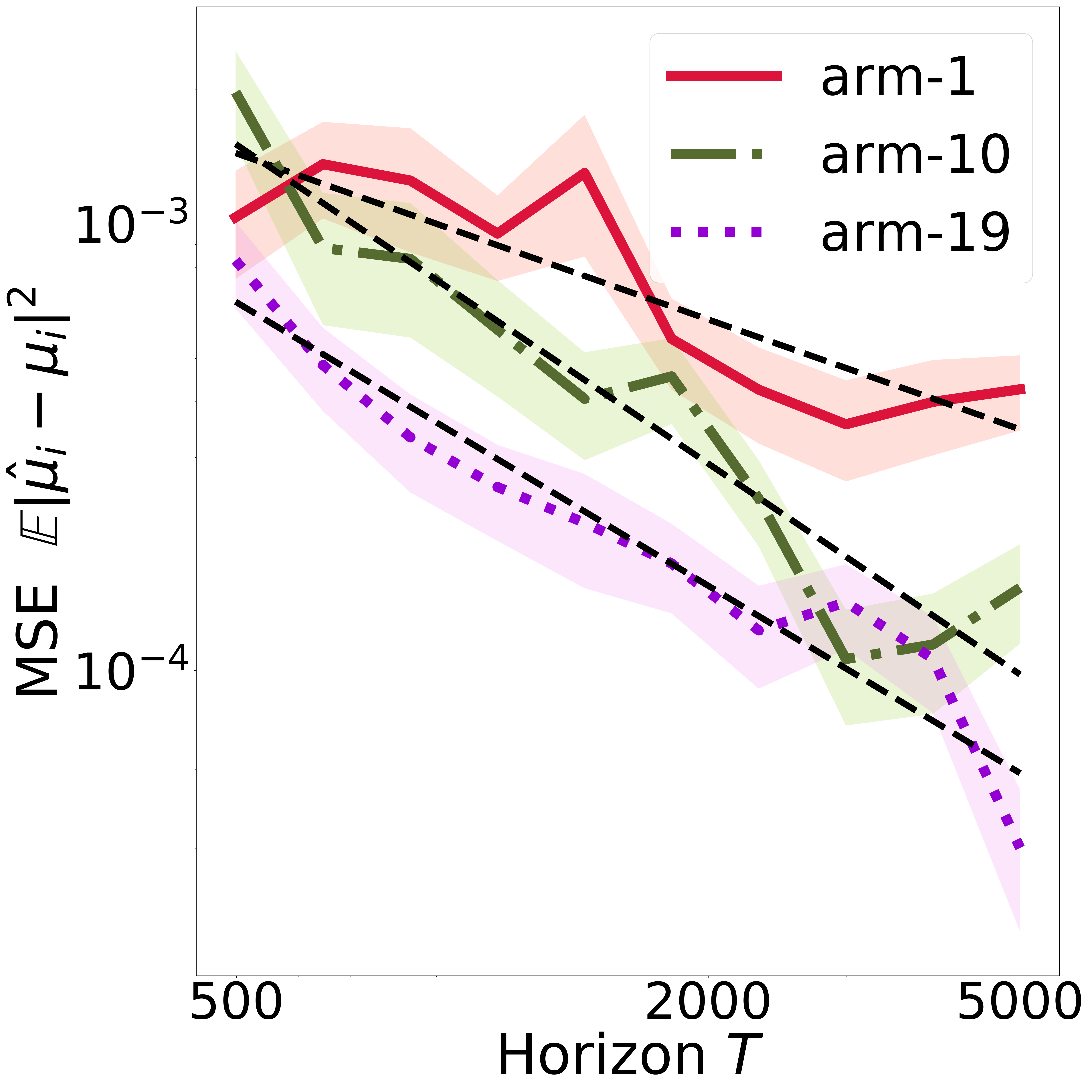} 
    \includegraphics[width=0.3\textwidth]{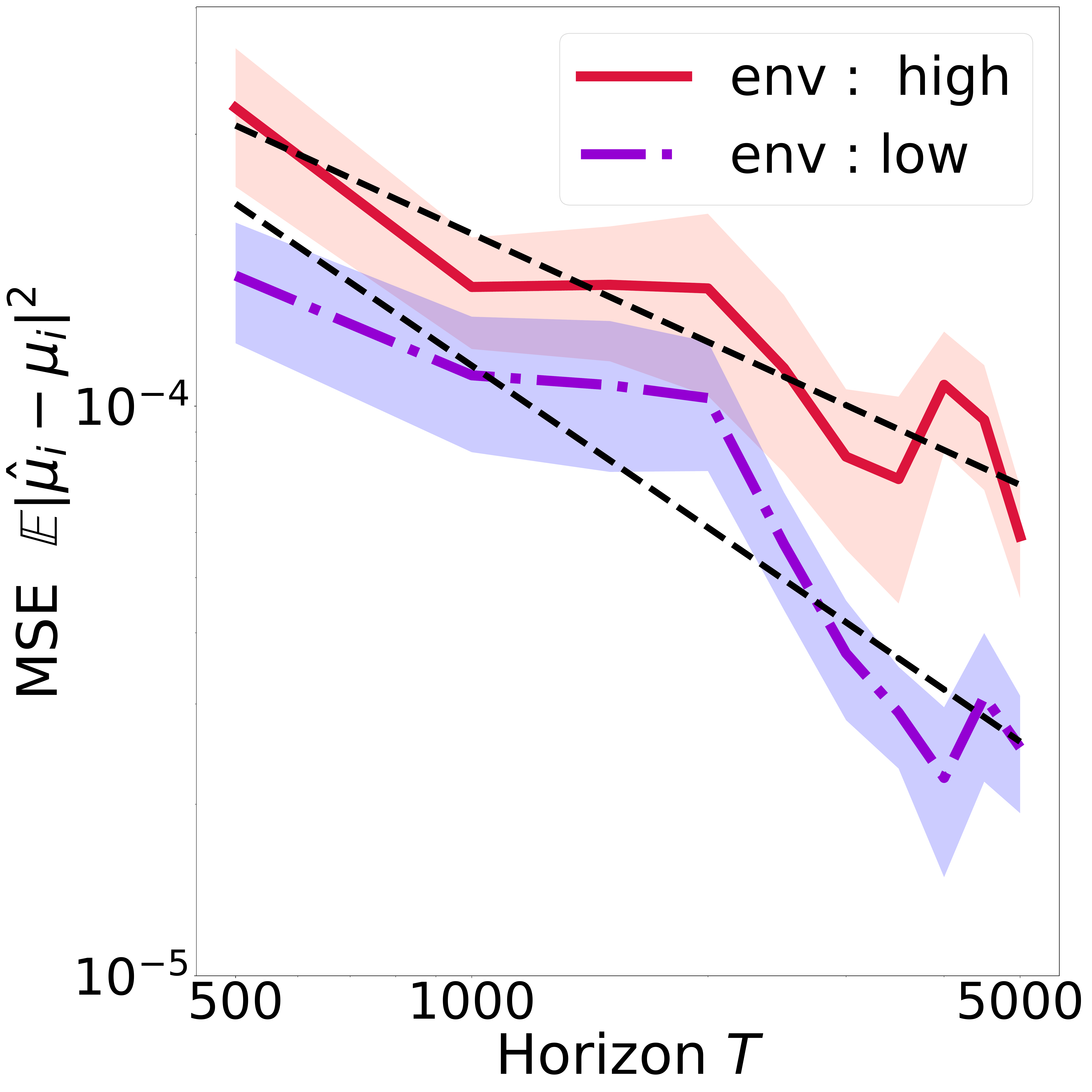} 
    \end{tabular}
    \caption{Results from 250 runs of estimating (normalized) battery lifetimes from a UCB experiment design procedure (a variance-adjusted version of Algorithm~\ref{alg:UCB-any-regret} with $\alpha = 0.25$).
    (a) Estimation error for a random subset of $3$ arms in the ``high" regime when algorithm is run on a $20$-armed instance. 
    (b) Error of estimating arm 12 in both ``low" and ``high" regimes with $4$ protocols. }
    \label{fig:battery}
\end{figure}

Our data at hand often consists of the results of experiments that were designed to search for lifetime maximizing configurations, and we would like to estimate the landscape of lifetimes from this data. We can cast this problem as one of reward estimation from an exploring demonstration. In particular, we map every temperature regime to a bandit instance where each charging protocol is an arm and the arm's reward is given by it's expected lifetime.
Given a demonstration of sequential experiments (i.e., arm pulls), 
our goal is to infer the lifetimes of \textit{all} charging protocols.

We consider $K$ distinct charging protocols from~\citet{attia2020closed} in two temperature regimes: \textit{low} and \textit{high}.
These two operating regimes exhibit different ranges of expected battery lifetime: low in [901, 962] and high in \mbox{[573, 1208]}. We obtain lifetime distributions for each protocol by fitting a Gaussian to a mix of real-world experimental data and physical simulations~\citep{attia2020closed}, and perform our experiments on this semi-synthetic data using the UCB algorithm as a representative experimental design approach. The reward means are normalized to lie in the range $[0, 1]$.


Figure~\ref{fig:battery}(a), plotted in the high temperature regime for $K = 20$ (see Appendix~\ref{app:exp} for other regimes), shows that the estimate for each charging protocol improves as the length of the trajectory $T$ increases. Lifetime estimation is thus possible even in cases where the number of protocols is moderately large. 
Next, we consider the problem of evaluating a particular charging protocol across temperature regimes with  $K = 4$. In Figure \ref{fig:battery}(b), we plot the estimation error for a representative arm having similar lifetime in both temperature regimes. Here, the behavior in panels (d) of Figure~\ref{fig:all_sim_experiments} is observed again: since the arm-gap in the low temperature regime is lower than in the high temperature regime, the error of Procedure~\ref{proc:UCB-estimator} is  correspondingly reduced.

%% file: sec_arxiv_v2/conclu.tex
\section{Discussion} \label{sec:conc}

We introduced and studied the inverse bandit problem of estimating rewards from observing a low-regret demonstrator. We provided information-theoretic lower bounds and simple, optimal reward estimation procedures. Our results quantify a tradeoff between exploration and reward estimation, and are corroborated by extensive  synthetic and semi-synthetic experiments. While this work takes a first step towards theoretically optimal reward estimation from an exploring demonstration, many open questions remain. It is interesting to study other demonstrator algorithms, e.g., randomized algorithms, in which the reward estimation comes with new challenges. Tackling these challenges is crucial to deploying this paradigm in scenarios where humans are known to randomize their behavior~\citep{daw2006cortical, schulz2015learning, speekenbrink2015uncertainty}.
Another interesting direction is to extend our insights to more expressive settings like contextual bandits, tabular RL, and continuous control.




%% file: sec_arxiv_v2/app_proof_LB.tex
\section{Proof of Theorem~\ref{thm:LB}} \label{app:sec:proof-LB}

The proof of Theorem~\ref{thm:LB} establishes a natural link to information-theoretic lower bounds on best-arm identification.
Denote the $K$-arm bandit instance by $\instance = \{\Bern(\mu_1), \ldots, \Bern(\mu_K)\}$, 
and suppose without loss of generality that the arms of $\instance$ are indexed with decreasing expected rewards, i.e. $\armmean^\ast = \armmean_1 > \armmean_2 \geq \cdots \geq \armmean_K$. 
(Note that the demonstrator's algorithm $\alg$ does not know this indexing.)
Recall that $\nalg_{i, T}$ denotes the number of times arm $i$ is pulled by the demonstrator's algorithm $\alg$. Further, for any $t$, let $\cF_t(\alg)$ be the sigma algebra of the sequence of actions and random reward samples generated by the algorithm $\alg$, i.e. $\cF_t(\cA) \defeq \sigma(\{\action_1, r_{\action_1}, \action_2, r_{\action_2}, \ldots, \action_t, r_{\action_t}\})$ where $r_{\action_t} \sim \Bern(\armmean_{\action_t})$ denotes a random reward sample, and $\mathbb{F}(\cA) = \{\cF_t(\cA)\}_{t\geq 1}$ is a filtration. 

Corresponding to some suboptimal arm $i \neq 1$, we 
construct another bandit instance $\instance' = \{\Bern(\mu'_1), \ldots, \Bern(\mu'_K))$ with $\mu'_j$ defined as follows. Let $\mu'_j = \mu_j$ for each $j \neq i$, and set 
\begin{align*}
\mu'_i = 
    \begin{cases}
    \mu_i + \epsilon \quad \text{ if } \mu_i \leq 1/2, \\
    \mu_i - \epsilon \quad \text{ otherwise.}
    \end{cases}
\end{align*}
for some scalar $\epsilon \in (0,\nicefrac{1}{4})$ that we will subsequently specify. Because $\armmean_i \in [0,1]$, we have $\armmean'_i \in [\nicefrac{1}{4}, \nicefrac{3}{4}]$ for all $i \in \aset$.

We now reduce the reward estimation problem to one of binary testing via the classic Le-Cam approach. 
Suppose one of instance $\instance$ or $\instance'$ is chosen uniformly at random, and 
we observe sequence $\infoset_T=\{\action_1, \action_2, \cdots, \action_T\}$ generated by algorithm $\alg$.
Let $\infoset_T^0$ denote this random sequence under the bandit instance $\instance$, and denote by $\infoset_T^1$ the random sequence observed under the bandit instance $\instance'$. We denote the distributions of $\infoset_T^0$ and $\infoset_T^1$ by $\nu_T^0$ and $\nu_T^1$, respectively.
We use $\E_0[\cdot]$ to denote expectations under the bandit instance $\instance$, and $\E_1[\cdot]$ to denote expectations 
under the bandit instance $\instance'$. Analogously, we use $\Prob_0(\cdot)$ to denote $\E_0[\Ind(\cdot)]$, and $\Prob_1(\cdot)$ to denote $\E_1[\Ind(\cdot)]$. 


Now suppose the reward estimation procedure has knowledge of $\mu_1 = \mu'_1 = \mu^*$, and must estimate the sequence of reward means $\{\mu_i\}_{i \in \aset}$. Since the error of estimation is lower bounded by the error of testing between the instances $\instance$ and $\instance'$, we have


\begin{align}\label{eq:lb-est-test}
\begin{split}
	\max\{\E_0[|\estmean_{\aidx} - \armmean_{\aidx}|], \E_1[|\estmean_{\aidx} - \armmean_{\aidx}|]\} 
	&\geq \frac{\epsilon}{2}\left(1-\|\nu_T^0 - \nu_T^1\|_{\tv}\right) \\
	&\geq \frac{\epsilon}{2}\left(1-\sup_{\cE\in \cF_T(\alg)}\left|\Prob_0(\cE)-\Prob_1(\cE)\right|\right),
\end{split}
\end{align}
where the last step follows from the definition of the total variation (TV) distance and the fact that the action sequence is in the filtration.

We now apply~\citet[Lemma 1]{kaufmann2016complexity} 
to obtain
\begin{align*}
	\sup_{\cE\in \cF_T(\alg)}|\Prob(\cE)-\Prob_1(\cE)|\leq \sqrt{\frac{\E[\numpulls_{\aidx,T}] \cdot \kl(\Bern(\mu_i),\Bern(\mu'_i))}{2}}
\end{align*}
where $\kl(\cdot,\cdot)$ denotes the Kullback-Leibler (KL) divergence between two distributions.
Then, we have
\begin{align*}
\kl(\Bern(\mu_i),\Bern(\mu'_i)) &= (\armmean'_\aidx + \epsilon)\log\left(\frac{\armmean'_\aidx + \epsilon}{\armmean'_\aidx}\right) + (1-\armmean'_\aidx - \epsilon)\log\left(\frac{1-\armmean'_\aidx - \epsilon}{1-\armmean'_\aidx}\right)\\
& \leq \left(\frac{1}{\armmean'_\aidx} + \frac{1}{1-\armmean'_\aidx}\right) \cdot \epsilon^2 \leq \frac{16}{3}\epsilon^2 
\end{align*}
where the first inequality follows from applying  $\log(1 + x) \leq x$, and the second inequality follows from the fact that $\armmean'_i \in [1/4,3/4]$.


Therefore, we have
\[
\sup_{\cE\in \mathcal{F}_T(\ACal)}|\Prob(\cE)-\Prob_1(\cE)| \leq \sqrt{ \frac{8}{3} \cdot \epsilon^2 \cdot \E_0[\numpulls^\alg_{\aidx,T}]  }.
\]
Combining the above with Eq~\eqref{eq:lb-est-test}, we have
\[
	\max\{\E_0[|\estmean_{\aidx} - \armmean_{\aidx}|], \E_1[|\estmean_{\aidx} - \armmean_{\aidx}|]\}  \geq \frac{\epsilon}{2} \left(1-\epsilon  \sqrt{\frac{8}{3}\E[\numpulls^\alg_{\aidx,T}]}  \right) \geq \frac{\epsilon}{2} \left(1-\epsilon  \sqrt{\frac{8}{3}(\E[\numpulls^\alg_{\aidx,T}] \lor 3/2)}  \right)
\]

Picking $\epsilon = \nicefrac{\sqrt{3}}{\{4\sqrt{2(\E[\numpulls^\alg_{\aidx,T}] \lor 3/2)}\}} < \nicefrac{1}{4}$ to maximize the right hand side of the above equation, we have
\[
	\max\{\E_0[|\estmean_{\aidx} - \armmean_{\aidx}|], \E_1[|\estmean_{\aidx} - \armmean_{\aidx}|]\} \geq  \frac{\sqrt{3}}{16\sqrt{2}} \cdot \left(\frac{1}{\sqrt{\E[\nalg_{\aidx, T}]}} \land \frac{1}{\sqrt{3/2}}\right) \geq \frac{1}{16} \cdot \left(\frac{1}{\sqrt{\E[\nalg_{\aidx, T}]}} \land 1 \right).
\]
This completes the proof.
\qed

\begin{remark}
Note from the proof that an identical lower bound applies even if the procedure has access to the random reward samples themselves, in addition to the demonstrator's action sequence.
\end{remark}

%% file: sec_arxiv_v2/prelim_lemmas.tex
\section{Preliminary lemmas for general bandit algorithms}\label{app:prelimlemmas}

We first present a convenient interpretation of the multi-armed bandit instance using the notion of ``reward tapes"~\citep[Chapter 1]{slivkins2019introduction}.
We consider a reward tape of length $T$ for each arm $\aidx \in \cA$, each cell of which contains a random reward sample from that arm. 
In particular, cell $j$ on the tape corresponding to arm $i$ contains the reward sample $X_{\aidx,j} \sim \armdist_\aidx$ (recall that $\armdist_\aidx$ denotes the reward distribution of arm $\aidx$). 
Each time arm $\aidx$ is pulled, we move one cell forward on its reward tape, and obtain a reward from the new cell. 
Note that $n_{\aidx, T}$ simply denotes the number of cells we have gone through on the reward tape of arm $\aidx$ by round $T$, and we trivially have $n_{\aidx,T} \leq T$.
Corresponding to the $n^{th}$ cell of the reward tape, we define confidence width $C(n) \defeq \sqrt{\frac{2(T^{\alpha} - 1)}{\alpha n}}$.

The reward tape construction applies to a generic adaptive sampling algorithm (including both the SAE and UCB algorithms), and simplifies the construction of certain critical events concerning the concentration of sample means of arms around their true means. 
We start by stating and proving a basic lemma, which essentially follows from Hoeffding's inequality.

\begin{lemma}\label{lem:UCB-TL1} 
Denote by $\samplemean_\aidx(\numpulls)$ the sample mean of arm $i$ obtained by moving $\numpulls$ cells along the reward tape.
Then, for any $n \geq 1$, we have
\[
|\samplemean_\aidx(\numpulls) - \armmean_{\aidx}| <  \sqrt{\frac{\log(2/\delta)}{2\numpulls}},
\]
with probability at least $1 - \delta$. 
\end{lemma}

\begin{proof}
By construction of the reward tape, the $j$-th pull of arm $\aidx$ generates the random reward $X_{\aidx,j} \sim \armdist_i$. This random variable is bounded in the range $[0,1]$ and has expectation $\E[X_{\aidx,j}] = \armmean_{\aidx}$. The sample mean is given by $\samplemean_\aidx(n) = \frac{1}{n} \sum_{j=1}^n X_{\aidx,j}$.
Applying Hoeffding's inequality yields
\[
\Prob \left(\left|\samplemean_{\aidx}(n) - \armmean_{\aidx} \right| \geq \epsilon \right) \leq 2e^{-2\epsilon^2\numpulls}.
\]
Setting $\epsilon =   \sqrt{\frac{\log(2/\delta)}{2\numpulls}}$, we obtain
\[
\Prob \left(|\samplemean_{\aidx}(n) - \armmean_{\aidx}| <  \sqrt{\frac{\log(2/\delta)}{2\numpulls}} \right) \geq 1 -\delta,
\]
which completes the proof.
\end{proof}

\noindent The following series of events will be used as building blocks in all of our proofs.
\begin{definition}\label{def:concentrationevents}
We define the following events that ensure concentration of the sample means of arms obtained along the reward tape around their true means.
\begin{enumerate}
    \item ``Anytime" concentration events:
    \begin{align}\label{eq:eventi0}
    \Ecali_0 \defeq \{|\samplemean_i(n) - \armmean_i| \leq C(n) \text{ for all } n = 1,\ldots, T\},
    \end{align}
    corresponding to each suboptimal arm $i \in [K]$.
    These events will be used to prove sub-linear regret guarantees for the SAE and UCB algorithms. 
    \item ``Small-sample" concentration events
    \begin{align}\label{eq:eventi1}
    \Ecali_1 \defeq \{\sqrt{n}|\samplemean_i(n) - \armmean_i| \leq \sqrt{\log (8 \nn)} \text{ for all } n = 1,\ldots, 8 \kappa_i \},
    \end{align}
    corresponding to each suboptimal arm $i \in [K] \setminus i^*$. 
    These events will be used to provide an eventual guarantee on estimation error of rewards of suboptimal arms. With a slight abuse of notation, we also define the event
    \begin{align}\label{eq:eventistar1}
    \Ecalij_1 \defeq \{\sqrt{n}|\samplemean_{j}(n) - \armmean_{j}| \leq \sqrt{\log (8 \nn\sqrt{K})} \text{ for all } n = 1,\ldots, 8 \kappa_i \},
    \end{align}
    where $j$ is the index of an arm that remains active during the first $8\kappa_i$ rounds.
    \item Tighter concentration events
    \begin{subequations}
    \begin{align}
        \Ecali_2 &\defeq \left\{|\samplemean_i(n) - \armmean_i| \leq \sqrt{\frac{3}{4}} C(n) \text{ for all } n = 1,\ldots, T \right\} \text{ and } \label{eq:eventi2} \\
         \Ecali_3 &\defeq \left\{ |\samplemean_i(n) - \armmean_i| \leq \frac{C(n)}{\sqrt{2}} \text{ for all } n = 1,\ldots, \frac{\nn}{32} \right\}, \label{eq:eventi3}
    \end{align}
    \end{subequations}
    corresponding to each arm $i \in [K]$.
    These events will be used to ensure that suboptimal arms are pulled sufficiently often to guarantee low error in estimation of their rewards.
    \item ``Large-sample" concentration events
    \begin{align}\label{eq:eventi4}
    \Ecaliistar_4 \defeq \begin{cases}
    |\samplemean_{i^*}(n) - \armmean_{i^*}| \leq \sqrt{\frac{2 \log \nn}{c \nn}} \text{ for all } n \in \{c\kappa_i, \ldots, \nn^2\} \text{ and } \\
    |\samplemean_{i^*}(n) - \armmean_{i^*}| \leq \sqrt{\frac{\log T}{\nn^2}} \text{ for all } n \in \{\nn^2 + 1, \ldots, T\},
    \end{cases}
    \end{align}
    defined for the optimal arm $i^*$ with reference to a suboptimal arm $i \in [K] \setminus i^*$.
    These events will be used in the case of the UCB algorithm to ensure high-probability lower bounds on the random variable $n_{i^*, \tau_i}$.
\end{enumerate}
\end{definition}
\noindent The following lemma shows that each of these events occurs with high probability.

\begin{lemma}\label{lem:UCB-TL2}
For each $i \in [K]$, the following results hold:
\begin{itemize}
    \item Event $\Ecali_0$ occurs with probability at least $1 - \nicefrac{2}{T^3}$.
    \item Event $\Ecali_1$ occurs with probability at least $1 - \nicefrac{1}{4\nn}$.
    \item Event $\Ecalij_1$ occurs with probability at least $1 - \nicefrac{1}{4\nn K}$.
    \item Event $\Ecali_2$ holds with probability at least $1 - \nicefrac{2}{T^2}$.
    \item Event $\Ecali_3$ holds with probability at least $1 - \nicefrac{1}{16\nn}$.
    \item Event $\Ecaliistar_4$ holds with probability at least $1 - \nicefrac{2}{T} - \nicefrac{c}{\nn}$.
\end{itemize}
\end{lemma}
\begin{proof}
The proof of Lemma~\ref{lem:UCB-TL2} proceeds by repeatedly applying the basic using the basic Lemma~\ref{lem:UCB-TL1} for different choices of $\delta$ and union bounding over varying ranges of $n$.
We prove each claim separately.
\noindent \textbf{Proof for event $\Ecali_0$:} 
For each $i \in [K]$ and a fixed $n \geq 1$, we have
\begin{align*}
    \Prob \left(|\samplemean_{\aidx}(n) - \armmean_{\aidx}| \geq  C(n) \right) &\leq \Prob \left(|\samplemean_{\aidx}(n) - \armmean_{\aidx}| \geq  \sqrt{\frac{2 \log T}{n}} \right) \\
    &\leq \frac{2}{T^4} ,
\end{align*}
where the first inequality follows because $C(n) \geq \sqrt{\nicefrac{2 \log T}{n}}$, and the second inequality follows by applying Lemma~\ref{lem:UCB-TL1} with the choice $\delta = \nicefrac{2}{T^4}$.
Taking a union bound over $n = 1,\ldots, T$ yields 
\begin{align*}
    \Prob\left(|\samplemean_{\aidx}(n) - \armmean_{\aidx}| \geq  C(n) \text{ for some } n = 1,\ldots, T \right) \leq \frac{2}{T^3} .
\end{align*}
This shows that the event $\Ecali_0$ holds with probability at least $1 - \nicefrac{2}{T^3}$, completing the proof.
\newline
\newline
\noindent \textbf{Proof for event $\Ecali_1$:}
For each $i \in [K]$ and each $n = 1,\ldots, 8\nn$, we apply Lemma~\ref{lem:UCB-TL1} with $\delta = \nicefrac{2}{64\nn^2}$. 
Then, we take a union bound over all $n = 1,\ldots, 8\nn$ to obtain
\begin{align*}
    \Prob \left(|\samplemean_i(n) - \armmean_i| > \sqrt{\frac{\log (8 \nn)}{n}} \text{ for some } n = 1,\ldots, 8\nn \right) &\leq 8 \nn \cdot \frac{2}{64 \nn^2} = \frac{1}{4\nn} .
\end{align*}
This completes the proof. 
\newline
\newline

\noindent \textbf{Proof for event $\Ecalij_1$:} Applying Lemma~\ref{lem:UCB-TL1} with the choice $\delta = \nicefrac{2}{64 \nn^2 K}$ yields
\begin{align*}
    \Prob \left(|\samplemean_{i^*}(n) - \armmean_{i^*}| \geq  \sqrt{\frac{\log 8\nn \sqrt{K}}{n}} \right) &\leq \frac{2}{64\nn^2K},
\end{align*}
for each fixed $n$, and taking a union bound over $n$ in the desired range completes the proof.
\newline
\newline

\noindent \textbf{Proof for event $\Ecali_2$:}
For each $i \in [K]$ and a fixed $n \geq 1$, we have
\begin{align*}
 \Prob \left(|\samplemean_{\aidx}(n) - \armmean_{\aidx}| \geq  \sqrt{\frac{3}{4}} C(n) \right) &\leq \Prob \left(|\samplemean_{\aidx}(n) - \armmean_{\aidx}| \geq  \sqrt{\frac{3 \log T}{2n}} \right) \\
    &\leq \frac{2}{T^3} ,
\end{align*}
where the first inequality follows because $C(n) \geq \sqrt{\nicefrac{2 \log T}{n}}$, and the second inequality follows by applying Lemma~\ref{lem:UCB-TL1} with the choice $\delta = \nicefrac{2}{T^3}$.
Taking a union bound over $n = 1,\ldots, T$ yields 
\begin{align*}
    \Prob\left(|\samplemean_{\aidx}(n) - \armmean_{\aidx}| \geq  C(n) \text{ for some } n = 1,\ldots, T \right) \leq \frac{2}{T^2} .
\end{align*}
This shows that the event $\Ecali_2$ holds with probability at least $1 - \nicefrac{2}{T^2}$, completing the proof.
\newline
\newline
\noindent \textbf{Proof for event $\Ecali_3$:} 
For each $i \in [K]$ and a fixed $n \in \{1,\ldots,\nicefrac{\nn}{32}\}$, we have
\begin{align*}
     \Prob \left(|\samplemean_{\aidx}(n) - \armmean_{\aidx}| \geq  \frac{C(n)}{\sqrt{2}} \right) &\leq \Prob \left(|\samplemean_{\aidx}(n) - \armmean_{\aidx}| \geq  \sqrt{\frac{\log \nn}{n}} \right) \\
    &\leq \frac{2}{\nn^2} ,
\end{align*}
where the first inequality follows because $\nicefrac{C(n)}{\sqrt{2}} \geq \sqrt{\nicefrac{\log \nn}{n}}$ for the specified range of $n$, and the second inequality follows by applying Lemma~\ref{lem:UCB-TL1} with the choice $\delta = \nicefrac{2}{\nn^2}$.
Taking a union bound over the specified range of $n$ completes the proof.
\newline
\newline
\noindent \textbf{Proof for event $\Ecaliistar_4$:}
First, consider the case where $c \nn \leq n \leq \nn^2$.
In this case, we have
\begin{align*}
    \Prob \left(|\samplemean_{i^*}(n) - \armmean_{i^*}| \geq  \sqrt{\frac{2 \log \nn}{c \nn}} \right) &\leq \Prob \left(|\samplemean_{i^*}(n) - \armmean_{i^*}| \geq  \sqrt{\frac{2 \log \nn}{n}} \right) \\
    &\leq \frac{2}{\nn^4} ,
\end{align*}
where the first inequality follows because  $n \geq c\nn$, and the second inequality follows by applying Lemma~\ref{lem:UCB-TL1} with the choice $\delta = \nicefrac{2}{\nn^4}$.
Second, consider the case where $\nn^2 < n \leq T$.
In this case, we have
\begin{align*}
    \Prob \left(|\samplemean_{i^*}(n) - \armmean_{i^*}| \geq  \sqrt{\frac{\log T}{\nn^2}} \right) &\leq \Prob \left(|\samplemean_{i^*}(n) - \armmean_{i^*}| \geq  \sqrt{\frac{\log T}{n}} \right) \\
    &\leq \frac{2}{T^2} ,
\end{align*}
where the first inequality follows because we are in the case $n > \nn^2$, and the second inequality follows by applying Lemma~\ref{lem:UCB-TL1} with the choice $\delta = \nicefrac{2}{T^2}$.
Taking a union bound over $n = c \nn, \ldots, T$ completes the proof.
\end{proof}

We will work on combinations of these events to prove Theorem~\ref{thm:SAE-UCB-reward-estimator} for the case of the SAE algorithm (Appendix~\ref{app:proof:thm:SAE-reward-estimator}) and the case of the UCB algorithm (Appendix~\ref{app:proof:thm:UCB-reward-estimator}).

%% file: sec_arxiv_v2/forward_proofs.tex
\section{Sub-linear regret guarantees for UCB and SAE} \label{app:prop-regret}

For completeness, we provide a proof for Proposition~\ref{prop:SAE-UCB-regret}, which bounds the regret of the UCB and SAE algorithms. 

\begin{proposition}\label{prop:SAE-UCB-regret}
Recall that $\gap_i = \bestmean - \mu_i$. For any $T >\numarms$, Algorithm~\ref{alg:SAE-any-regret} and Algorithm~\ref{alg:UCB-any-regret} both incur regret
$
\Regret_T \leq \sum_{\aidx \neq i^*} \frac{32(T^\alpha - 1)}{\alpha \gap_\aidx}
$
with probability at least $1 - \frac{4\numarms}{T^3}$. 
\end{proposition}

For clarity, we prove it separately for the SAE and UCB algorithms in Sections~\ref{sec:prop-regret-sae} and~\ref{sec:prop-regret-ucb}, respectively. These proofs follow from straightforward modifications to classical results~\citep{even2006action, lattimore2020bandit}, and readers familiar with regret analysis are advised to skip to Sections~\ref{app:proof:thm:SAE-reward-estimator} and~\ref{app:proof:thm:UCB-reward-estimator} for novel analyses of our reward estimation procedures.

\subsection{Proof of Proposition~\ref{prop:SAE-UCB-regret} with SAE algorithm}\label{sec:prop-regret-sae}

We begin with a useful lemma whose proof is provided at the end of the subsection (see Section~\ref{sec:proof:lem:saw-forward-numpulls-ub}). Recall the definition of the scalar $\kappa_i$ from Equation~\eqref{eq:n0}.

\begin{lemma}\label{lem:saw-forward-numpulls-ub}
For the SAE algorithm, we have
\[
\numpulls_{\aidx, T} \leq 8 \nn,
\]
simultaneously for all suboptimal arms $i$ with probability at least $1 - \nicefrac{2K}{T^3}$.
Furthermore, on the same event, the optimal arm $i^*$ is never eliminated.
\end{lemma}

With this lemma in hand, the proof of Proposition~\ref{prop:SAE-UCB-regret} for the SAE algorithm follows immediately.

\begin{proof}[Proof of Proposition~\ref{prop:SAE-UCB-regret},  SAE]
By the definition of pseudo-regret, we have
\begin{align*}
\Regret_T & = T\bestmean - \sum_{t=1}^T \reward_{t, \action_t} \\
&= \sum_{\aidx \in \aset}\gap_\aidx \numpulls_{\aidx, T}\\
&\leq \sum_{\aidx \in \aset} 8\nn \gap_\aidx \defeq \frac{32(T^\alpha-1)}{\alpha \Delta_\aidx},
\end{align*}
where the final inequality holds with probability at least $1 - \nicefrac{2K}{T^3}$ by applying Lemma~\ref{lem:saw-forward-numpulls-ub}.
This completes the proof.
\end{proof}

\subsubsection{Proof of Lemma~\ref{lem:saw-forward-numpulls-ub}}\label{sec:proof:lem:saw-forward-numpulls-ub}

Throughout this proof, we work on the event $\cap_{i=1}^K \Ecali_0$, which we showed in Lemma~\ref{lem:UCB-TL2} holds with probability at least $1 - \nicefrac{2K}{T^3}$.
Recall the definition of $\nn$ from Equation~\eqref{eq:n0}.
It is easy to verify that $C(8\nn) \leq \nicefrac{\Delta_i}{4}$.
Consequently, we have
\begin{align}\label{eq:armiSAE}
    \samplemean_i(8\nn) + C(8\nn) \leq \armmean_i + 2C(8\nn) \leq \armmean_i + \frac{\Delta_i}{2} .
\end{align}

Similarly, for the optimal arm $i^*$ we have
\begin{align}\label{eq:arm1SAE}
    \samplemean_{i^*}(8\nn) - C(8\nn) \geq \armmean_{i^*} - 2C(8\nn) \geq \armmean_{i^*} - \frac{\Delta_i}{2} .
\end{align}

On the other hand, we have $ \samplemean_i(n) - C(n) \leq \armmean_i \text{ and } \samplemean_{i^*}(n) + C(n) \geq \armmean_{i^*}$ for every $n = 1,\ldots,T$ and every $i \in [K]$. 
Because $\armmean_{i^*} > \armmean_i$, this yields
\begin{align}\label{eq:notelimrewardtape}
    2C(n) \geq \armmean_{i^*} - \samplemean_{i^*}(n) + \samplemean_i(n) - \armmean_i > \samplemean_i(n) - \samplemean_{i^*}(n),
\end{align}
for every $n = 1,\ldots,T$ and every $i \in [K] \setminus i^*$.
Equation~\eqref{eq:notelimrewardtape} guarantees that arm $i^*$ remains active throughout, as claimed.

To complete the proof, we show that each arm $i \neq i^*$ is eliminated by the time we arrive at epoch $8\nn$. Denote by $\tend(\tepoch)$ the (random) last round of epoch $\tepoch$.
If arm $i$ has already been eliminated in an epoch preceding epoch $8\nn$, we are done. Otherwise, since arm $i^*$ is always active, combining Equations~\eqref{eq:armiSAE} and~\eqref{eq:arm1SAE} gives us
\begin{align*}
    2C_{i,\tend(8\nn)} &\leq \armmean_i + \frac{\Delta_i}{2} - \armmean_{i^*} + \frac{\Delta_i}{2} + \samplemean_{i^*,\tend(8\nn)} - \samplemean_{i,\tend(8\nn)} \\
    &= \samplemean_{i^*,\tend(8\nn)} - \samplemean_{i,\tend(8\nn)} \\
    &\leq \samplemean_{\max}(\tend(8\nn)) - \samplemean_{i,\tend(8\nn)} .
\end{align*}

In summary, the condition for arm $i$ to be eliminated is met by epoch at most $8\nn$, directly implying that $n_{i,T} \leq 8\nn$.
This completes the proof.
\qed

\subsection{Proof of Proposition~\ref{prop:SAE-UCB-regret} for UCB}\label{sec:prop-regret-ucb}

The structure of this proof is identical to the SAE case. Recall the definition of the scalar $\nn$ from Equation~\eqref{eq:n0}.





\begin{lemma}\label{lem:UCB-L1}
For the UCB algorithm, we have
\[
\numpulls_{\aidx, T} \leq 8 \nn
\]
for a suboptimal arm $i$ with probability at least $1 - \nicefrac{4}{T^3}$.
\end{lemma}

As with the SAE case, the proof of Proposition~\ref{prop:SAE-UCB-regret} follows immediately from this lemma.
The steps are exactly identical and we omit them for brevity.
We conclude this section by proving Lemma~\ref{lem:UCB-L1}.




\subsubsection{Proof of Lemma~\ref{lem:UCB-L1}}\label{app:proof:ucb-L1}

Throughout this proof, we work on the event $ \Ecali_0 \cap \Ecalistar_0$, which we showed in Lemma~\ref{lem:UCB-TL2} holds with probability at least $1 - \nicefrac{4}{T^3}$.
Since $C(n) \leq \nicefrac{\Delta_i}{4}$ for all $n \geq 8\nn$, we have
\begin{align}\label{eq:armiUCB}
    \samplemean_i(n) + C(n) \leq \armmean_i + 2C(n) \leq \armmean_i + \frac{\Delta_i}{2} = \armmean_{i^*} - \frac{\Delta_i}{2}
\end{align}
for all $n \geq 8\nn$.

On the other hand, for the optimal arm $i^*$ we have
\begin{align}\label{eq:arm1UCB}
    \samplemean_{i^*}(n) + C(n) \geq \armmean_{i^*} 
\end{align}
for all $n = 1,\ldots,T$.
Denote by $\tend(8\nn)$ the (random) earliest round after which arm $i$ was pulled for the $8\nn$-th time.
If no such round exists, then we are done.
Otherwise, combining Equations~\eqref{eq:armiUCB} and~\eqref{eq:arm1UCB} gives us
\begin{align*}
    \samplemean_{i^*,t} + C_{i^*,t} > \samplemean_{i,t} + C_{i,t} 
\end{align*}
for all $t \geq \tend(8\nn)$.
Thus, we have shown that the upper-confidence bound of arm $i^*$ dominates the upper confidence bound of arm $i$ for all rounds $t \geq \tend(8\nn)$, implying that arm $i$ is never pulled thereafter.
This directly gives us $n_{i,T} = n_{i,\tend(8\nn)} \leq 8\nn$, which completes the proof for each suboptimal arm $i$.
\qed


%% file: sec_arxiv_v2/app_proof_SAE.tex
\section{Proof of Theorem~\ref{thm:SAE-UCB-reward-estimator} for SAE}\label{app:proof:thm:SAE-reward-estimator}

In this section, we provide the proof of Theorem~\ref{thm:SAE-UCB-reward-estimator} for the case of the SAE algorithm.
Recall that we need to bound the estimation error $|\estmean_\aidx - \armmean_\aidx|$, and recall the notation $\nn \defeq \nicefrac{4(T^\alpha - 1)}{\alpha \Delta_i^2}$ from Equation~\eqref{eq:n0}. 
This proof will follow as a series of deterministic statements working on the high-probability event
\begin{align}\label{eq:eventc}
    \EcalSAE \defeq \bigcap_{i=1}^K \left( \Ecali_0 \cap \Ecali_2  \right) \cap \Ecali_1 \cap \left(\cap_{j \in [K]}\Ecalij_1\right).
\end{align}
Lemma~\ref{lem:UCB-TL2} together with an application of the union bound ensures that the event $\EcalSAE$ holds with probability at least $1 - \nicefrac{2K}{T^3} - \nicefrac{1}{2\nn} - \nicefrac{2K}{T^2}$.

First, we claim that on the event $\EcalSAE$ and under our assumption that $T \geq 32 \sum_{i \neq i^*} \nn$, we have $\hat{\imath} = i^*$.
In order to see this, note that on the event $\cap_{i=1}^K \Ecali_0$ we may apply the statement of Lemma~\ref{lem:saw-forward-numpulls-ub} to conclude that $\numpulls_{\aidx, T} \leq 8\nn$ simultaneously for
all suboptimal arms $\aidx \neq i^*$.
Consequently, we have
\[
n_{i^*, T} = T - \sum_{i \neq i^*} n_{i, T} \geq 8 \sum_{i \neq i^*} \kappa_i > \max_{i \neq i^*} n_{i, T}.
\]

Next, we consider the estimation error $|\estmean_\aidx - \armmean_\aidx|$ for any suboptimal arm $\aidx$. Recall that $\tau_\aidx$ is the round at which suboptimal arm $\aidx \in \cA$ is eliminated (Equation~\eqref{eq:switchingroundSAE}).
Since we identified the optimal arm, i.e. $\hat{\imath} = i^*$, we have $\tau_i < T$.
Further, since arm $\aidx$ is eliminated at round $\tau_\aidx$, we have
\begin{align}\label{eq:armieliminated}
    2\ci_{\aidx, \tau_{\aidx}} \leq \samplemean_{\max }(\tau_\aidx) - \samplemean_{\aidx, \tau_\aidx}.
\end{align}

On the other hand, denote by $\tau'_i$ the \emph{penultimate} round on which arm $i$ is pulled. Since arm $i$ is still active during this round, we have  
\begin{align}\label{eq:armiactive}
      2\ci_{\aidx, \tau'_{\aidx}} > \samplemean_{\max }(\tau'_\aidx) - \samplemean_{\aidx, \tau'_\aidx}.
\end{align}
%

Note that $\numpulls_{i,\tau'_\aidx} = \numpulls_{i,\tau_\aidx}-1 = \numpulls_{i,T} - 1$.
Therefore, we have
\begin{align*}
    2\ci_{\aidx, \tau_\aidx}  = 2\sqrt{\frac{T^\alpha-1}{\alpha \numpulls_{\aidx, T}}} &\geq  2\sqrt{\frac{T^\alpha-1}{\alpha (\numpulls_{\aidx, T}-1)}} - 4\sqrt{\frac{T^\alpha-1}{\alpha}} \cdot (\numpulls_{\aidx, T} - 1)^{-\nicefrac{3}{2}}\\
    &= 2\ci_{i,\tau_i'} -  4\sqrt{\frac{T^\alpha-1}{\alpha}} \cdot (\numpulls_{\aidx, T} - 1)^{-\nicefrac{3}{2}} \\
    &> \samplemean_{\max}(\tau'_{\aidx}) - \samplemean_{\aidx, \tau'_{\aidx}} -  4\sqrt{\frac{T^\alpha-1}{\alpha}} \cdot (\numpulls_{\aidx, T} - 1)^{-\nicefrac{3}{2}}.
\end{align*}
Above, the first inequality follows from the fact that $\frac{1}{\sqrt{x}} - \frac{1}{\sqrt{x+1}} \leq 2 x^{-\nicefrac{3}{2}}$ for any $x \geq 1$, and the second inequality is a direct substitution of Equation~\eqref{eq:armiactive}.

Furthermore, we obtain
\[
\samplemean_{\max}(\tau'_{\aidx}) - \samplemean_{\aidx, \tau'_{\aidx}} \geq \samplemean_{\max}(\tau_{\aidx}) - \samplemean_{\aidx, \tau_{\aidx}} - \frac{2}{\numpulls_{\aidx, T}},
\]
as a consequence of the rewards being bounded between $[0, 1]$. Therefore, we have 
\begin{equation}\label{eq:est-sae-useful-1}
    (\samplemean_{\max}(\tau_{\aidx}) - \samplemean_{\aidx, \tau_{\aidx}})  - 2\ci_{\aidx, \tau_\aidx} \leq  \frac{2}{\numpulls_{\aidx, T}} + 4\sqrt{\frac{T^\alpha-1}{\alpha}} \cdot (\numpulls_{\aidx, T} - 1)^{-\nicefrac{3}{2}}.
\end{equation}
Proceeding now to the error term of interest, we have
\begin{align}\label{eq:est-sae-nice-1}
\begin{split}
    |\estmean_\aidx - \armmean_\aidx| &= |2C_{\aidx,\tau_\aidx} - (\bestmean - \armmean_\aidx)| \\
&\leq |2C_{\aidx,\tau_\aidx}  - (\samplemean_{\max}(\tau_\aidx) - \samplemean_{\aidx,\tau_\aidx})| + |\samplemean_{\max}(\tau_\aidx) - \samplemean_{\aidx,\tau_\aidx} - (\bestmean - \armmean_\aidx) |\\
&\leq \frac{2}{\numpulls_{\aidx, T}} + 4\sqrt{\frac{T^\alpha-1}{\alpha}} \cdot (\numpulls_{\aidx, T} - 1)^{-\nicefrac{3}{2}} +|\samplemean_{\max}(\tau_\aidx) - \bestmean| + |\samplemean_{\aidx,\tau_\aidx} - \armmean_\aidx| \\
&\leq \frac{2}{\numpulls_{\aidx, T}} + 2\sqrt{\nn} \cdot (\numpulls_{\aidx, T} - 1)^{-\nicefrac{3}{2}} +|\samplemean_{\max}(\tau_\aidx) - \bestmean| + |\samplemean_{\aidx,\tau_\aidx} - \armmean_\aidx|,
\end{split}
\end{align} 
where the first inequality follows from triangle inequality and rearranging terms, and the second inequality follows from Equation~\eqref{eq:est-sae-useful-1} and noting that $|2C_{\aidx,\tau_\aidx}  - (\samplemean_{\max}(\tau_\aidx) - \samplemean_{\aidx,\tau_\aidx})| = (\samplemean_{\max}(\tau_\aidx) - \samplemean_{\aidx,\tau_\aidx}) - 2C_{\aidx,\tau_\aidx}$ as a consequence of Equation~\eqref{eq:armieliminated}.

It remains to bound the sample-mean deviations $|\samplemean_{\max}(\tau_\aidx) - \bestmean|$ and $|\samplemean_{\aidx,\tau_\aidx} - \armmean_\aidx|$.
Towards that end, we require two technical lemmas, stated below and proved at the end of this section.
%
The first lemma bounds the deviation of the sample mean for each arm.

\begin{lemma} \label{lem:mu1-taui-sae}
Fix a suboptimal arm $i$. On the event $\EcalSAE$, there are universal positive constants $C, C_1, C_2$ such that for any arm $j \in [K]$ that remains active at round $\tau_i$, we have
\begin{equation*}
|\samplemean_{j,\tau_\aidx} - \armmean_j| < C \sqrt{\frac{\log (\nn\sqrt{K})}{\nn}}.
\end{equation*}
\end{lemma}

The second lemma provides a high-probability lower bound on $\numpulls_{\aidx, \tau_i}$, which is significantly more intricate than the typical lower bound on $\E[\numpulls_{\aidx,\tau_i}]$.
\begin{lemma} \label{lem:mui-taui-sae}
There is a universal constant $c > 0$ such that on the event $\EcalSAE$, we have $n_{i,T} \geq c \nn$.
\end{lemma}

Having stated these technical lemmas, we now use them to complete the proof of Theorem~\ref{thm:SAE-UCB-reward-estimator} for SAE. 
First, Lemma~\ref{lem:mu1-taui-sae} applied to the case $j = i$ directly gives us
\begin{align*}
    |\samplemean_{i,\tau_\aidx} - \armmean_i| < C \sqrt{\frac{\log (\nn\sqrt{K})}{\nn}}.
\end{align*}
Next, we again use Lemma~\ref{lem:mu1-taui-sae} to bound $|\samplemean_{\max}(\tau_\aidx) - \bestmean|$.
We denote by $i_{\max} \in [K]$ the arm index such that $\samplemean_{i_{\max}} = \samplemean_{\max}(\tau_\aidx)$.
On one hand, we have
\begin{align*}
    \samplemean_{\max}(\tau_\aidx) - \bestmean &\leq \samplemean_{i_{\max}} - \armmean_{i_{\max}} \\
    &< C \sqrt{\frac{\log (\nn\sqrt{K})}{\nn}}
\end{align*}
where the last step follows from Lemma~\ref{lem:mu1-taui-sae}.
On the other hand, we have
\begin{align*}
    \bestmean - \samplemean_{\max}(\tau_\aidx) &\leq \bestmean - \samplemean_{i^*} \\
    &< C \sqrt{\frac{\log (\nn\sqrt{K})}{\nn}}
\end{align*}
where, again, the last step follows from Lemma~\ref{lem:mu1-taui-sae}.

Proceeding from equation~\eqref{eq:est-sae-nice-1} and applying the inequalities established above, we have 
we have
\begin{align*}
    |\estmean_\aidx - \armmean_\aidx| &\leq \frac{C}{\nn} + C' \sqrt{\frac{\log (\nn\sqrt{K})}{\nn}}
\end{align*}
on the event $\EcalSAE$.
Taking an expectation to include the complement of $\EcalSAE$, we have
\begin{align*}
    \E |\estmean_\aidx - \armmean_\aidx| \leq C \sqrt{\frac{\log (\nn\sqrt{K})}{\nn}} + \frac{C_1}{\nn} + \frac{C_2 K}{T^2} \leq C'' \sqrt{\frac{\log (\nn\sqrt{K})}{\nn}}.
\end{align*}
We have used that $T^2/K \gtrsim \sqrt{\kappa_i}$ in stating the second inequality.

To complete the proof of the theorem, note that the proof of Proposition~\ref{prop:SAE-UCB-regret} (see Lemma~\ref{lem:saw-forward-numpulls-ub}) yields
$\E[n_{i,T}] \leq c' \nn$ for some positive constant $c' > 0$.
Since the map $x \mapsto \log x / x$ is decreasing, we obtain
\begin{align*}
    \E |\estmean_\aidx - \armmean_\aidx| &\leq C'' \sqrt{\frac{\log (\E[n_{i,T}\sqrt{K})]}{\E[n_{i,T}]}}
\end{align*}
for some adjusted constant $C''$.
This completes the proof of the first part of the theorem. The second part follows directly from Lemma~\ref{lem:mui-taui-sae}.
%
\qed

\subsection{Proof of Lemma~\ref{lem:mu1-taui-sae}}
As detailed at the beginning of Appendix~\ref{app:proof:thm:SAE-reward-estimator}, working on the event $\EcalSAE$ guarantees that $n_{i,T} \leq 8\nn$ for each suboptimal arm $i$ and that arm $i^*$ remains active throughout.
In addition, because event $\EcalSAE$ holds, we have that $\Ecali_1 \cap (\cup_{j \in \cS_i} \Ecalij_1) $ (defined in Equations~\eqref{eq:eventi1} and~\eqref{eq:eventistar1}, and $\cS_i$ denotes the set of arms that remain active in the first $8\nn$ rounds) holds.
This gives us
\begin{align}\label{eq:event1sup}
   \sup_{1 \leq \numpulls \leq 8\nn} \sqrt{\numpulls} |\samplemean_{k}(\numpulls) - \armmean_k| \leq \sqrt{\log (8\nn\sqrt{K})}
\end{align}
where $k$ can denote either the optimal arm $i^*$ or any suboptimal arm $j$ that remains active until round $\tau_i$.
Finally, note by the definition of $\tau_i$ that for all arms $j \in [K]$ that are active, we have $n_{j,\tau_i} = n_{i,\tau_i} = n_{i,T} \leq 8\nn$, which ensures that $n_{j,\tau_i}$ lies in the required range $\{1,\ldots,8\nn\}$ to apply Equation~\eqref{eq:event1sup}.
Moreover, by Lemma~\ref{lem:mui-taui-sae} (which also holds on event $\EcalSAE$) we have $n_{i,\tau_i} \geq c \nn$ for some constant $c > 0$.
Combining this with Equation~\eqref{eq:event1sup} yields
\[
|\samplemean_{j,\tau_\aidx} - \armmean_j| < \sqrt{\frac{\log (8\nn\sqrt{K})}{c\nn}}
\]
for all active arms $j \in [K]$.
This completes the proof.
\qed
%
%
%

\subsection{Proof of Lemma~\ref{lem:mui-taui-sae}}

Because the event $\EcalSAE$ holds, we have that $\cap_{i=1}^K \Ecali_2$ (defined in Equation~\eqref{eq:eventi2}) holds, which guarantees that \begin{align}\label{eq:rewardtapeevent}
    |\samplemean_i(n) - \armmean_i| \leq \sqrt{\frac{3}{4}} C(n) \text{ for all } n = 1,\ldots, T \text{ and all } i \in [K] .
\end{align}
For only this proof, we let $c \defeq (2 - \sqrt{3})^2$ for brevity.
It is easy to verify that $\Delta_i = c \cdot C\left(\frac{c^2 \nn}{2}\right)$.
By the triangle inequality, we have
\begin{align}\label{eq:triangleinequality}
\begin{split}
    |\samplemean_{i'}(n) - \samplemean_{i}(n)| &\leq |\samplemean_{i'}(n)-\armmean_{i'}| + \gap_i + |\samplemean_i(n)-\armmean_i| \\
    &\leq 2C(n)
\end{split}
\end{align}
for all $n \leq \frac{c^2 \nn }{2}$ and every $i \neq i'$. In other words, provided the number of pulls of arm $i$ does not exceed $\frac{c^2 \nn }{2}$, the condition for elimination is not met. Arm $i$ thus stays active for at least $\frac{c^2 \nn}{2}$ epochs, establishing the desired lemma.
\qed

%% file: sec_arxiv_v2/app_proof_UCB.tex
\section{Proof of Theorem~\ref{thm:SAE-UCB-reward-estimator} for UCB}\label{app:proof:thm:UCB-reward-estimator}

In this section, we provide the proof of Theorem~\ref{thm:SAE-UCB-reward-estimator} for the more complex case of the UCB algorithm. The structure of the proof resembles the SAE case, but the steps themselves are significantly more involved.
Recall that we need to bound the estimation error $|\estmean_\aidx - \armmean_\aidx|$, and recall the notation $\nn \defeq \nicefrac{4(T^\alpha - 1)}{\alpha \Delta_i^2}$.
As before, the proof will follow as a series of deterministic statements working on the high-probability event
\begin{align}\label{eq:eventc-ucb}
    \EcalUCB \defeq \left( \bigcap_{i=1}^K \Ecali_0 \right) \cap \Ecali_1 \cap \Ecaliistar_1 \cap \Ecali_2 \cap \Ecali_3 \cap \Ecalistar_3 \cap \Ecaliistar_4
\end{align}
From Lemma~\ref{lem:UCB-TL2} and the union bound, we have that the event $\EcalUCB$ holds with probability at least $1 - \nicefrac{C_1}{\nn} - \nicefrac{C_2}{T}$ for universal constants $C_1,C_2 > 0$.

First, we note that on the event $\EcalUCB$ and under our assumption of $T \geq 32 \sum_{i \neq i^*} \nn$, we have $\hat{\imath} = i^*$ via an argument that is identical to the SAE case (provided at the beginning of Appendix~\ref{app:proof:thm:SAE-reward-estimator}).
%
%
For any suboptimal arm $\aidx$, let $\tauhigh_\aidx$ denote the first round after $\tau_\aidx$ in which the best arm is pulled, noting that such a round always exists by the definition of $\tau_i$.

Since arm $\aidx$ is pulled at round $\tau_\aidx$ and arm $i^*$ is pulled at round $\tauhigh_\aidx$, the respective  upper confidence relations yield the bounds
\begin{align*}
    \samplemean_{\aidx,\tau_i} + C_{\aidx,\tau_\aidx} - \left(\samplemean_{i^*,\tau_i} + C_{i^*, \tau_i} \right) &\geq 0 \quad \text{ and }\\
    \samplemean_{\aidx,\tauhigh_i} + C_{\aidx,\tauhigh_i} - \left(\samplemean_{i^*,\tauhigh_i} + C_{i^*,\tauhigh_i} \right) &\leq 0.
\end{align*}
Combining the above two equations, rearranging terms, and applying the triangle inequality, we obtain
\begin{align}\label{eq:ucb-ci-gap-error}
\begin{split}
|C_{\aidx, \tau_i} - C_{i^*, \tau_i} - (\samplemean_{i^*, \tau_i} - \samplemean_{\aidx, \tau_i})|
    &\leq |C_{\aidx, \tau_i} - C_{\aidx, \tauhigh_i}| + |C_{i^*, \tau_i} - C_{i^*, \tauhigh_i}| + |\samplemean_{i^*, \tau_i} - \samplemean_{i^*, \tauhigh_i}| + |\samplemean_{\aidx, \tau_i} - \samplemean_{\aidx, \tauhigh_i}| \\
    &= |C_{\aidx, \tau_i} - C_{\aidx, \tauhigh_i}| + |C_{i^*, \tau_i} - C_{i^*, \tauhigh_i}| + |\samplemean_{\aidx, \tau_i} - \samplemean_{\aidx, \tauhigh_i}|,
\end{split}
\end{align}
where the last equality follows because arm $i^*$ is not pulled between round $\tau_i$ and $\tauhigh_i$. Thus, we have $|\samplemean_{i^*, \tau_i} - \samplemean_{i^*, \tauhigh_i}| = 0$.
Proceeding now to the error term of interest, we have
\begin{align*}
    |\estmean_\aidx - \armmean_\aidx| &= |C_{\aidx,\tau_\aidx} - C_{i^*,\tau_\aidx} - (\bestmean - \armmean_\aidx)| \\
&\leq |C_{\aidx,\tau_\aidx} - C_{i^*,\tau_\aidx} - (\samplemean_{i^*,\tau_\aidx} - \samplemean_{\aidx,\tau_\aidx})| + |\samplemean_{i^*,\tau_\aidx} - \bestmean| + |\samplemean_{\aidx,\tau_\aidx} - \armmean_\aidx|\\
&\leq  |C_{\aidx,\tau_\aidx} - C_{\aidx,\tauhigh_i}| + |C_{i^*,\tau_\aidx} - C_{i^*,\tauhigh_i}| + |\samplemean_{\aidx, \tau_\aidx} - \samplemean_{\aidx, \tauhigh_\aidx}| \\ & \quad +|\samplemean_{i^*,\tau_\aidx} - \bestmean| + |\samplemean_{\aidx,\tau_\aidx} - \armmean_\aidx|,
\end{align*} 
where the second inequality follows from equation~\eqref{eq:ucb-ci-gap-error}.
We bound each of the above terms separately.
First, because arm $i^*$ is not pulled between rounds $\tau_\aidx$ and $\tauhigh_\aidx$, we have $n_{i^*,\tau_\aidx} = n_{i^*,\tauhigh_\aidx}$, and consequently, its confidence interval stays the same, with
\begin{align}\label{eq:ucb-bound-change-c1}
     |C_{i^*,\tau_\aidx} - C_{i^*,\tauhigh_\aidx}|=0.
\end{align}
On the other hand, the confidence interval of arm $i$ changes, but not by much. We have
\begin{align}\label{eq:ucb-bound-change-ci}
\begin{split}
     |C_{\aidx,\tau_\aidx} - C_{\aidx,\tauhigh_\aidx}|&=\sqrt{\frac{2(T^\alpha-1)}{\alpha}}\left(\frac{1}{\sqrt{n_{\aidx,\tau_\aidx}}} - \frac{1}{\sqrt{n_{\aidx,\tauhigh_\aidx}}} \right) \\
     &= \sqrt{\frac{2(T^\alpha-1)}{\alpha}}\left(\frac{1}{\sqrt{n_{\aidx,\tau_\aidx}}} - \frac{1}{\sqrt{n_{\aidx,\tau_\aidx}+1}} \right)\\
    &\leq \sqrt{\frac{2(T^\alpha-1)}{\alpha}} \cdot n^{-3/2}_{\aidx, \tau_\aidx}.
\end{split}
\end{align}
where the inequality uses the fact that $\sqrt{j + 1} - \sqrt{j} \leq j^{-1/2}$ for any integer $j \geq 1$.
Next, using our reward tape notation, note that
\begin{align}\label{eq:bound-change-samplemean}
\begin{split}
    |\samplemean_{\aidx, \tauhigh_i} - \samplemean_{\aidx, \tau_\aidx}| &\defeq \left|\frac{\sum_{j=1}^{n_{\aidx, \tauhigh_i}} X_{\aidx,j}}{n_{\aidx, \tauhigh_i}} - \frac{\sum_{j=1}^{n_{\aidx, \tau_\aidx}} X_{\aidx,j}}{n_{\aidx, \tau_\aidx}}\right| \\
    &= \left|\frac{\samplemean_{\aidx, \tau_\aidx} \cdot \numpulls_{\aidx, \tau_\aidx} + X_{n_{i,\tauhigh_i}}}{\numpulls_{\aidx, \tau_\aidx}+1} - \samplemean_{\aidx, \tau_\aidx}\right|
    \\&=\left|\frac{X_{n_{i,\tauhigh_i}} -\samplemean_{\aidx, \tau_\aidx}}{n_{\aidx, \tau_\aidx}+1}\right|\leq \frac{1}{\numpulls_{\aidx, \tau_\aidx}},
    \end{split}
\end{align}
where the last inequality follows from the fact that both $X_{\aidx,j}$ and $\samplemean_{\aidx, \tau}$ are bounded between $0$ and~$1$.

It remains to bound the sample-mean deviations $|\samplemean_{i^*,\tau_\aidx} - \bestmean|$ and $|\samplemean_{\aidx,\tau_\aidx} - \armmean_\aidx|$. Towards that end, we require three technical lemmas, stated below and proved at the end of this section. 
The first lemma bounds the deviation of the sample mean for arm $i$ in terms of the number of times it has been pulled.

\begin{lemma}\label{app:lem:ucb-sup}
On the event $\EcalUCB$, we have
\begin{align}\label{eq:hoeffding-cell}
    \sup_{1 \leq t \leq T} \sqrt{\numpulls_{\aidx, t}}|\samplemean_{\aidx, t} - \armmean_{\aidx}| < \sqrt{\log 8\nn}
\end{align}
for any suboptimal arm $i$.
\end{lemma}

\noindent The second lemma provides a high-probability lower bound on $\numpulls_{\aidx, \tau_i}$, which is significantly more intricate than the typical lower bound on $\E[\numpulls_{\aidx,\tau_i}]$.
\begin{lemma} \label{lem:mui-taui}
On the event $\EcalUCB$, there is an absolute constant $c > 0$ such that we have $n_{i,\tau_i} \geq c \nn$ for any suboptimal arm $i$.
\end{lemma}
\noindent Our third and final lemma bounds the deviation of the sample mean for arm $i^*$.
\begin{lemma} \label{lem:mu1-taui}
On the event $\EcalUCB$, there is an absolute constant $C > 0$ such that
\begin{equation}\label{eq:cell-1}
|\samplemean_{i^*,\tau_\aidx} - \bestmean| < C \sqrt{\frac{\log \nn}{\nn}}.
\end{equation}
\end{lemma}

\noindent Having stated these technical lemmas, let us now complete the proof of Theorem~\ref{thm:SAE-UCB-reward-estimator} for UCB, operating throughout on the event $\EcalUCB$.
Applying Lemma~\ref{app:lem:ucb-sup} yields
\begin{align*}
    \sup_{1 \leq t \leq T} \sqrt{\numpulls_{\aidx, t}}|\samplemean_{\aidx, t} - \armmean_{\aidx}| < \sqrt{\log 8 \nn},
\end{align*}
from which we obtain
\begin{equation}\label{eq:cell-i}
|\samplemean_{\aidx,\tau_\aidx} - \armmean_\aidx| < \sqrt{\frac{\log \nn}{\numpulls_{\aidx, \tau_i}}}.
\end{equation}
In addition, Lemma~\ref{lem:mu1-taui} yields the bound
\begin{align} \label{eq:mu1-samplemean}
    |\samplemean_{i^*,\tau_\aidx} - \armmean_{i^*}| < C \sqrt{\frac{\log \nn}{\nn}}.
\end{align}
Putting together equations~\eqref{eq:ucb-bound-change-c1},~\eqref{eq:ucb-bound-change-ci},~\eqref{eq:bound-change-samplemean},~\eqref{eq:cell-i} and~\eqref{eq:mu1-samplemean}, the following sequence of bounds holds (where constants change from line-to-line, but are always absolute):
\begin{align*}
|\estmean_\aidx - \armmean_\aidx|  &\leq \frac{1}{\numpulls_{\aidx, \tau_\aidx}} +  \sqrt{\frac{2(T^\alpha - 1)}{\alpha}} \cdot n^{-3/2}_{\aidx, \tau_\aidx} + \sqrt{\frac{\log \nn}{\numpulls_{\aidx, \tau_i}}} + C \sqrt{\frac{\log \nn}{\nn}} \\
&\leq \frac{1}{\numpulls_{\aidx, \tau_\aidx}} +  C \sqrt{\nn} \cdot n^{-3/2}_{\aidx, \tau_\aidx} + \sqrt{\frac{\log \nn}{\numpulls_{\aidx, \tau_i}}} + C \sqrt{\frac{\log \nn}{\nn}}.
\end{align*}
Here, the second inequality holds by definition of $\nn$.
Finally, Lemma~\ref{lem:mui-taui} provides a lower bound on $\numpulls_{\aidx, \tau_i}$, which gives us
\begin{align*}
    |\estmean_\aidx - \armmean_\aidx|  &\leq C \sqrt{\frac{\log \nn}{\nn}}
\end{align*}
on the event $\EcalUCB$.
Recall that the event $\EcalUCB$ holds with probability greater than $1 - \nicefrac{C_1}{\nn} - \nicefrac{C_2}{T}$ for absolute constants $C_1,C_2$. Substituting the value of $\nn$ and reasoning exactly as in the SAE case about the complementary event, we obtain the desired upper bound on the expected error.
\qed


\subsection{Proof of Lemma~\ref{app:lem:ucb-sup}}
As detailed at the beginning of Appendix~\ref{app:proof:thm:UCB-reward-estimator}, working on the event $\EcalUCB$ guarantees that $n_{i,T} \leq 8\nn$ for each suboptimal arm $i$ (see Lemma~\ref{lem:UCB-L1}).
Moreover, since event $\EcalUCB$ holds, we have that event $\Ecali_1$ (defined in Equation~\eqref{eq:eventi1}) holds, which guarantees that
\begin{align*}
    \sqrt{\numpulls} \cdot |\samplemean_i(\numpulls) - \armmean_i| &\leq \sqrt{\log 8\nn}  \quad \text{ for all } n = 1,\ldots, 8\nn .
\end{align*}
Thus, it follows directly that
\begin{align*}
   \sup_{1 \leq t \leq T} \sqrt{\numpulls_{\aidx, t}}|\samplemean_{\aidx, t} - \armmean_{\aidx}| \leq \sup_{1 \leq \numpulls \leq 8\nn} \sqrt{\numpulls} |\samplemean_{\aidx}(\numpulls) - \armmean| \leq \sqrt{\log 8 \nn},
\end{align*}
which completes the proof.
\qed

\subsection{Proof of Lemma~\ref{lem:mui-taui}}

Since event $\EcalUCB$ holds, we have that events $\cap_{i=1}^K \Ecali_0$ (defined in Equation~\eqref{eq:eventi0}) and $\Ecali_2$ (defined in Equation~\eqref{eq:eventi2}) hold.
Our first observation is that we have $n_{i,\tau_i} \geq n_{i,3T/4}$ on the event $\EcalUCB$.
To see this, we define $\widetilde{\tau}_i$ to be the \textit{last} time that arm $i$ was pulled before round $3T/4$, and make a series of observations:
\begin{enumerate}
    \item \emph{$\widetilde{\tau}_i$ always exists and is well-defined}, since an examination of Algorithm~\ref{alg:UCB-any-regret} reveals that all arms will be pulled at least once in the first $K$ rounds, including any suboptimal arm $i$.
    \item On the event $\cap_{i=1}^K \Ecali_0$, we have $n_{i,T} \leq 8 \nn$ for all $i \neq i^*$. Since $T \geq 32 \sum_{i \neq i^*} \kappa_i$, this implies that the optimal arm $i^*$ will be pulled at least $3T/4$ times, and hence, at least once between rounds $\nicefrac{3T}{4}$ and $T$. Thus, \emph{the optimal arm $i^*$ is pulled at least once after $\widetilde{\tau}_i$}.
    \item By definition, $n_{i,\widetilde{\tau}_i} = n_{i,3T/4}$.
\end{enumerate}
The first two observations (italicized) are also satisfied for the round $\tau_i$, except that it is the maximal round for which these observations hold.
Thus, we have $n_{i,\tau_i} \geq n_{i,\widetilde{\tau}_i} = n_{i,3T/4}$, implying that it suffices to lower bound $n_{i,3T/4}$.

%
Consider the reward tapes for arms $i^*$ and $\aidx$, indexed by $\numpulls = 1,\ldots,\nicefrac{3T}{4}$.
Recall that the confidence interval in reward-tape notation is given by $C(\numpulls) \defeq \sqrt{\frac{2(T^{\alpha} - 1)}{\alpha \numpulls}}$.
First, note that for $n \geq 8 \nn$ we have $C(\numpulls) \leq \nicefrac{\Delta_\aidx}{4}$.
Thus, we have
\begin{align}\label{eq:arm1lcb}
    \samplemean_{i^*}(\numpulls) + C(\numpulls) \leq \armmean_{i^*} + 2C(n) \leq \armmean_{i^*} + \frac{\Delta_\aidx}{2} \text { for all } \numpulls \geq 8 \nn 
\end{align}
where the first inequality holds on event $\Ecaliistar_0$.
On the other hand, event $\Ecali_2$ gives us
\begin{align}\label{eq:armiucb}
    \samplemean_\aidx(\numpulls) \geq \armmean_\aidx - \sqrt{\frac{3}{4}} \cdot C(\numpulls) \text { for all } \numpulls \geq 1 .
\end{align}
Finally, Lemma~\ref{lem:UCB-L1} guarantees (see also its proof) that under event $\Ecali_0 \cap \Ecalistar_0$, we have $n_{i,\tau_i} \leq 8 \nn$.

We now use these statements to prove the lemma.
We consider indices $\numpulls,\numpulls'$ for reward tapes corresponding to arms $i^*$ and $\aidx$ respectively.
Provided that $\numpulls \geq 8\nn$ and $\numpulls' \leq \nicefrac{\nn}{32}$, we obtain
\begin{align*}
    \samplemean_\aidx(\numpulls') + C(\numpulls') &> \armmean_\aidx - \sqrt{\frac{3}{4}} \cdot C(\numpulls') + C(\numpulls') \\
    &= \armmean_\aidx + C(\numpulls') \left(1 - \sqrt{\frac{3}{4}}\right) \\
    &\geq \armmean_{i^*} + \frac{\Delta_\aidx}{2}\\
    &\geq \samplemean_{i^*}(\numpulls) + C(\numpulls) ,
\end{align*}
where the first and third inequality follow from Equations~\eqref{eq:armiucb} and~\eqref{eq:arm1lcb} respectively, and the second inequality follows from the constraint on $\numpulls'$.
Ultimately, we obtain
\begin{align}\label{eq:armipicked}
    \samplemean_\aidx(\numpulls') + C_\aidx(\numpulls') \geq \samplemean_{i^*}(\numpulls) + C_{i^*}(\numpulls) \text{ as long as } \numpulls \geq 8\nn \text{ and } \numpulls' \leq \frac{\nn}{32} .
\end{align}

In essence, Equation~\eqref{eq:armipicked} describes a sufficient condition for arm $i^*$ \textit{not} to be picked, i.e. the reward tape for arm $i^*$ has been run for greater than $8\nn$ cells and the reward tape for arm $\aidx$ has been run for at most $\nicefrac{\nn}{32}$ steps.
At a high level, our proof strategy is as follows: on the event $\cap_{i=1}^K \Ecali_0$, arm $i^*$ has to be picked sufficiently often at ``regular intervals".
For this to be possible, arm $\aidx$ needs to be pulled a minimal number of times to ensure that the condition in Equation~\eqref{eq:armipicked}
is \textit{not} satisfied.

We expand on this proof intuition below.
Consider the round $\nicefrac{T}{2}$ and note 
from Lemma~\ref{lem:UCB-L1} that the optimal arm $i^*$ needs to be pulled at least once between rounds $\nicefrac{T}{2}$ and $\nicefrac{3T}{4}$.
This requires 
\begin{align}\label{eq:arm1pulledonce}
    \samplemean_{i^*}(\numpulls_{i^*, t}) + C_{i^*,t} \geq \samplemean_\aidx(\numpulls_{\aidx,t}) + C_{i,t} \text{ for some } t \in \left[\frac{T}{2},\frac{3T}{4}\right] .
\end{align}
We now split the proof into two cases.

\noindent \underline{Case $n_{i, T/2} \geq \nicefrac{\kappa_i}{32}$:} In this case, we have $\numpulls_{\aidx,3T/4} \geq \numpulls_{\aidx,T/2} \geq \nicefrac{\nn}{32}$ and we are done.

\noindent \underline{Case $n_{i, T/2} < \nicefrac{\kappa_i}{32}$:}
We provide a proof-by-contradiction for this case.
Suppose that $n_{i,3T/4} < \nicefrac{\nn}{32}$.
By Lemma~\ref{lem:UCB-L1}, arm $i^*$ has to be pulled at least $8\nn$ times within the horizon $\nicefrac{T}{2}$, i.e. we have $\numpulls_{i^*,t} \geq 8\nn$ for all $t \in [\nicefrac{T}{2},\nicefrac{3T}{4}]$.
Thus, if we had $n_{i,3T/4} < \nicefrac{\nn}{32}$, the condition in Equation~\eqref{eq:armipicked} would be satisfied for all $t \in [\nicefrac{T}{2},\nicefrac{3T}{4}]$, implying that arm $i^*$ can \textit{never} be picked in this interval.
This contradicts our statement that arm $i^*$ has to be pulled at least once in this interval, and shows that we require $n_{i, 3T/4} \geq \nicefrac{\nn}{32}$ in this case.
This completes the proof.
\qed

\subsection{Proof of Lemma~\ref{lem:mu1-taui}}
Since event $\EcalUCB$ holds, we have that the event $\Ecaliistar_4$ (defined in Equation~\eqref{eq:eventi4}) holds.
We begin with the following claim, which we return to prove momentarily.

\begin{claim}\label{clm:n1taui-lb}
Under the event $\EcalUCB$, there exists a universal positive constant $c$ such that
\begin{align} \label{eq:n1taui-lb}
    \numpulls_{i^*, \tau_\aidx} \geq c \nn.
\end{align}
\end{claim}
%
\noindent Taking this claim as given, we split the proof of the lemma into two cases:

\noindent \underline{Case 1:} $c \nn \leq \numpulls_{i^*, \tau_\aidx} \leq \nn^2$. 
Here, the first case under the event $\Ecaliistar_4$ directly yields
\begin{align*}
    |\samplemean_{i^*}(n) - \armmean_{i^*}| \leq \sqrt{\frac{2 \log \nn}{c \nn}}.
\end{align*}

\noindent \underline{Case 2:} $\numpulls_{i^*,\tau_\aidx} > \nn^2$. 
Here, the second case under the event $\Ecaliistar_4$ directly yields
\begin{align*}
    |\samplemean_{i^*}(n) - \armmean_{i^*}| \leq \sqrt{\frac{\log T}{\nn^2}}.
\end{align*}
To complete the proof for this case, note that 
\[
\nn = \frac{4 (T^\alpha - 1)}{\alpha \Delta_\aidx^2} \geq \frac{4 \log T}{\Delta_\aidx^2} \geq 4 \log T,
\]
which gives us $|\samplemean_{i^*}(\numpulls_{i^*, \tau_\aidx}) - \armmean_{i^*}| \leq  \sqrt{\nicefrac{1}{4\nn}}$.
It only remains to establish Claim~\ref{clm:n1taui-lb}, which we do below.

\paragraph{Proof of Claim~\ref{clm:n1taui-lb}:}

Since event $\EcalUCB$ holds, we have that events $\Ecali_3$ and $\Ecalistar_3$ (defined in Equation~\eqref{eq:eventi3}) hold and the statement of Lemma~\ref{lem:mui-taui} holds. Then, we have:
\begin{subequations}
\begin{align}
    \samplemean_{i^*}(\numpulls) + C(\numpulls) &\geq \armmean_{i^*} + \left(1 - \frac{1}{\sqrt{2}}\right) C(\numpulls) \; \text{ for all } \numpulls \leq \frac{\nn}{32} \label{eq:event3a} \\
    \samplemean_\aidx(\numpulls) + C(\numpulls) &\leq \armmean_\aidx + \left(1 + \frac{1}{\sqrt{2}}\right) C(\numpulls) \; \text{ for all } \numpulls \leq \frac{\nn}{32},\text{ and } \label{eq:event3b} \\
    \numpulls_{\aidx, \tau_\aidx} &\geq \frac{\nn}{32} \label{eq:event3c}.
\end{align}
\end{subequations}


Let us define $\gamma_i$ as the $\nicefrac{\nn}{32}$-th time that arm $i$ is pulled, and $\underline{\gamma_i}$ as the $(\nicefrac{\nn}{32}-1)$-th time that arm $i$ is pulled. 

We will prove the lemma for the explicit choice $c = \nicefrac{1}{968}$.
We now have two cases.

\underline{Case 1:} $\numpulls_{i^*, \gamma_\aidx - 1} \geq c\nn$. In this case, the claim follows immediately, since on event $\Ecal''_3$, we have $\numpulls_{i^*, \tau_i } \geq \numpulls_{i^*, \gamma_\aidx}$.

\underline{Case 2:} $\numpulls_{i^*, \gamma_\aidx - 1} < c\nn$. 
As a consequence of the above inequalities, we have:
\begin{align*}
    \samplemean_{i^*}(\numpulls_{i^*, \gamma_\aidx -1}) + C(\numpulls_{i^*, \gamma_\aidx-1}) &\geq \armmean_{i^*} + \left(1 - \frac{1}{\sqrt{2}}\right) C(\numpulls_{i^*, \gamma_\aidx-1}) \\
    &\geq  \armmean_{i^*} + \left(1 - \frac{1}{\sqrt{2}}\right) C(c \nn) \\
    &= \Delta_\aidx + \armmean_\aidx +\left(1 - \frac{1}{\sqrt{2}}\right) C(c \nn) \\
    &\geq \Delta_\aidx + \samplemean_\aidx\left(\frac{\nn}{32} - 1\right) - \frac{1}{\sqrt{2}} C\left(\frac{\nn}{32} - 1\right) + \left(1 - \frac{1}{\sqrt{2}}\right) C(c \nn),
\end{align*}
where the first inequality follows from Equation~\eqref{eq:event3a}, the second inequality is because $C(\cdot)$ is decreasing in its argument, and the third inequality follows from Equation~\eqref{eq:event3b}.
By definition, $\numpulls_{i, \underline{\gamma}_\aidx} = \nicefrac{\nn}{32} - 1$, and by Lemma~\ref{lem:mui-taui} arm $\aidx$ must be pulled at least one more time. Furthermore, since $c = \nicefrac{1}{968}$ we have
\begin{align*}
- \frac{1}{\sqrt{2}} C\left(\frac{\nn}{32} - 1\right) + \left(1 - \frac{1}{\sqrt{2}}\right) C(c \nn) &= - \frac{1}{\sqrt{2}} C\left(\frac{\nn}{32} - 1\right) + 5.5 \left(1 - \frac{1}{\sqrt{2}}\right) C\left(\frac{\nn}{32}\right) \\
&\geq 3.2 \Delta_i \\
&>  C\left(\frac{\nn}{32} - 1\right) - \Delta_\aidx,
\end{align*}
where the final two steps follow from the relations 
\[
4 \Delta_i \leq C(\kappa_i/32) \leq C(\kappa_i/32 - 1) \leq 4.2 \Delta_i.
\]
Putting together the pieces yields
\[
\samplemean_{i^*}(\numpulls_{i^*, \gamma_\aidx-1}) + C(\numpulls_{i^*, \gamma_\aidx-1}) > \samplemean_{i}(\numpulls_{i, \gamma_\aidx-1}) + C(\numpulls_{i, \gamma_\aidx-1}) = \samplemean_\aidx(\numpulls_{\aidx, \underline{\gamma}_\aidx}) + C(\numpulls_{\aidx, \underline{\gamma}_\aidx}).
\]
But by definition, arm $\aidx$ is pulled at round $\gamma_\aidx$, and so we have the desired contradiction. Consequently, we must have $\numpulls_{i^*, \gamma_\aidx} > c\nn$. The claim then follows from the observation that $n_{i, \tau_i} \geq n_{i, \gamma_i}$ (by Eq~\eqref{eq:event3c}).
\qed

%% file: sec_arxiv_v2/app_exp.tex
\section{Additional Experimental Details and Results}\label{app:exp}

In this section, we provide additional details on the experiments with simulated and battery charging data in Section~\ref{sec:expts-main}, as well as further experimental results.



\subsection{Simulation results with SAE}

First, we present the simulation results with the SAE algorithm in Figure~\ref{fig:all_sim_experiments_appendix}.

\begin{figure}[ht]
    \centering
    \subfloat[\centering ]{{\includegraphics[width=0.3\textwidth]{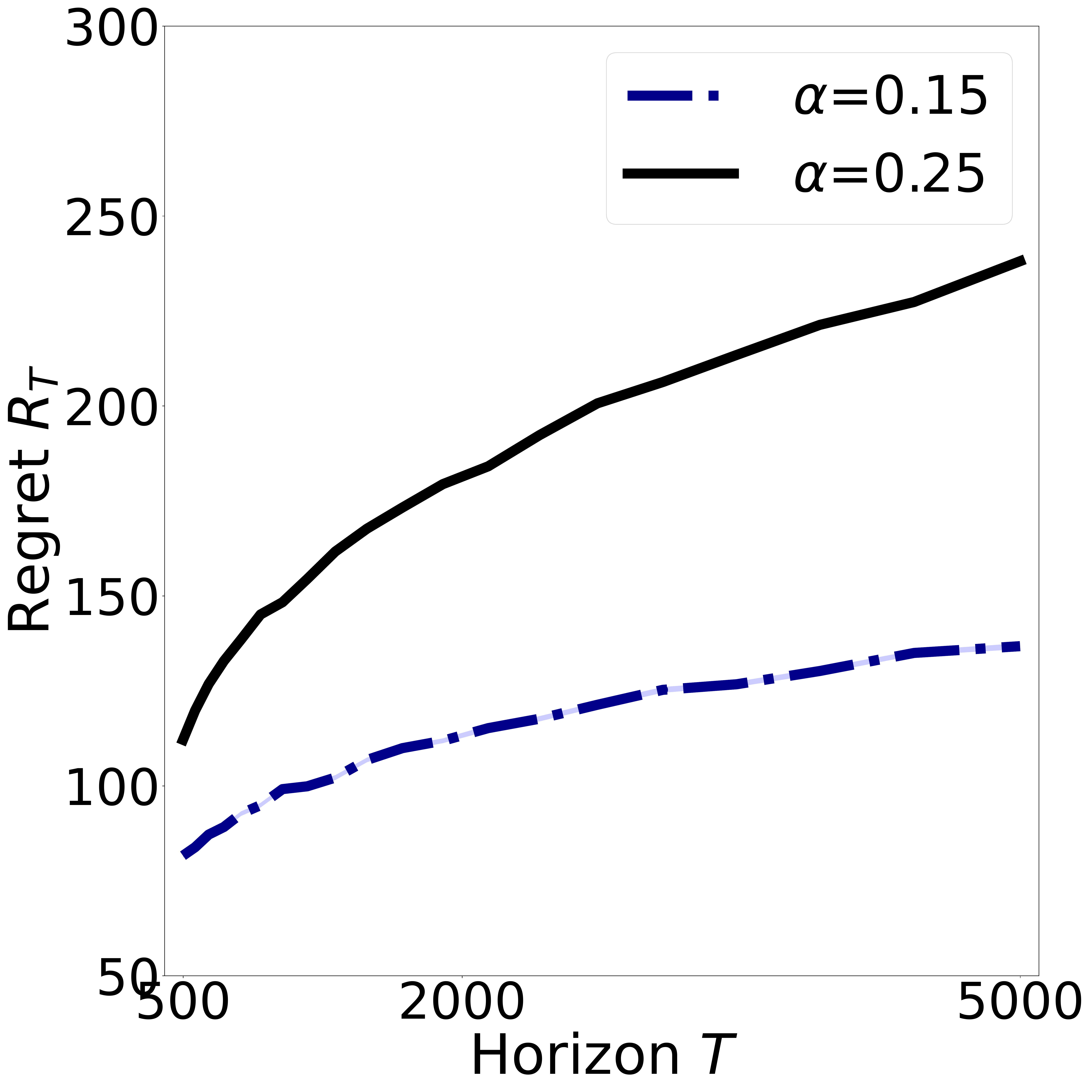} }}
    \hspace{0.01\textwidth}%
    \subfloat[\centering ]{{\includegraphics[width=0.3\textwidth]{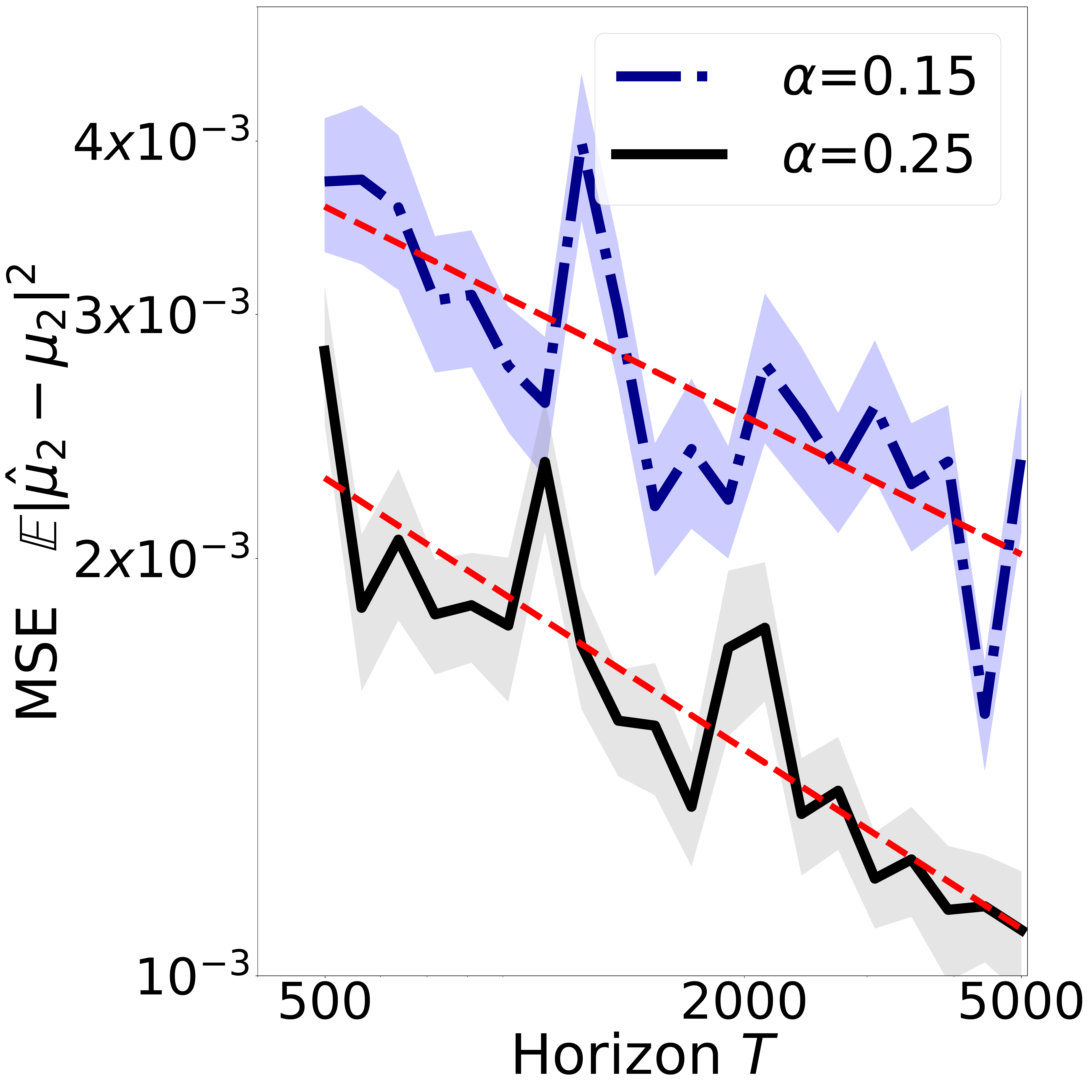} }}
    \subfloat[\centering ]{{\includegraphics[width=0.3\textwidth]{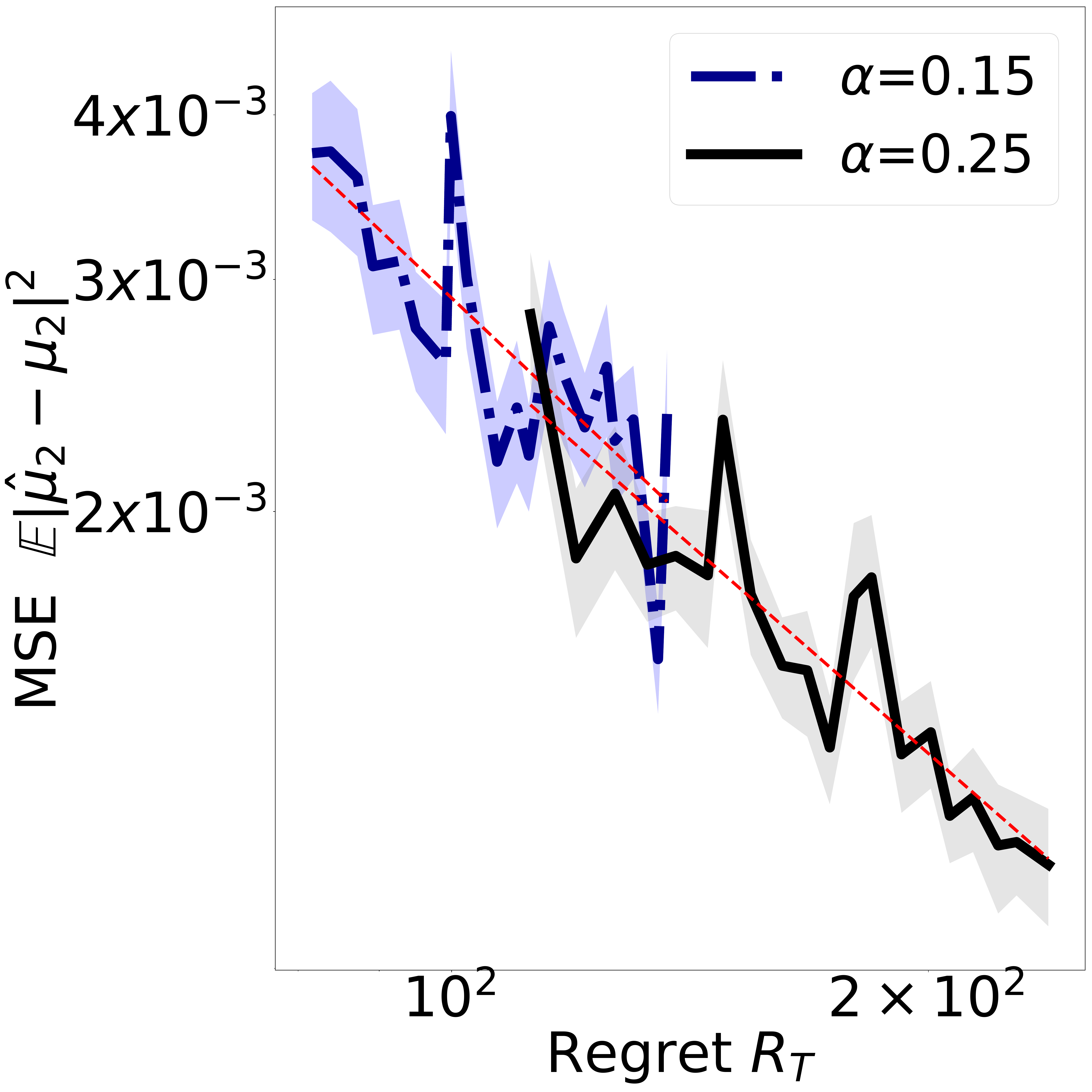} }} \\
    \subfloat[\centering ]{{\includegraphics[width=0.3\textwidth]{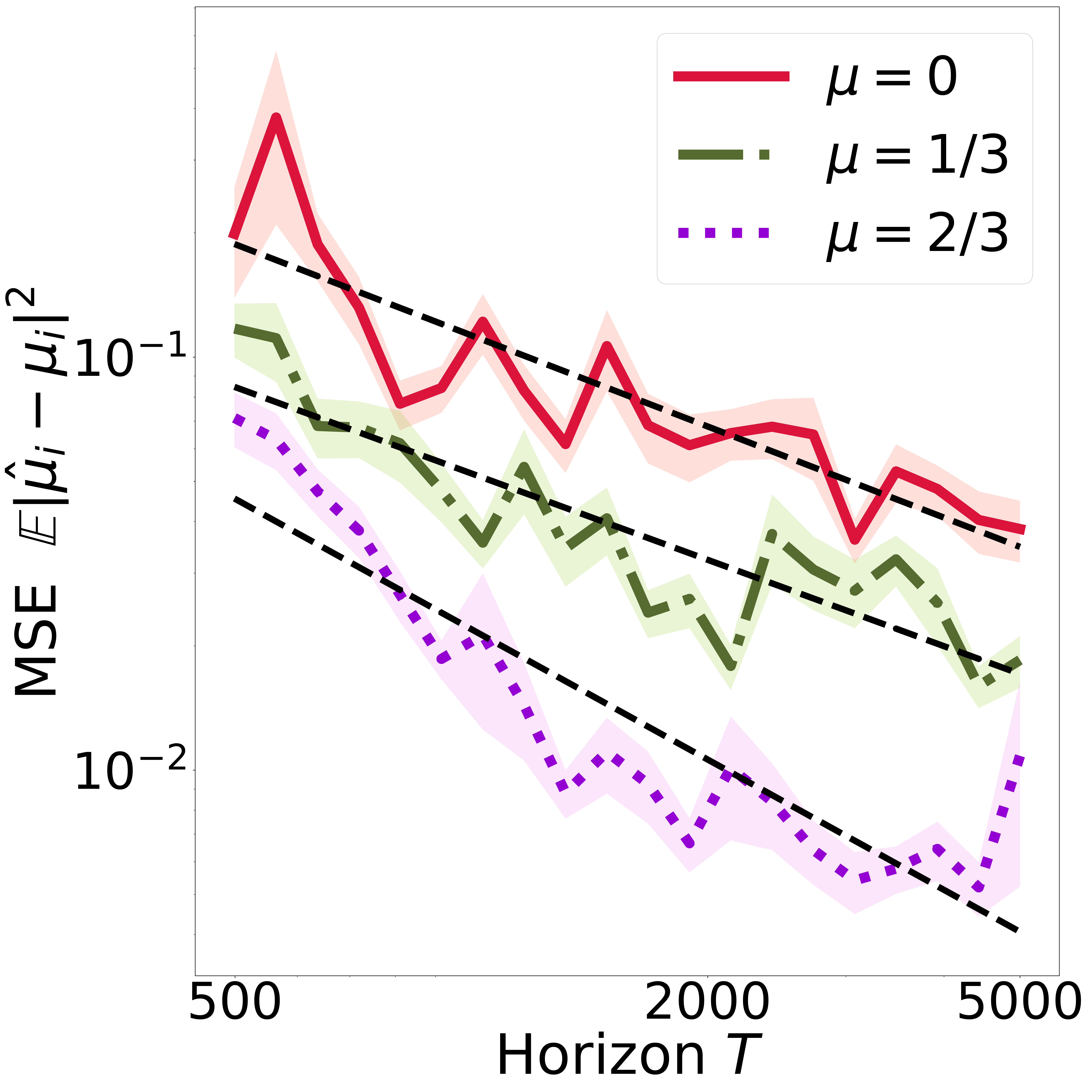} }}
    \caption{Results of 250 runs of simulation experiments for the SAE algorithm. Figures (a-c) are for a two-armed bandit instance with $\mu = (1, 1/2)$ and Gaussian rewards with unit variance. Here, individual curves represent two values of $\alpha \in \{0.15, 0.25\}$. Figure (d) is a $4$-armed instance with  $\mu = (1, 2/3, 1/3, 0)$ and Gaussian rewards with variance $1/4$. Here, individual curves represent the three suboptimal arms. Overall, these log-log plots corroborate our principal finding that better reward estimation is achievable from higher regret demonstrations; see the text for a detailed discussion.}
    \label{fig:all_sim_experiments_appendix}
\end{figure}


\subsection{Simulating multi-armed bandits}

We design a simple simulator for $K$-arm multi-bandit instances. In all our experiments, we assume Gaussian rewards for each arm, i.e $r_i \sim N(\mu_i, \sigma^2)$. Note that we fixed the variance $\sigma^2$ across all arms. The code for reproducing the results will be shared publicly on publication.

\paragraph{Algorithm implementation:} Algorithms \ref{alg:SAE-any-regret} and \ref{alg:UCB-any-regret} provide $O(T^\alpha)$ regret for any $\alpha \in (0,1)$. Using our simulator, we collect demonstrations for different $\alpha \in \{0.15, 0.25\}$. 

For the two-armed bandit instance, we let $\mu_1 = 1$ and $\mu_2 = 0.5$, i.e with a fixed $\Delta=0.5$.
For the $K$-arm instances, we choose the means $\mu_i$ to be linearly spaced between $[0, 1]$, (e.g for K=4, $\mu=(1, 2/3, 1/3, 0)$) with fixed variance $\sigma^2=0.25$ across all arms. 
We report results averaged over $\NUMRUNS$ independent demonstrations. 
We evaluate the estimators in Procedures \ref{proc:SAE-estimator} and \ref{proc:UCB-estimator} for different time-horizons $T$, evenly spaced in log space $\in [500, 5000]$.




\paragraph{Mean-squared error vs regret:} In Corollary~\ref{cor:two-armed}, we characterized the relationship between error in estimating rewards, and regret of the demonstrator's algorithm. Recall that for different values of $T$, the regret of both our upper-confidence-bound algorithms grows as $O(T^\alpha)$. To study relationship between mean-squared error (MSE) and regret, we fix $T$ and collect multiple demonstrations for instance with a fixed gap $\Delta$. The mean regret $R_T$ and corresponding standard error are computed by averaging across these demonstrations. Similarly, we estimate the gap using Procedure \ref{proc:UCB-estimator} and \ref{proc:SAE-estimator}, measuring MSE as average of the squared error for $T \in [500, 5000]$. 

\subsection{Dependence on $\Delta$} 

\begin{figure}[h]
    \centering
    \subfloat[\centering ]{{\includegraphics[width=0.35
    \textwidth]{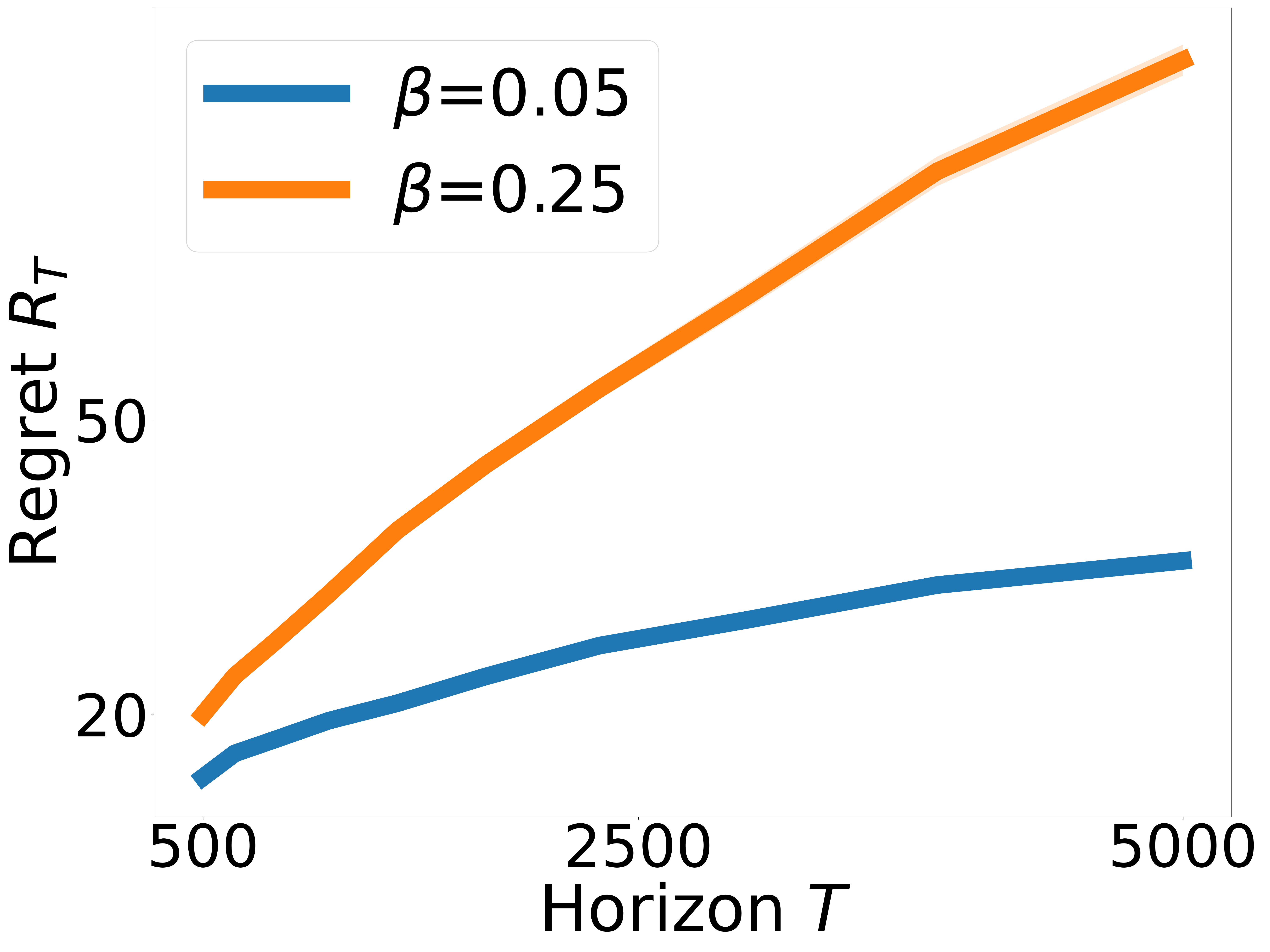} }}
    \hspace{0.01\textwidth}%
    \subfloat[\centering ]{{\includegraphics[width=0.3\textwidth]{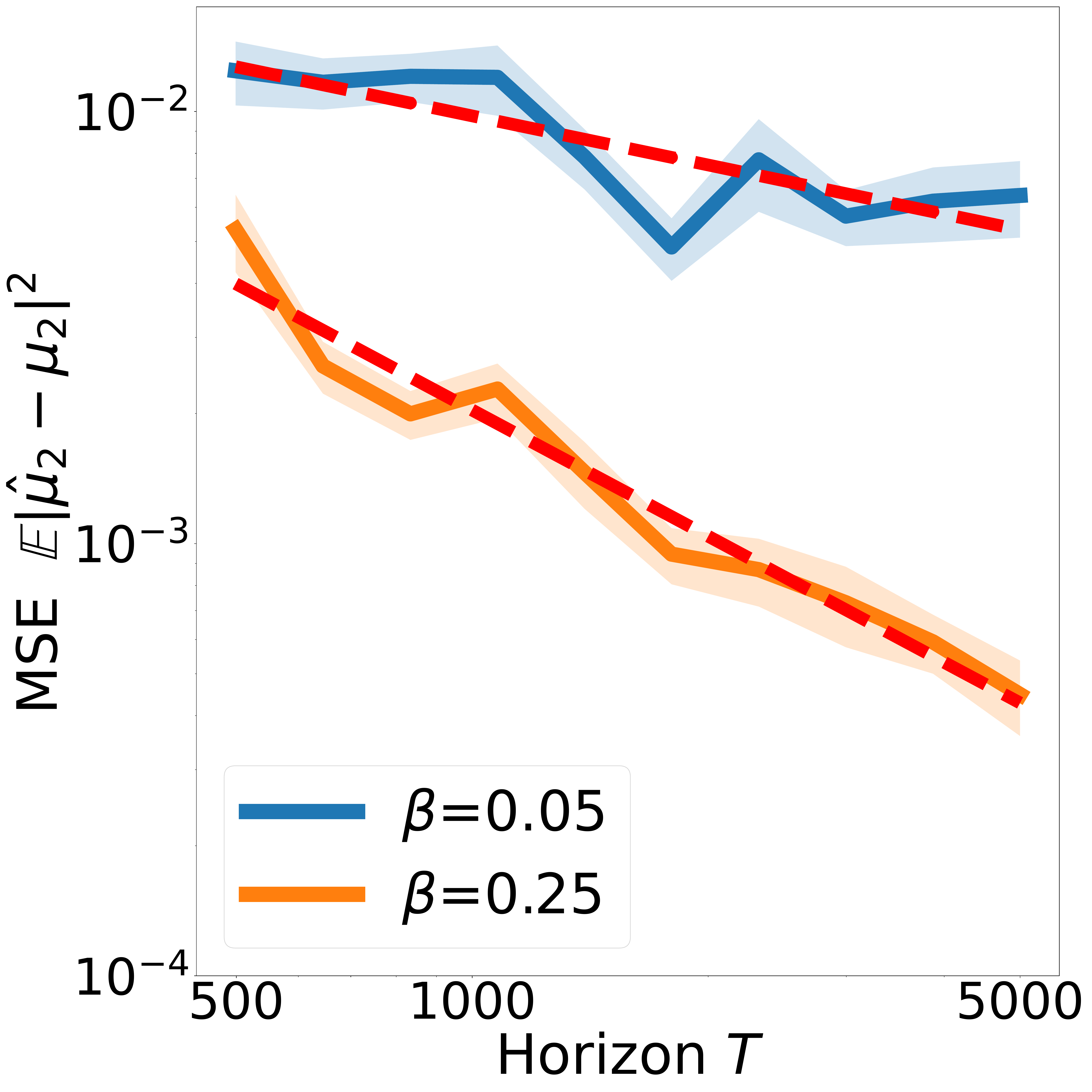} }}
    \caption{For the two-arm case, we construct MAB instances of varying difficulty by choosing the value of the suboptimality gap $\Delta_T$ for horizon $T$ as $\Delta_T = 1/T^\beta$. In figure (a), we verify that the regret increases for higher $\beta$ (i.e for fixed T, suboptimality gap reduces with higher $\beta$). Figure (b) empirically supports our predictions; for fixed $\beta$, estimation error decreases with $T$.}
    \label{fig:vary_delta}
\end{figure}

In this section, we explore the role of the suboptimality gap $\Delta$ in reward estimation for the case $K = 2$. Corollary~\ref{cor:two-armed} predicts that decreasing $\Delta$ will make reward estimation easier because it increases the regret.
To investigate whether this happens empirically, we make $\Delta_T$ decay smoothly with increasing time-horizon $T$, following the power law $\Delta_T = 1/T^\beta$, for $\beta \in (0, 0.5)$.
We expect the following behavior:
\begin{enumerate}
    \item for fixed $\beta$, the reward estimation error decays with increasing horizon $T$.
    \item for fixed horizon $T$, the rate of decay of reward estimation error increases with $\beta$.
\end{enumerate}

We consider two cases: $\beta \in \{0.05, 0.5\}$.
Figure~\ref{fig:vary_delta} shows that the regret indeed increases with a decrease in suboptimality gap (higher $T$). 
We observe that the reward estimation error decreases with $T$ for both values of $\beta$.
Moreover, the estimation error decays faster for larger values of $\beta$.
 
\subsection{Comparisons with the naive estimator}\label{app:naive-estimator-results}



We provide further comparisons between the proposed estimator and the naive estimator as described in Section~\ref{sec:upper_bound}. We evaluated the naive estimator with different values of $C_0 \in \{ 0.2, 0.75, 1.0, 1.5\}$  in~Fig. \ref{fig:simple_baseline}. The experiment setting is similar to Fig \ref{fig:all_sim_experiments} with two-arm stochastic bandits and a UCB demonstrator, where the mean rewards of the two arms are $\mu_1=1.0, \mu_2=0.5$ with standard deviation $\sigma=1.0$. For the baseline plots in Fig \ref{fig:simple_baseline} (b)-(e), the horizon $T$ is linearly spaced in $[500, 1000]$. All the results are averaged over 50 runs and plotted with the standard errors.

\begin{figure}[ht]
    \centering
    \subfloat[\centering ]{{\includegraphics[width=0.3\textwidth]{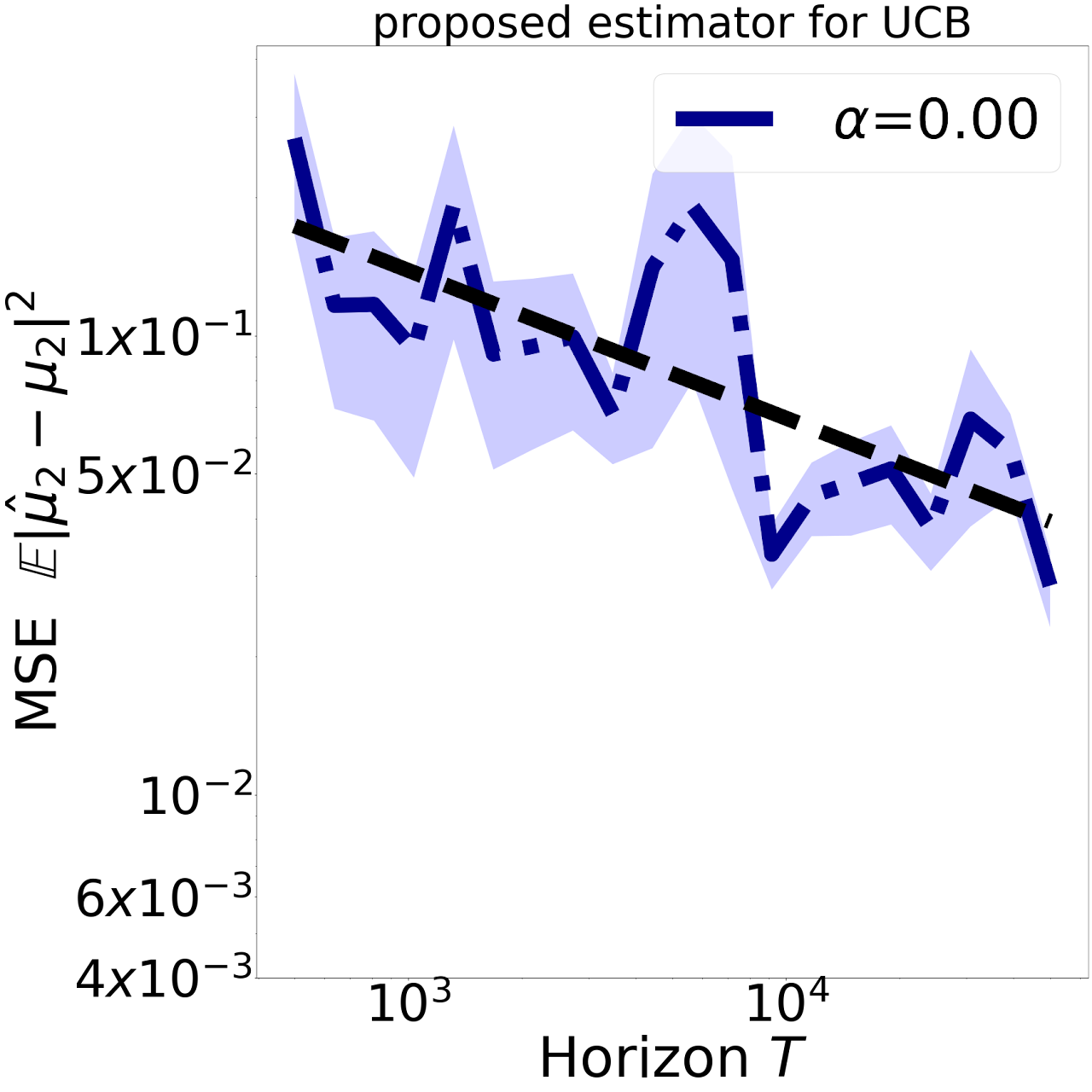} }}
    \hspace{0.01\textwidth}%
    \subfloat[\centering ]{{\includegraphics[width=0.3\textwidth]{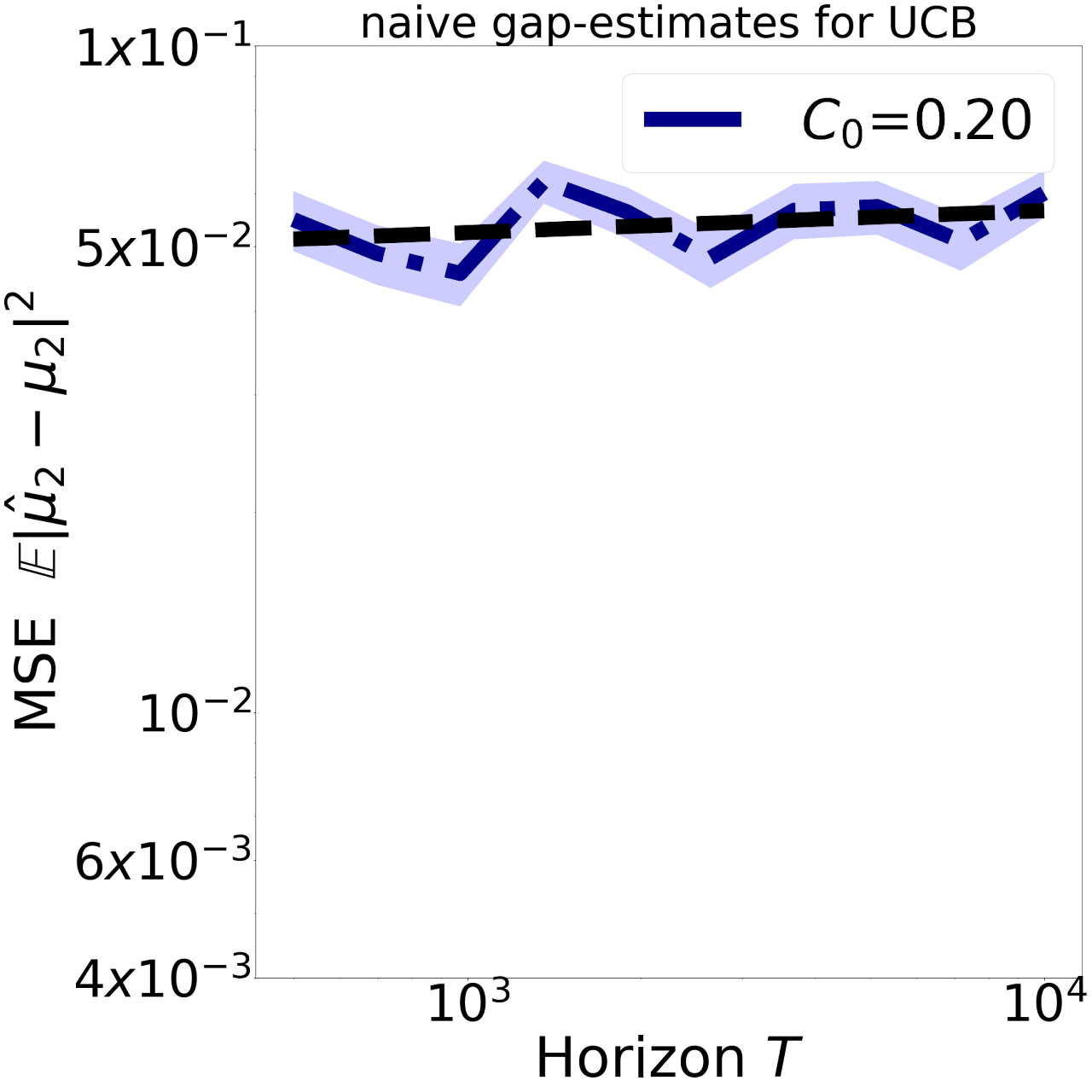} }}
    \newline
    \subfloat[\centering ]{{\includegraphics[width=0.3\textwidth]{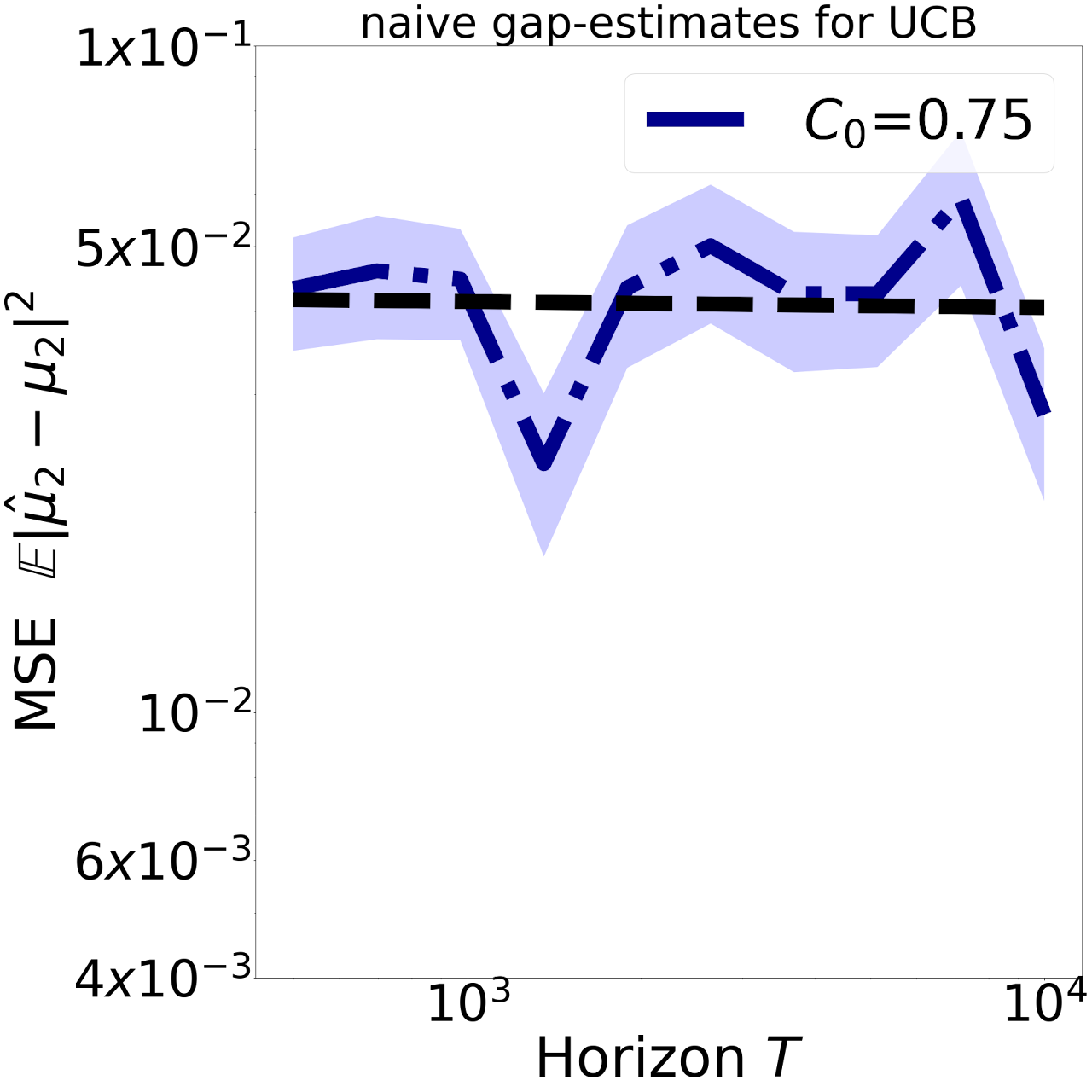} }}
    \subfloat[\centering ]{{\includegraphics[width=0.3\textwidth]{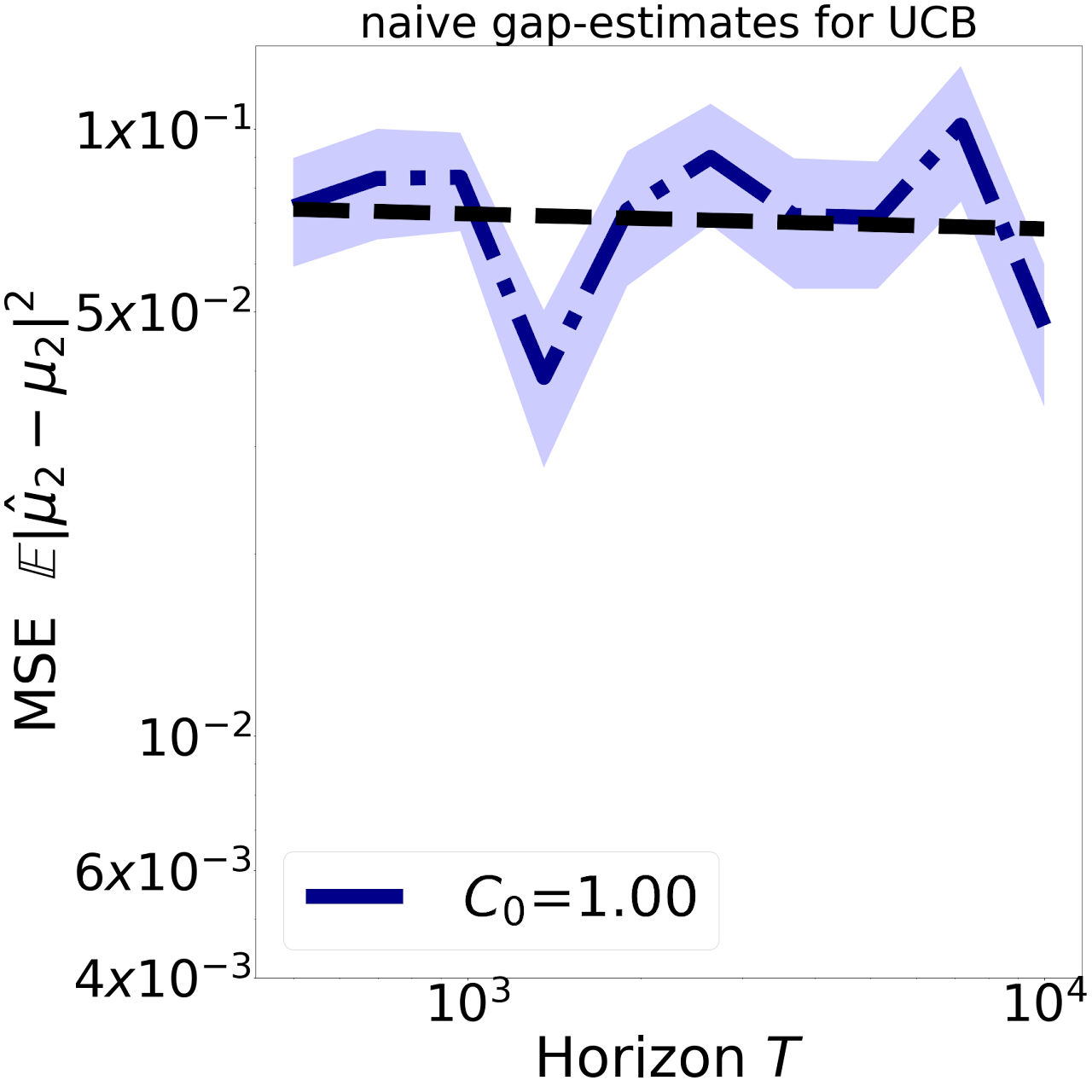} }}
    \subfloat[\centering ]{{\includegraphics[width=0.3\textwidth]{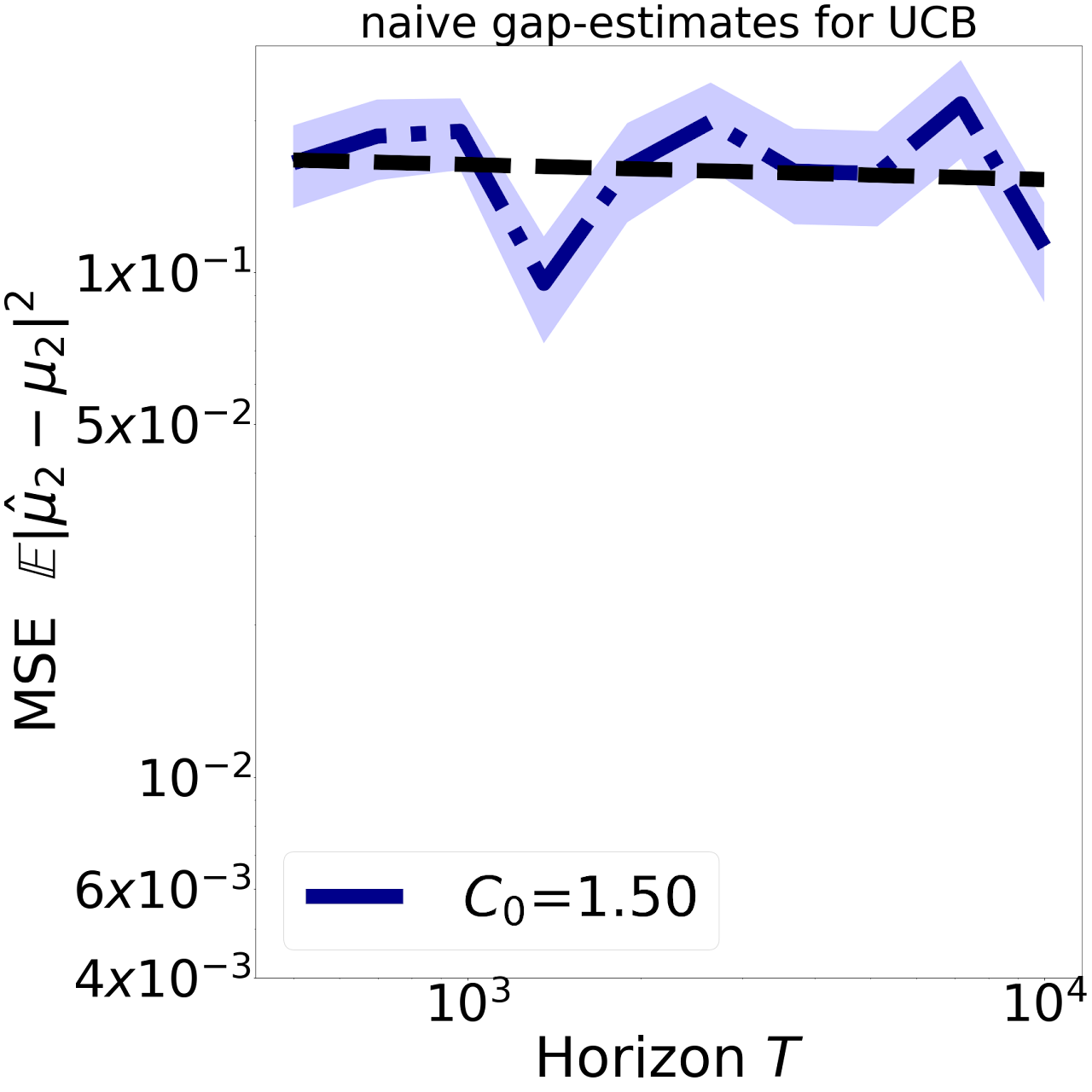} }}
    \caption{Comparing our proposed estimator with a simple estimator $\hat{\Delta} = C_0 \sqrt{ \frac{\log T}{n_a}}$ for $C_0 \in \{0.2, 0.75, 1.0, 1.5 \}$.}
    \label{fig:simple_baseline}
\end{figure}

\subsection{Battery charging}
 
\paragraph{Dataset:} The original dataset \citep{attia2020closed} provides battery life-cycles for 224 protocols, in different temperature regimes. The problem of identifying the optimal protocol (with highest mean lifetime) is cast as a MAB problem with 224 arms, where the three regimes are different instances. The distribution of reward means $\mu_i$ varies significantly across the regimes (see Figure~\ref{fig:batter_exp_appendix}).
In our experiments, we compare the ``low" and ``high" temperature regimes.
We subsample 20 arms which are representative of the distribution. In particular, we generate a histogram of rewards for the ``high" setting with $n=20$ bins, and pick an arm randomly from each bin. We fix this subset of $20$ arms for all our experiments in this section. 
Unless mentioned, we evaluate the estimators with number of independent runs ($N$) as $N=\NUMRUNS$. 

\paragraph{Normalization:} The lifecycle of batteries across the regimes is in the range [573, 1208], with empirical standard-deviation (for the Gaussian prior) given by $\sigma = 164$. We normalize the distribution parameters such that $\mu_i  \in [0, 1]$ for all arms $i \in [K]$. The normalization constant is fixed to be maximum of the life-cycles across all environments, i.e., $\mu_{max}=1208$. This preprocessing provides instances with $\mu_i \in (0.474, 1]$ and $\sigma^2=0.018$.

\paragraph{Adjusting for variance:} The estimators in Procedures \ref{proc:SAE-estimator} and \ref{proc:UCB-estimator} are defined under the assumption that $\sigma=1$. For non-unit $\sigma$, we extend the procedures to their \textit{variance-adjusted} versions, scaling the confidence interval by $5\sigma$, i.e $C_{i,t} = 5\sigma\sqrt{\frac{T^\alpha - 1}{\alpha n_{i,t}}}$, while still using the same estimators. 

\begin{figure}[ht]
    \centering
    \subfloat[\centering ]{{\includegraphics[width=0.25\textwidth]{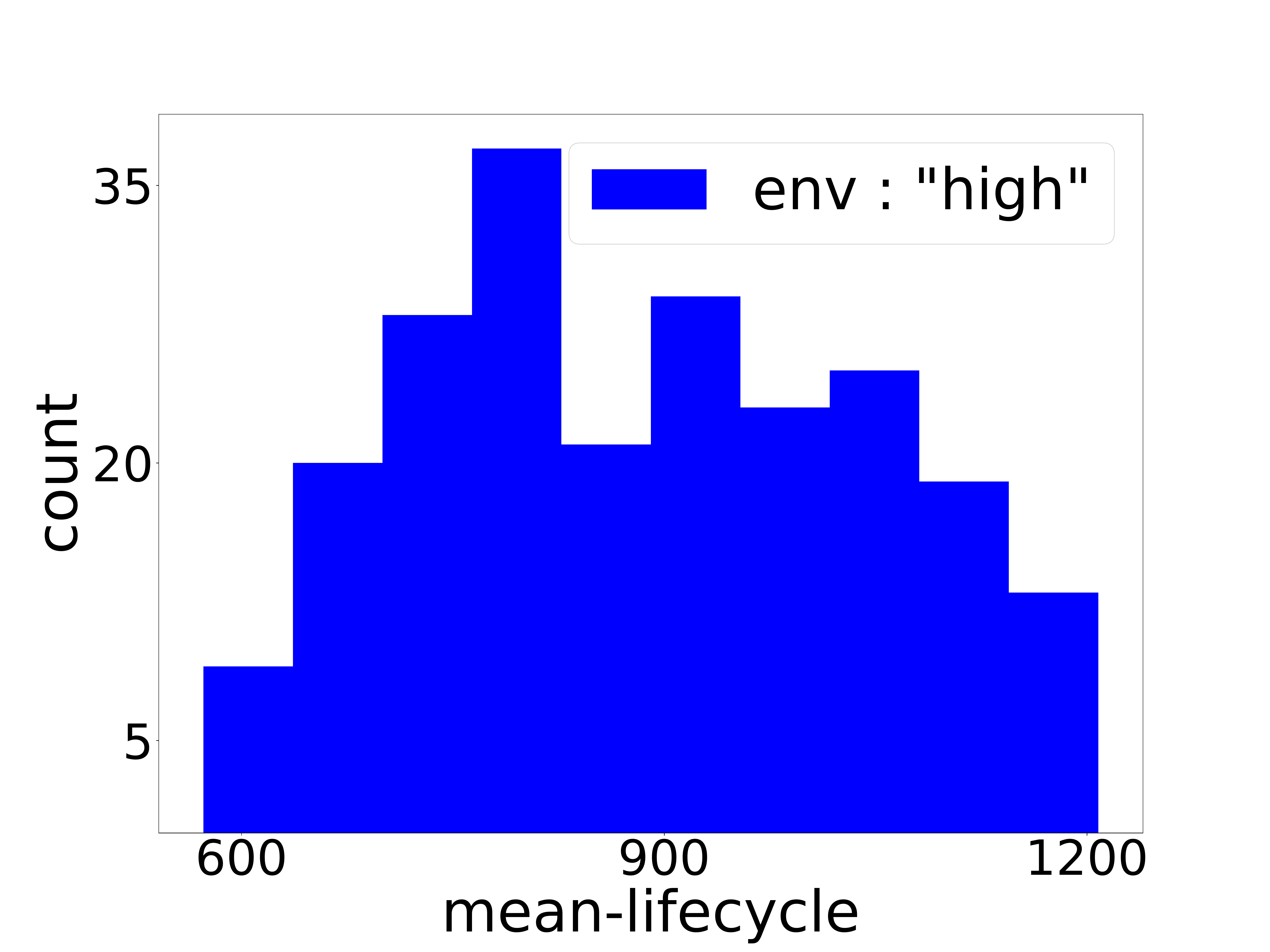} }}
    \hspace{0.01\textwidth}%
    \subfloat[\centering ]{{\includegraphics[width=0.25\textwidth]{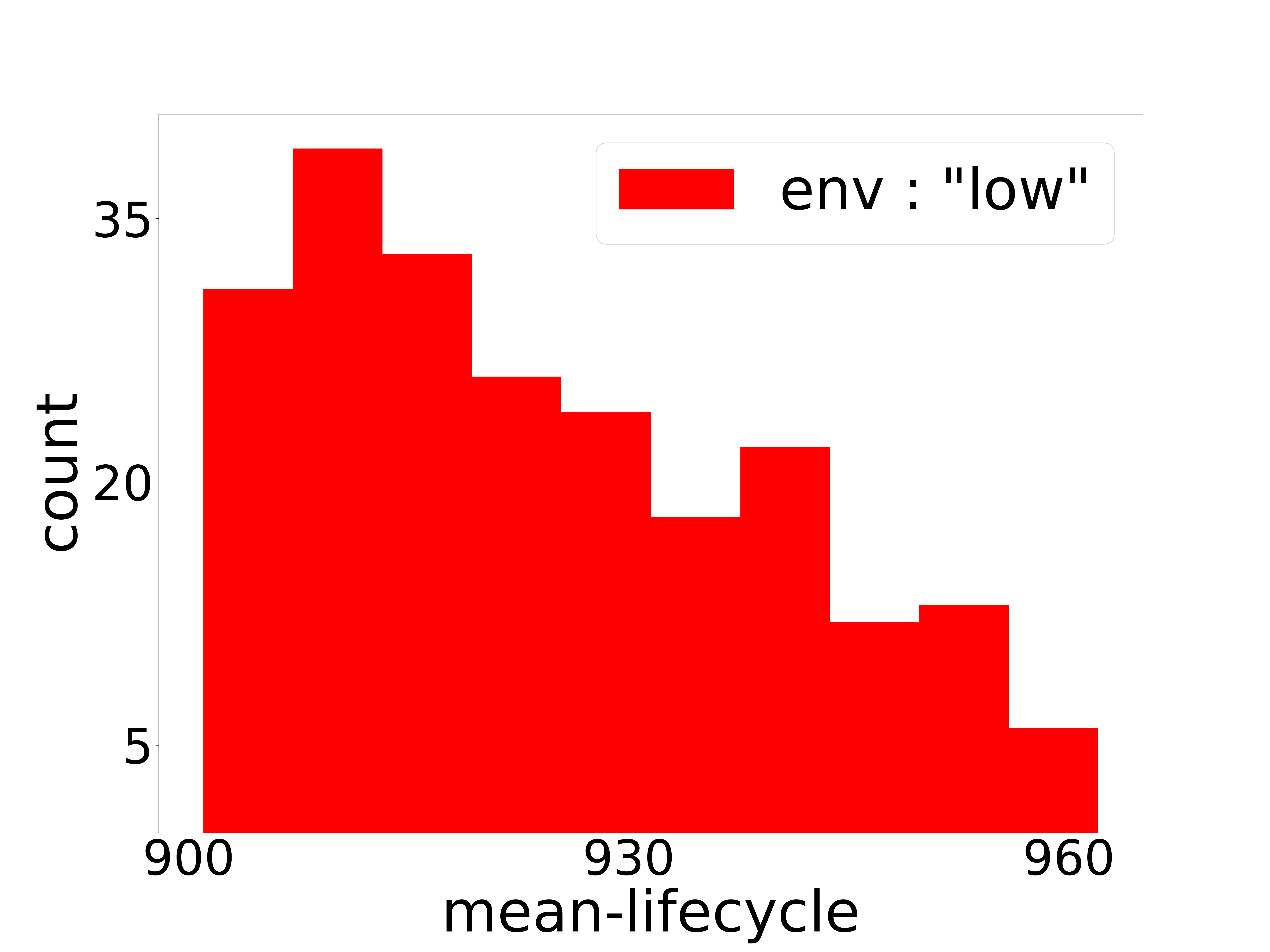} }}
    \hspace{0.01\textwidth}%
    \subfloat[\centering ]{{\includegraphics[width=0.2\textwidth]{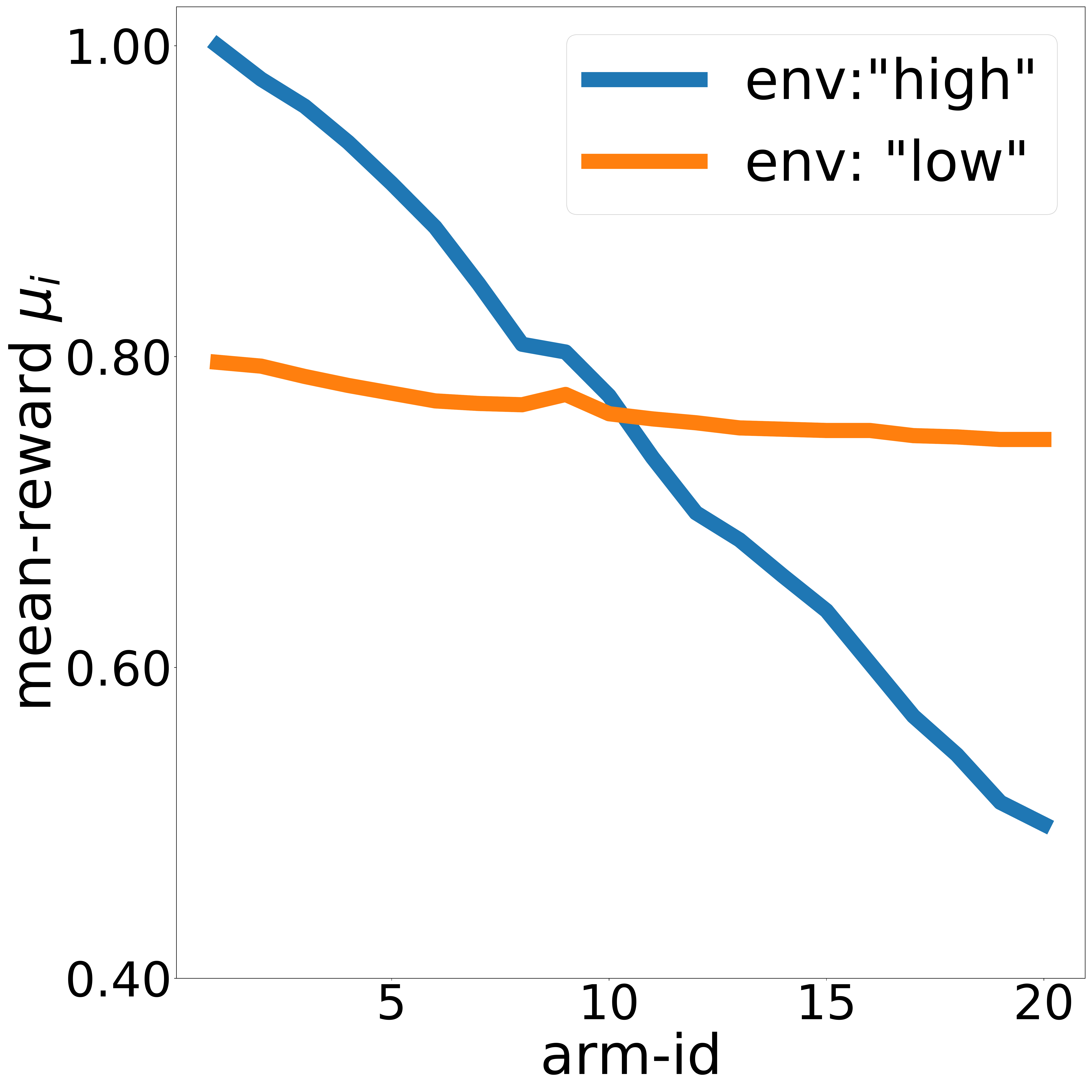} }}%
    \hspace{0.01\textwidth}%
    \subfloat[\centering ]{{\includegraphics[width=0.2\textwidth]{figures/experiments/battery/error_ucb_mse_all_env_4arm_5K.pdf} }}%
    \caption{Reward distributions in the ``low" and ``high" temperature regimes vary significantly. Figure (d) represents the error of estimating arm 12 in both ``low" and ``high" regimes with $4$ protocols.}
    \label{fig:batter_exp_appendix}
\end{figure}

\paragraph{Additional results:} 
While the distributions of mean lifetime vary between the high and low temperature regimes, there are protocols that enjoy similar performance across both regimes. For instance, Figure~\ref{fig:batter_exp_appendix}(a) shows that arm $8$ has normalized rewards of $0.802$ and $0.775$ in the high and low regimes respectively. 
In Figure~\ref{fig:batter_exp_appendix}(d), we demonstrate that it is easier to estimate the mean lifetime of this arm in the low regime. 
We run the same experiment for large subset of arms in the dataset. 
In this setting we take $\alpha=0.001$ to get a low-regret demonstrator, and fix T$\in \{25000, 45000, 70000\}$. 
In Figure~\ref{fig:battery_allarms}, we verify pictorially that the reward estimation error reduces uniformly across all arms as we increase the demonstration horizon (see the caption for a detailed explanation).


\begin{figure}[ht]
    \centering
    \includegraphics[width=\textwidth]{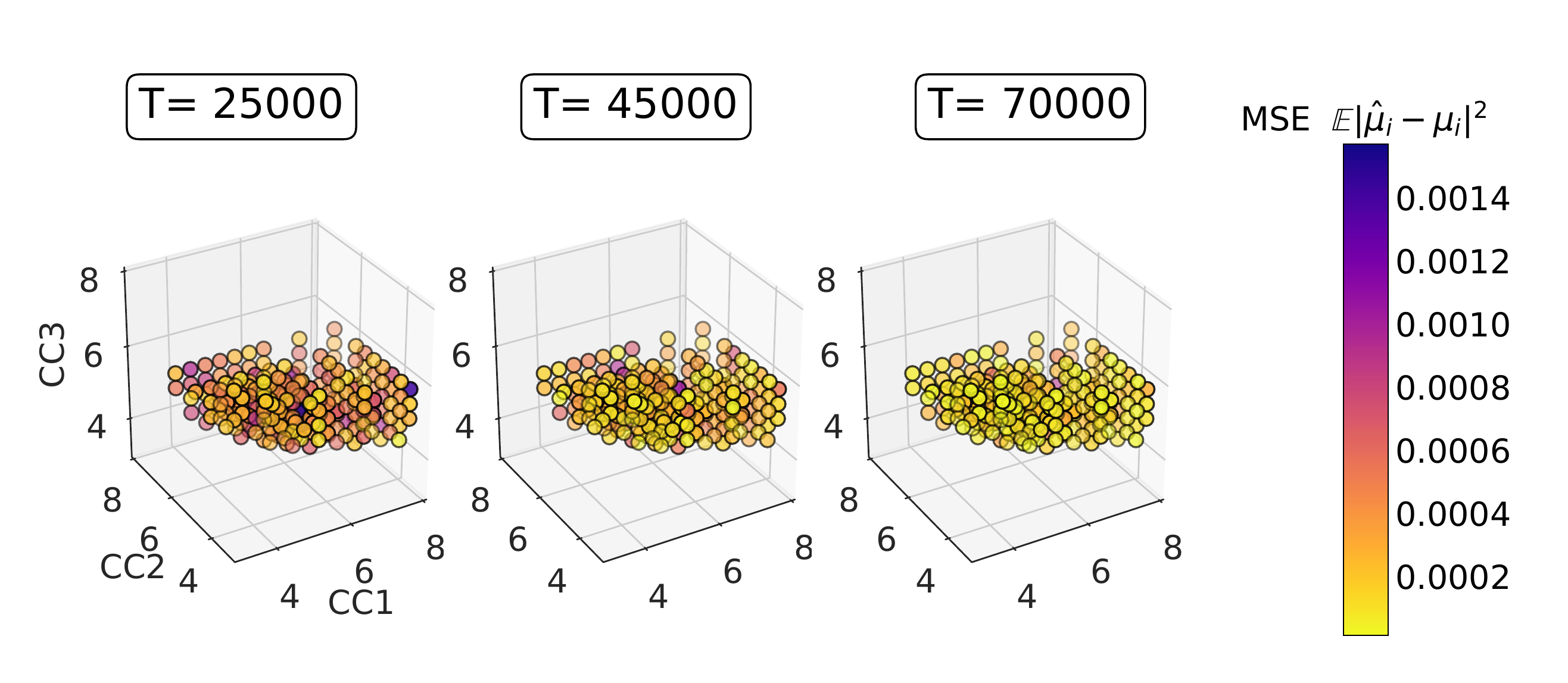}%
    \caption{Each charging protocol in the battery lifecycle dataset is defined by three independent variables (CC1, CC2, CC3). These parameters correspond to constant-currents applied to the battery in a specified range ($0–20\%, 20–40\%$ and $40–60\%$, respectively). Here, each point in the plot corresponds to one such protocol and the color profile represents mean-squared-error in estimating the average lifetime. As we increase the time-horizon $T$ of the demonstration, our estimates improve uniformly across protocols.}
    \label{fig:battery_allarms}
\end{figure}




\subsection{Gene expression}

\paragraph{Dataset:} Identifying the top genes responsible for virus replication could provide information about potential targets for antiviral therapy in the host. In one such study, \cite{hao2008drosophila} investigate 13K genes in \emph{drosophila} in the context of influenza, by adding fluorescence virus to single-gene knock-down cell strains. Measuring the fluorescence level, the authors estimate the importance of the corresponding gene in replication; where lower fluorescence indicates that the knock-down gene encourages replication. 
This problem of identifying top-$k$ genes under noisy measurements has previously been studied under the best-arm identification setting \citep{jun2016top}. 

\paragraph{Normalization:} Following the original dataset, we model rewards for arm $i$ to follow $N(\mu_i, 0.1)$. As indicated by Figure~\ref{fig:gene_expression} (a), the reward means  $\mu_i$ lie in the range $(-1.3, 2.01)$.
We normalize the reward means to be within the range $\mu_i \in$ [0, 1] by centering and scaling. 
Accordingly, the variance per arm is normalized to $0.0092$.
In summary, we have $r_i \sim N(\mu_i, 0.0092)$.

\paragraph{Results:} Our goal is to estimate the mean reward $\mu_i$ of each knock-down gene from a single demonstration with uniform error guarantees. We subsample K $\in \{100, 200, 400\}$ arms from a dataset of 12979 arms, and evaluate our estimator on each of the resulting instances. 
While sampling the subset with $K$ arms, we ensure that arm $12979$ is present across all instances, and track the error in estimating its mean reward across different instances. 
In Figure~\ref{fig:gene_expression}(b), we demonstrate that our estimator works well across all values of $K$.
Figure~\ref{fig:gene_expression}(b) also shows that the estimation error depends minimally on $K$ as predicted by our theory.

\begin{figure}[ht]
    \centering
    \subfloat[\centering ]{{\includegraphics[width=0.27\textwidth]{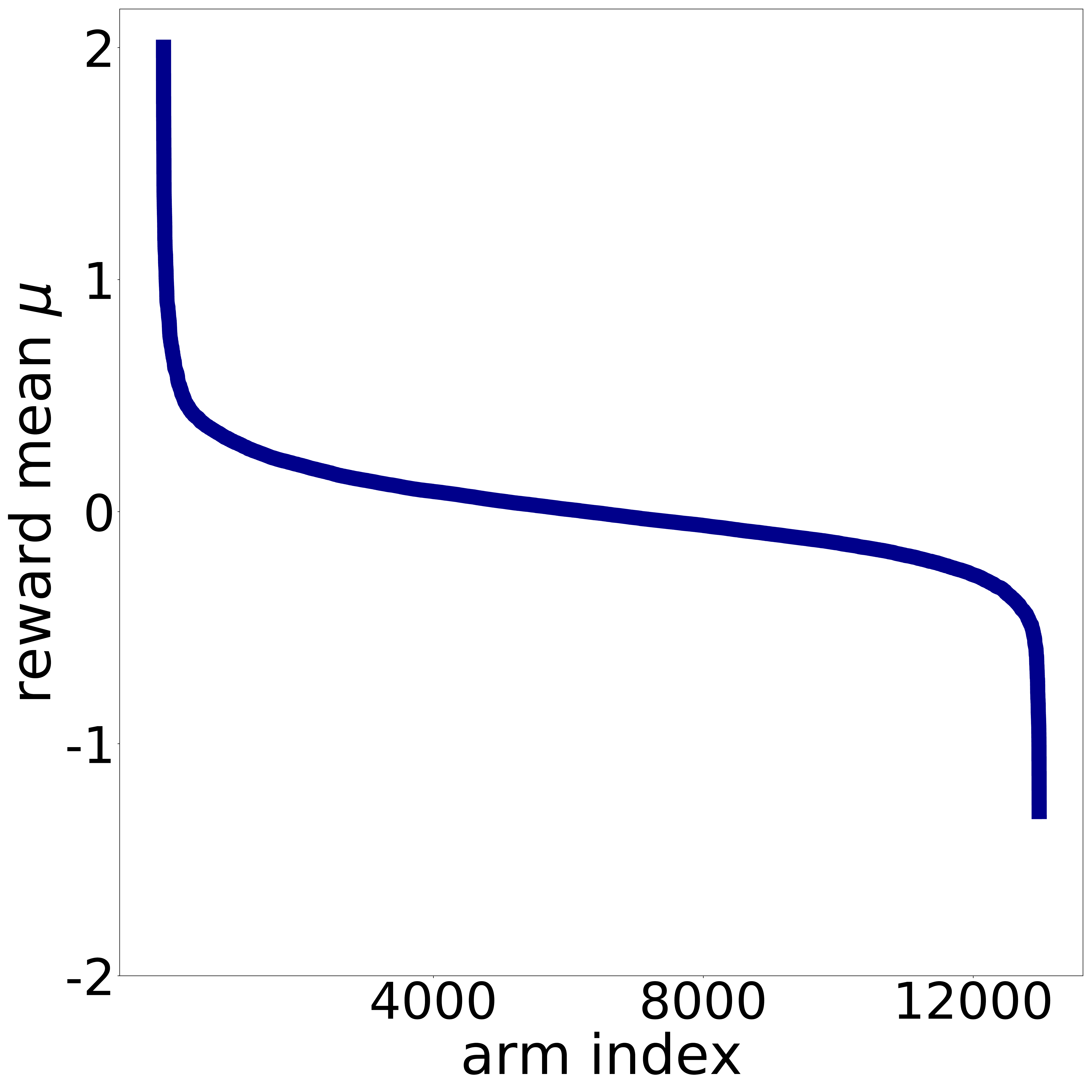} }}%
    \hspace{0.01\textwidth}%
    \subfloat[\centering ]{{\includegraphics[width=0.27\textwidth]{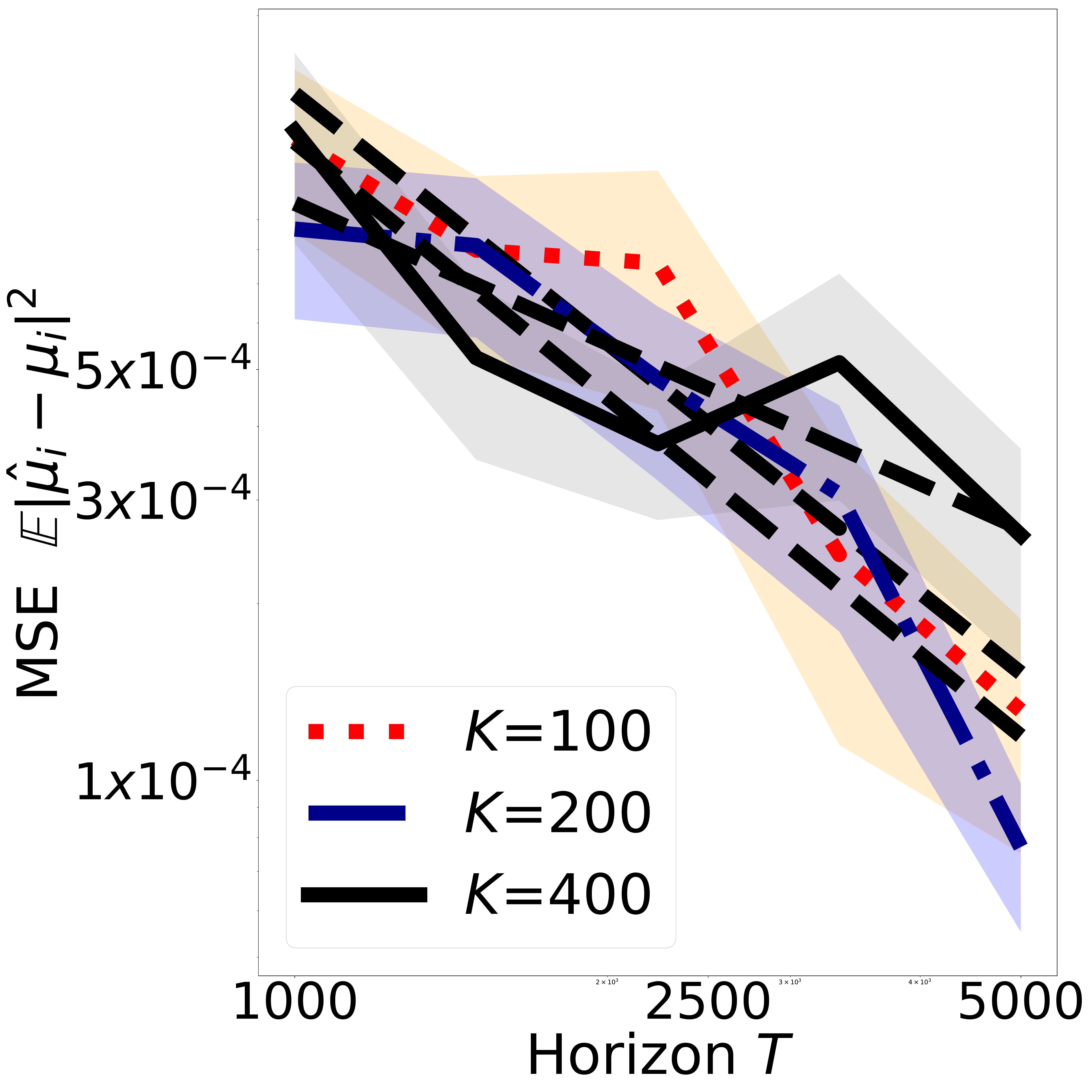} }}%
    \caption{(a) depicts mean-reward per-arm for all the $12979$ arms before normalization. (b) We track the reward estimation error of arm $12979$ (this arm is added to all instances) as a function of $T$ for $K \in \{100, 200, 400\}$.}
    \label{fig:gene_expression}
\end{figure}